\documentclass{article}

\PassOptionsToPackage{numbers, compress}{natbib}
\usepackage[final]{neurips_2020}

\usepackage[utf8]{inputenc} % allow utf-8 input
\usepackage[T1]{fontenc}    % use 8-bit T1 fonts
\usepackage{hyperref}       % hyperlinks
\usepackage{url}            % simple URL typesetting
\usepackage{booktabs}       % professional-quality tables
\usepackage{amsfonts}       % blackboard math symbols
\usepackage{nicefrac}       % compact symbols for 1/2, etc.
\usepackage{microtype}      % microtypography
\usepackage[multiple]{footmisc}
\usepackage{bm}
\usepackage{soul,color}
\usepackage{amsmath}
\usepackage{bm}
\usepackage{graphicx}
\usepackage{subcaption}
\usepackage{placeins}
\usepackage{comment}
\newsavebox{\imagebox}
\usepackage{bbm}
\usepackage{refcount}

\usepackage{enumitem}
\usepackage{amssymb}
\usepackage{amsthm}
\newtheorem{proposition}{Proposition}

\newtheorem{definition}{Definition}
\newtheorem{theorem}{Theorem}
\newtheorem{principle}{Principle}

\newcommand{\HV}{\textsc{HV}}
\newcommand{\HVI}{\textsc{HVI}}
\newcommand{\EHVI}{\textsc{EHVI}}
\newcommand{\HVIc}{\textsc{HVI}_\textsc{c}}

\newcommand{\qEHVI}{$q$\textsc{EHVI}}

\newcommand{\xcand}{\mathcal{X}_\text{cand}}
\newcommand{\aqEHVI}{\alpha_{q\textsc{EHVI}}}
\newcommand{\hataqEHVI}{\hat{\alpha}_{q\textsc{EHVI}}}
\newcommand{\aqEHVIc}{\alpha_{q\textsc{EHVI}_\textsc{c}}}
\newcommand{\qParego}{$q$\textsc{ParEGO}}

\DeclareMathOperator*{\argmax}{arg\,max}

\newcommand{\papertitle}{Differentiable Expected Hypervolume Improvement for Parallel Multi-Objective Bayesian Optimization}

\title{\papertitle}

\author{%
Samuel Daulton\\
  Facebook\\
  \texttt{sdaulton@fb.com} \\
  \And
  Maximilian Balandat\\
  Facebook\\
  \texttt{balandat@fb.com} \\
  \And
  Eytan Bakshy\\
  Facebook\\
  \texttt{ebakshy@fb.com} \\
}

\begin{document}

\maketitle

\begin{abstract}
In many real-world scenarios, decision makers seek to efficiently optimize multiple competing objectives in a sample-efficient fashion. Multi-objective Bayesian optimization (BO) is a common approach, but many of the best-performing acquisition functions do not have known analytic gradients and suffer from high computational overhead.  We leverage recent advances in programming models and hardware acceleration for multi-objective BO using Expected Hypervolume Improvement (\EHVI{})---an algorithm notorious for its high computational complexity. We derive a novel formulation of $q$-Expected Hypervolume Improvement (\qEHVI{}), an acquisition function that extends \EHVI{} to the parallel, constrained evaluation setting. \qEHVI{} is an exact computation of the joint \EHVI{} of $q$ new candidate points (up to Monte-Carlo (MC) integration error). Whereas previous \EHVI{} formulations rely on gradient-free acquisition optimization or approximated gradients, we compute exact gradients of the MC estimator via auto-differentiation, thereby enabling efficient and effective optimization using first-order and quasi-second-order methods. Our empirical evaluation demonstrates that \qEHVI{} is computationally tractable in many practical scenarios and outperforms state-of-the-art multi-objective BO algorithms at a fraction of their wall time.
\end{abstract}

\section{Introduction}
The problem of optimizing multiple competing objectives is ubiquitous in scientific and engineering applications. For example in automobile design, an automaker will want to maximize vehicle durability and occupant safety, while using lighter materials that afford increased fuel efficiency and lower manufacturing cost \citep{liao2008, vehicle2005}. Evaluating the crash safety of an automobile design experimentally is expensive due to both the manufacturing time and the destruction of a vehicle. % that is driven into a rigid wall for collision assessment. 
In such a scenario, sample efficiency is paramount. 
For a different example, video streaming web services commonly use adaptive control policies to determine the bitrate as the stream progresses in real time \citep{Mao2019RealworldVA}. A decision maker may wish to optimize the control policy to maximize the quality of the video stream, while minimizing the stall time. Policy evaluation typically requires using the suggested policy on segments of live traffic, which is subject to opportunity costs.
If long evaluation times are the limiting factor, multiple designs may be evaluated in parallel to significantly decrease end-to-end optimization time. For example, an automaker could manufacture multiple vehicle designs in parallel or a web service could deploy several control policies to different segments of traffic at the same time.
\subsection{Background}
\textbf{Multi-Objective Optimization:} In this work, we address the problem of optimizing a \emph{vector-valued objective} $\bm f(\bm x) : R^d \rightarrow \mathbb R^M$ with $\bm f(\bm x) = \bigl(f^{(1)}(\bm x),..., f^{(M)}(\bm x) \bigr)$ over a bounded set $\mathcal X \subset \mathbb R^d$. We consider the scenario in which the $f^{(i)}$ are expensive-to-evaluate black-box functions with no known analytical expression, and no observed gradients. %In contrast with the single objective setting, 
\emph{Multi-objective} (MO) optimization problems typically do not have a single best solution; rather, the goal is to identify the set of \emph{Pareto optimal} solutions such that any improvement in one objective means deteriorating another. Without loss of generality, we assume the goal is to maximize all objectives. We say a solution $\bm f(\bm x)$ \emph{Pareto dominates} another solution $\bm f(\bm x')$ if $f^{(m)}(\bm x) \geq f^{(m)}(\bm x') ~\forall~m = 1, \dotsc, M$ and there exists $m' \in \{1, \dotsc, M\}$ such that $f^{(m')}(\bm x) >  f^{(m')}(\bm x')$. We write $\bm f(\bm x) \succ \bm f(\bm x')$. 
Let $\mathcal P^* = \{\bm f(\bm x)~~s.t. ~~\nexists ~\bm x' \in \mathcal X ~:~ \bm f(\bm x') \succ \bm f(\bm x) \}$ and $\mathcal X^* = \{\bm{x} \in \mathcal X ~~s.t.~~ \bm f(\bm x) \in \mathcal P^*\}$ denote the set of Pareto optimal solutions and Pareto optimal inputs, respectively. Provided with the Pareto set, decision-makers can select a solution with an objective trade-off according to their preferences.

A common approach for solving MO problems is to use evolutionary algorithms (e.g. NSGA-II), which are robust multi-objective optimizers, but require a large number of function evaluations \citep{deb02nsgaii}. %When the objective functions are expensive-to-evaluate, 
Bayesian optimization (BO) offers a far more sample-efficient alternative \citep{shahriari16}. 

\textbf{Bayesian Optimization:} BO \citep{jones98} is an established method for optimizing expensive-to-evaluate black-box functions. BO relies on a probabilistic \emph{surrogate model}, 
% and an acquisition function. The Bayesian surrogate model, 
typically a Gaussian Process (GP) \citep{Rasmussen2004}, to provide a posterior distribution $\mathbb P(\bm f | \mathcal D)$ over the true function values $\bm f$ given the observed data $\mathcal D = \{(\bm x_i, \bm y_i)\}_{i=1}^n$. An \emph{acquisition function} $\alpha: \xcand \mapsto \mathbb{R}$ employs the surrogate model to assign a utility value to a set of candidates $\xcand = \{\bm x_i\}_{i=1}^q$ to be evaluated on the true function. While the true $\bm f$ may be expensive-to-evaluate, the surrogate-based acquisition function is not, and can thus be efficiently optimized to yield a set of candidates~$\xcand$ to be evaluated on $\bm f$. If gradients of $\alpha(\xcand)$ % with respect to $\xcand$ 
are available, gradient-based methods can be utilized. If not, gradients are either approximated (e.g. with finite differences) or gradient-free methods (e.g. DIRECT \citep{jones93} or CMA-ES \citep{cmaes}) are used.
%The core idea behind BO is to model the observed outcomes using a probabilistic \emph{surrogate model}, typically a Gaussian Processes (GP). GPs are flexible probabilistic models known for their well-calibrated uncertainty in many scenarios \citep{Rasmussen2004}. In contrast with model-free methods including evolutionary algorithms which solely rely on evaluating candidates on the actual function, BO performs optimization on the cheap-to-evaluate surrogate  in order to determine good candidate solutions to be evaluated on the actual function. The surrogate posterior can be used to guide exploration in a principled fashion through an acquisition function---a criteria that quantifies the utility of evaluating a set of one or more candidate inputs $\xcand \subset \mathcal X$ on the actual function. % and which can be cheaply optimized using the GP surrogate.

\subsection{Limitations of current approaches} 
In the single-objective (SO) setting, a large body of work focuses on practical extensions to BO for supporting \emph{parallel} evaluation and outcome constraints \citep{marmin,Ginsbourger2010, wu16, gardner2014constrained, letham2019noisyei}. Less attention has been given to such extensions in the MO setting. Moreover, the existing constrained and parallel MO BO options have limitations: 1) many rely on scalarizations to transform the MO problem into a SO one~\citep{parego}; 2) many acquisition functions are computationally expensive to compute~\citep{sur, emmerich2011,belakaria2019,yang_multipoints2019}; 3) few have known analytical gradients or are differentiable \citep{emmerich_ehvi2011, wada2019bayesian, pesmo}; 4) many rely on heuristics to extend sequential algorithms to the parallel setting \citep{garridomerchn2020parallel, wada2019bayesian}.

A natural acquisition function for MO BO is Expected Hypervolume Improvement (\EHVI). Maximizing the hypervolume (\HV{}) has been shown to produce Pareto fronts with excellent coverage~\citep{zitzler03,Couckuyt14, yang2019}. However, there has been little work on \EHVI{} in the parallel setting, 
% and none have suggested using and optimizing the ex
% exact \qEHVI{}: rather, 
and the work that has been done resorts to approximate methods \citep{yang_multipoints2019, Gaudrie_2019, wada2019bayesian}. A vast body of literature has focused on  efficient \EHVI{} computation \citep{hupkens15, Emmerich2016, yang17}, but the time complexity for computing \EHVI{} is exponential in the number of objectives---in part due the hypervolume indicator itself incurring a time complexity that scales super-polynomially with the number of objectives \citep{yang_emmerich2019}. Our core insight is that by exploiting advances in auto-differentiation and highly parallelized hardware \citep{paszke2017automatic}, we can make \EHVI{} computations fast and practical.

% Algorithmic complexity and computational demand has long been of great concern also in the machine learning (ML) community. 
% Over the last decade or so, ML has witnessed a revolution made possible by novel computational paradigms that exploit auto-differentiation in conjunction with highly parallelized hardware (e.g. GPUs). These paradigms have been made available and easy to used by various deep learning libraries \citep{jia2014caffe, chen2015mxnet, abadi2016tensorflow, paszke2017automatic}. 

\subsection{Contributions}
In this work, we derive a novel formulation of the parallel $q$-Expected Hypervolume Improvement acquisition function (\qEHVI{}) that is exact up to Monte-Carlo (MC) integration error. We compute the \emph{exact} gradient of the MC estimator of \qEHVI{} using auto-differentiation, which allows %optimizing \qEHVI{} with 
us to employ efficient and effective gradient-based optimization methods. Rather than using first-order gradient methods, we instead leverage the sample average approximation (SAA) approach from \citep{balandat2020botorch} to use higher-order deterministic optimization methods, and we prove theoretical convergence guarantees under the SAA approach. Our formulation of \qEHVI{} is embarrassingly parallel, and despite its computational cost would achieve constant time complexity given infinite processing cores. We demonstrate that, using modern GPU hardware and computing exact gradients, optimizing \qEHVI{} is faster than existing state-of-the art methods in many practical scenarios. Moreover, we extend \qEHVI{} to support auxiliary outcome constraints, making it practical in many real-world scenarios. Lastly, we demonstrate how modern auto-differentiation can be used to compute exact gradients of analytic \EHVI{}, which has never been done before for $M>2$ objectives. Our empirical evaluation shows that \qEHVI{} outperforms state-of-the-art multi-objective BO algorithms while using only a fraction of their wall time. 

\section{Related Work}
\label{sec:RelatedWork}
\citet{yang2019} is the only previous work to consider exact gradients of \EHVI{}, but the authors only derive an analytical gradient for the unconstrained two-objective, sequential optimization setting.
%Interestingly, the authors use CMA-ES to find good starting points for gradient ascent, but find that gradient ascent does not provide additional benefit.
All other works either do not optimize \EHVI{} (e.g. they use it for pre-screening candidates \citep{emmerich2006}), optimize it with gradient-free methods \citep{yang_emmerich2019}, or using approximate gradients \citep{wada2019bayesian}. In contrast, we use exact gradients and demonstrate that optimizing \EHVI{} using this gradient information is far more efficient. 

There are many alternatives to \EHVI{} for MO BO. For example, ParEGO \citep{parego} and TS-TCH \citep{paria2018randscalar} randomly scalarize the objectives and use Expected Improvement \citep{jones98} and Thompson Sampling \citep{thompson}, respectively. SMS-EGO \citep{smsego} uses \HV{} in a UCB-based acquisition function and is more scalable than \EHVI{} \citep{rahat17}. ParEGO and SMS-EGO have only been considered for the $q=1$, unconstrained setting. Predictive entropy search for MO BO (PESMO) \citep{pesmo} has been shown to be another competitive alternative and has been extended to handle constraints \citep{garrido2019predictive} and parallel evaluations~\citep{garridomerchn2020parallel}. MO max-value entropy search (MO-MES) has been shown to achieve superior optimization performance and faster wall times than PESMO, but is limited to $q=1$. 
% The authors argue one advantage of MO-MES is robustness to number of posterior samples, which are computationally expensive in their setting. MC samples are trivially parallelizable for \qEHVI{}, 
%has constant time complexity with respect to the number of MC samples (given infinite processor cores). 
% and modern, highly parallel hardware allows us to use \emph{many} (hundreds or thousands) MC samples. But this is generally not necessary, since we observe empirically that \qEHVI{} approximation is highly accurate and smooth using only a small number of MC samples---thanks to quasi-Monte-Carlo (QMC) integration and the sample average approximation of \qEHVI{} \citep{balandat2020botorch, caflisch1998monte,kleywegt2002sample}.

\citet{wilson2018maxbo} empirically and theoretically show that \emph{sequential greedy} selection of $q$ candidates achieves performance comparable to jointly optimizing $q$ candidates for many acquisition functions (including \citep{wang2016parallel, wu16}). The sequential greedy approach integrates over the posterior of the unobserved outcomes corresponding to the previously selected candidates in the $q$-batch. Sequential greedy optimization often yields better empirical results because the optimization problem has a lower dimension: $d$ in each step, rather than $q d$ in the joint problem. Most prior works in the MO setting use a sequential greedy approximation or heuristics \citep{wada2019bayesian, yang_multipoints2019, Gaudrie_2019, chaudhuri}, but impute the unobserved outcomes with the posterior mean rather than integrating over the posterior \citep{Ginsbourger2010}. For many joint acquisition functions involving expectations, this shortcut sacrifices the theoretical error bound on the sequential greedy approximation because the exact joint acquisition function over $\bm x_1, ..., \bm x_i,~ 1\leq i \leq q$ requires integration over the joint posterior $\mathbb P(\bm f(\bm x_1), ..., \bm f(\bm x_q) | \mathcal D)$ and is not computed for $i>1$. 
 
\citet{garridomerchn2020parallel} and \citet{wada2019bayesian} jointly optimize the $q$ candidates and, noting the difficulty of the optimization, both papers focus on deriving gradients to aid in the optimization.
% ; \citep{garridomerchn2020parallel} finds that joint optimization is is faster than sequential greedy using the posterior mean, but observe similar optimization performance. 
\citet{wada2019bayesian} defined the \qEHVI{} acquisition function, but after finding it challenging to optimize $q$ candidates jointly (without exact gradients), the authors propose optimizing an alternative acquisition function instead of exact \qEHVI{}. In contrast, our novel \qEHVI{} formulation allows for gradient-based parallel and sequential greedy optimization, with proper integration over the posterior for the latter.

\citet{Feliot_2016} and \citet{ehvic} proposed extensions of \EHVI{} to the constrained $q=1$ setting, but neither considers the batch setting and both rely on gradient-free  optimization.
\section{Differentiable $q$-Expected Hypervolume Improvement}
\label{sec:DqEHVI}
In this section, we review HVI and EHVI computation by means of box decompositions, and explain our novel formulation for the parallel setting.
\begin{definition}
Given a reference point $\bm r \in \mathbb R^M$, the hypervolume indicator (\HV{}) of a finite approximate Pareto set $\mathcal P$ is the $M$-dimensional Lebesgue measure $\lambda_M$ of the space dominated by $\mathcal P$ and bounded from below by $\bm r$: $\HV(\mathcal P, \bm r) = \lambda_M \big(\bigcup_{i=1}^{\vert \mathcal P\vert} [\bm r, \bm y_i]\big)$, where $[\bm r, \bm y_i]$ denotes the hyper-rectangle bounded by vertices $\bm r$ and $\bm y_i$.
\end{definition}
\begin{definition}
Given a Pareto set $\mathcal P$ and reference point $\bm r$, the hypervolume improvement (\HVI{}) of a set of points $\mathcal Y$ is: $\HVI{}(\mathcal Y, \mathcal P, \bm r) = \HV{}(\mathcal P \cup \mathcal Y, \bm r) -   \HV{}(\mathcal P, \bm r)$.\footnote{In this work, we omit the arguments $\mathcal P$ and $\bm r$ when referring to \HVI{} for brevity.}
\end{definition}
\EHVI{} is the expectation of \HVI{} over the posterior $\mathbb P(\bm f, \mathcal D)$: $\alpha_{\EHVI{}}(\xcand) = \mathbb E\big[\HVI{}(\bm f(\xcand))\big]$. In the sequential setting, and assuming the objectives are independent and modeled with independent GPs, \EHVI{} can be expressed in closed form \citep{yang2019}. In other settings, \EHVI{} can be approximated with MC integration. Following previous work, we assume that the reference point is known and specified by the decision maker \citep{yang2019} (see Appendix \ref{appdx:sec:RefPoint} for additional discussion).

\subsection{A review of hypervolume improvement computation using box decompositions}
\begin{definition}\label{def:delta} For a set of objective vectors $\{\bm f(\bm x_i)\}_{i=1}^q$, a reference point $\bm r \in \mathbb R^M$, and a non-dominated set $\mathcal P$, let $\Delta(\{\bm f(\bm x_i)\}_{i=1}^q, \mathcal P, \bm r) \subset \mathbb{R}^M$ denote the set of points (i) are dominated by  $\{\bm f(\bm x_i)\}_{i=1}^q$, dominate $\bm r$, and are not dominated by $\mathcal P$.
\end{definition}
\vspace{-1ex}

Given $\mathcal{P}, \bm r$, the \HVI{} of a new point $\bm f(\bm x)$ is the \HV{} of the intersection of space dominated by $\mathcal P \cup \{\bm f(\bm x)\}$ and the non-dominated space. Figure \ref{fig:HI_b} illustrates this for one new point $\bm f(\bm x)$ for $M=2$. The yellow region is $\Delta(\{\bm f(\bm x)\}, \mathcal P, \bm r)$ and the hypervolume improvement is the volume covered by $\Delta(\{\bm f(\bm x)\}, \mathcal P, \bm r)$. Since $\Delta(\{\bm f(\bm x)\}, \mathcal P, \bm r)$ is often a non-rectangular polytope, \HVI{} is typically computed by partitioning the non-dominated space into disjoint axis-parallel rectangles \citep{Couckuyt14, yang_emmerich2019} (see Figure \ref{fig:HI_a}) and using piece-wise integration \citep{emmerich2006}.

Let $\{S_k\}_{k=1}^{K}$ be a partitioning the of non-dominated space into disjoint hyper-rectangles, where each $S_k$ is defined by a pair of lower and upper vertices $\bm l_k \in \mathbb{R}^M$ and $\bm u_k \in \mathbb{R}^M \cup \{\bm\infty\}$. The high level idea is to sum the HV of $S_k \cap \Delta(\{\bm f(\bm x)\}, \mathcal P, \bm r)$ over all $S_k$. For each hyper-rectangle $S_k$, the intersection of $S_k$ and $\Delta(\{\bm f(\bm x)\}, \mathcal P, \bm r)$ is a hyper-rectangle where the lower bound vertex is $\bm l_k$ and the upper bound vertex is the component-wise minimum of $\bm u_k$ and the new point $\bm f(\bm x)$: $\bm z_k := \min \big[\bm u_k,\bm f(\bm x)\big]$.\\[-3.5ex]

% %%%%%%%%%%%%%%%%%%%%%%%%
% \subsection{Sequential Setting}
% \label{subsec:DqEHVI:Sequential}

\begin{figure}[ht]
    \centering
    \begin{subfigure}{.33\textwidth}
        \centering
        \includegraphics[width=0.9\linewidth]{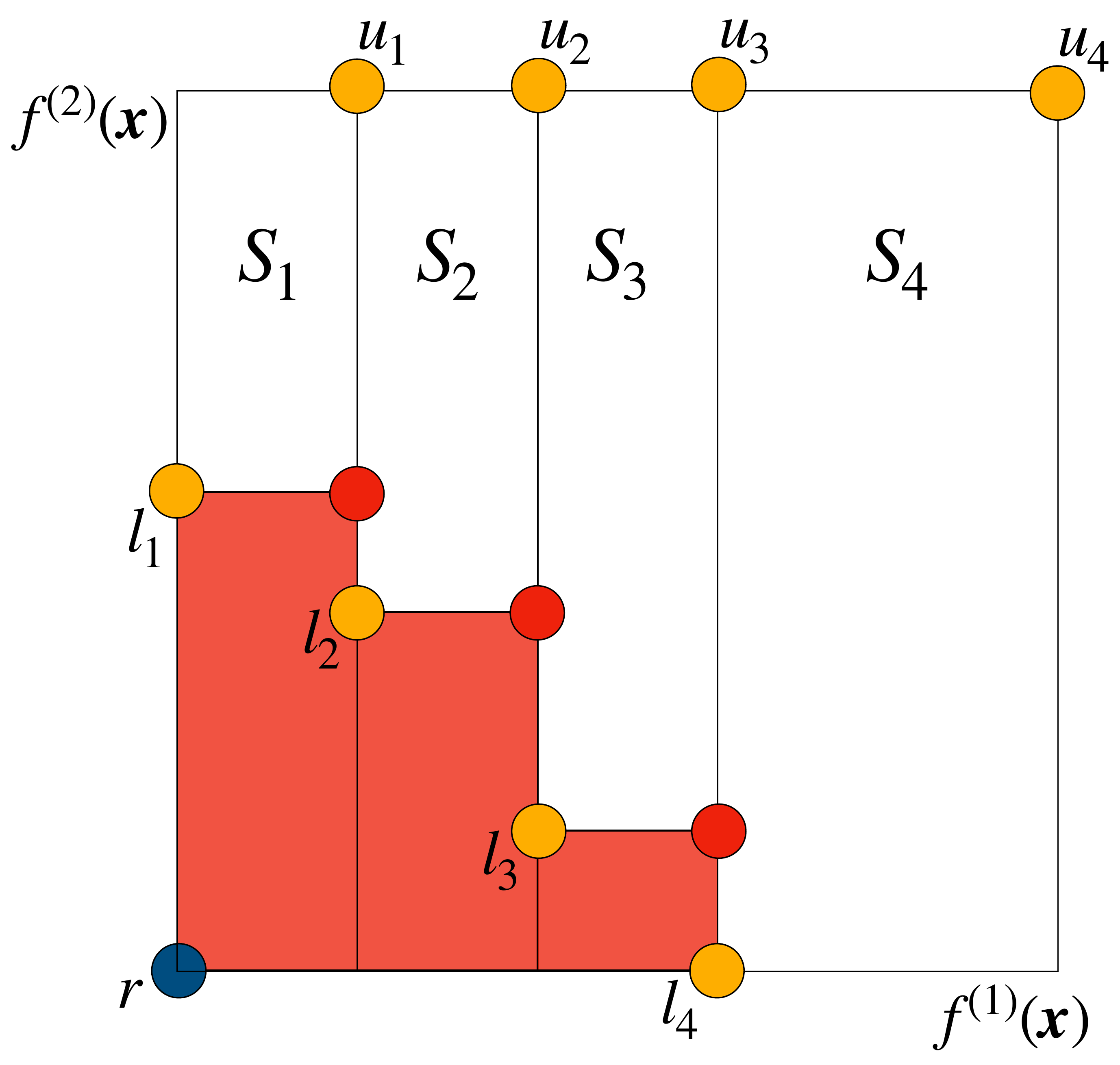}\vspace{-2ex}
        \subcaption{\label{fig:HI_a}}
    \end{subfigure} %
    \begin{subfigure}{.33\textwidth}    
        \centering
        \includegraphics[width=0.9\linewidth]{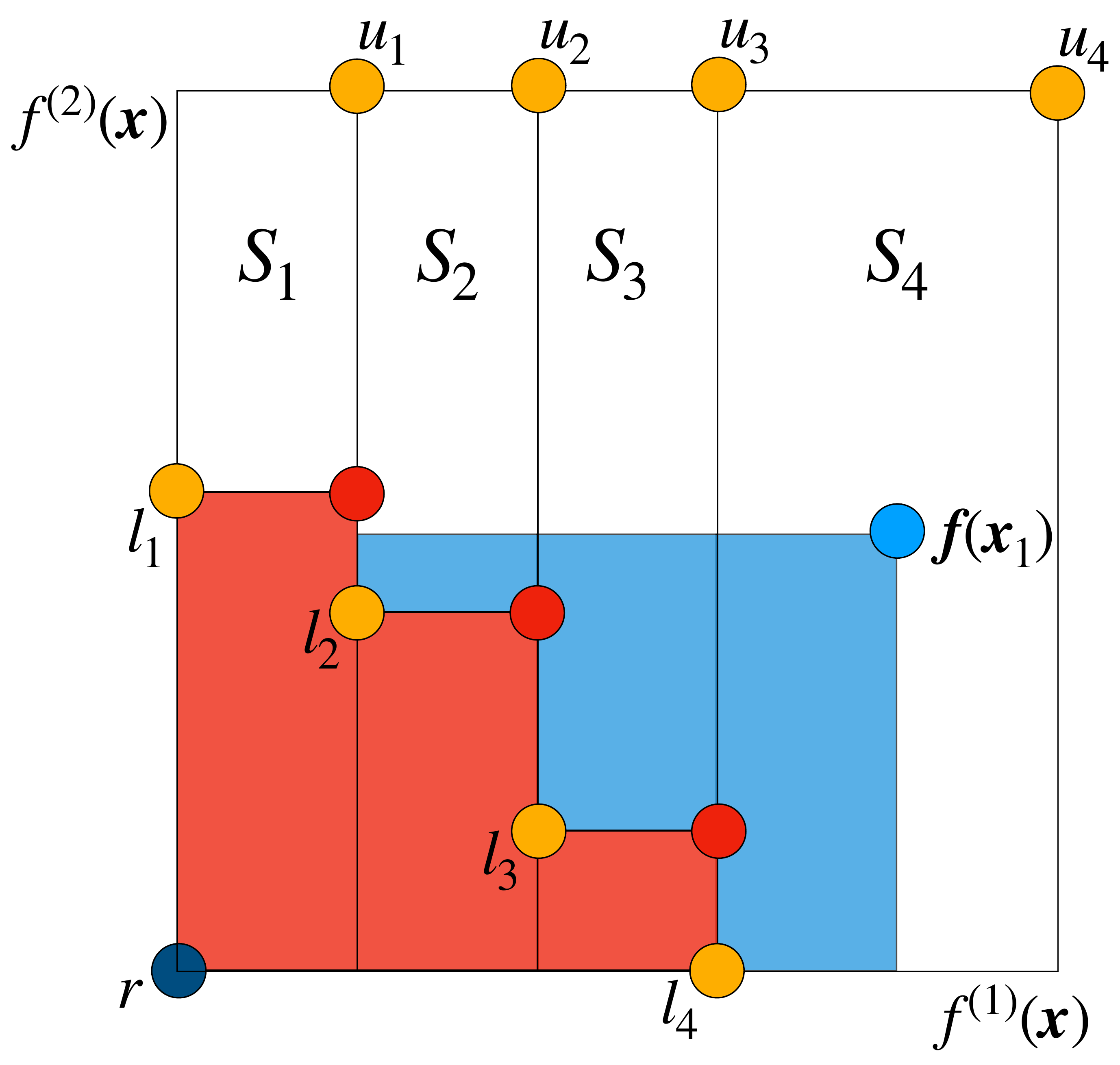}\vspace{-2ex}
        \subcaption{\label{fig:HI_b}}
    \end{subfigure}%
    \begin{subfigure}{.33\textwidth}    
        \centering
        \includegraphics[width=0.9\linewidth]{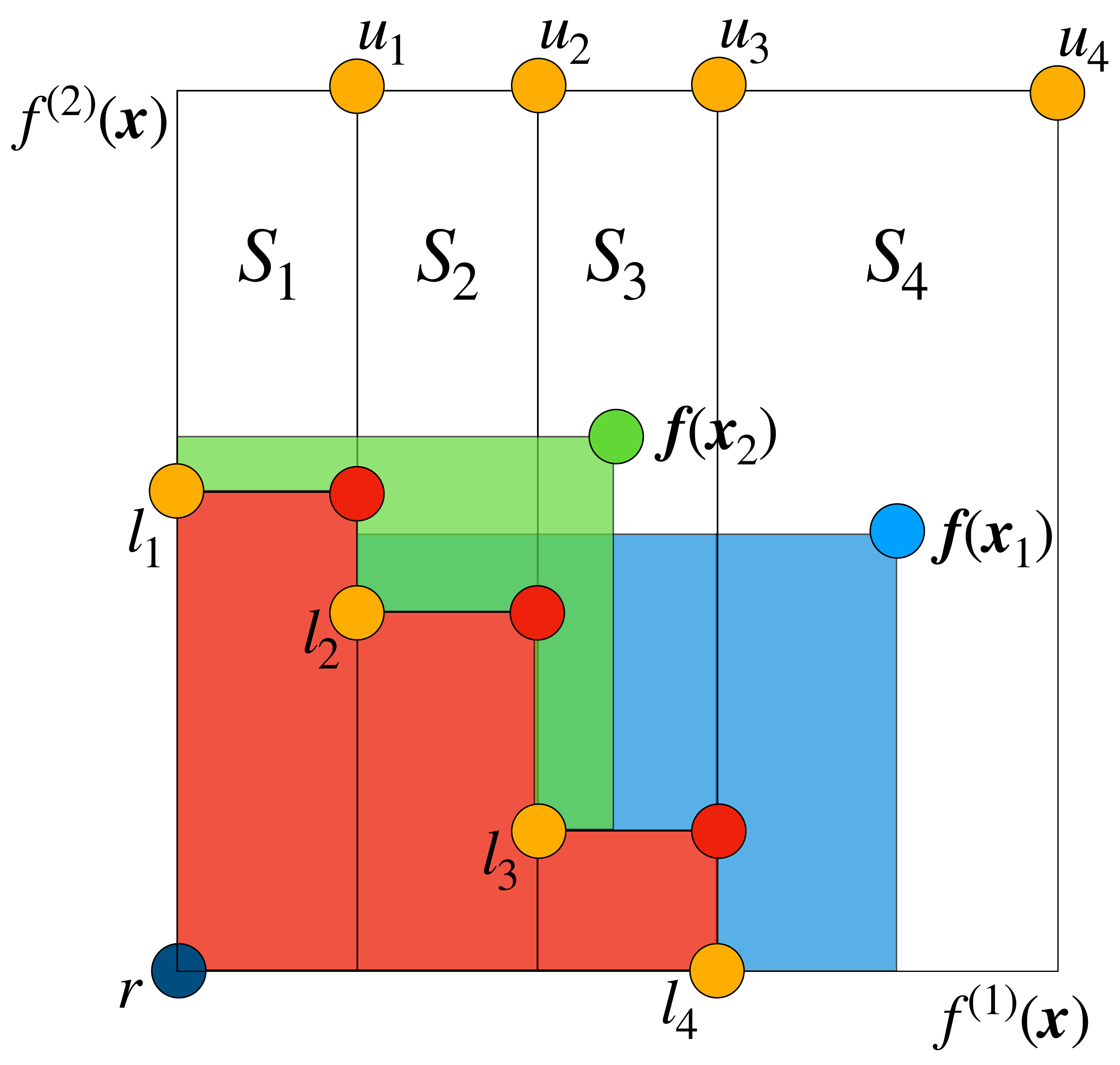}\vspace{-2ex}
        \subcaption{\label{fig:HI_c}}
    \end{subfigure}
    \caption{For M=2, (a) the dominated space (red) and the non-dominated space partitioned into disjoint boxes (white), (b) the \HVI{} of one new point $\bm f(\bm x)$, and (c) the \HVI{} of two new points $\bm f(\bm x_1), \bm f(\bm x_2)$.}
\end{figure}
\vspace{-1.5ex}

%%%%%
% \textbf{Sequential Setting: } 
% \label{subsec:DqEHVI:Sequential}
%
Hence, the \HVI{} of a single outcome vector $\bm{f}(\bm{x})$ within %the hyper-rectangle 
$S_k$ is given by 
$\HVI{}_k\bigl(\bm{f}(\bm{x}), \bm l_k, \bm u_k\bigr) = \lambda_M\bigl(S_k \cap \Delta(\{\bm f(\bm x)\}, \mathcal P, \bm r)\bigr) = \prod_{m=1}^M \bigl[z_k^{(m)} - l_k^{(m)}\bigl]_{+}$, 
where $u_k^{(m)}, l_k^{(m)}, f^{(m)}(\bm x)$, and $z_k^{(m)}$ denote the $m^{\text{th}}$ component of the corresponding vector and $[\cdot]_{+}$ denotes the $\min(\cdot, 0)$ operation.
% The first $\min$ ensures that we do not go past the right boundary of the rectangle specified by $\bm{r}_k$ and clamping the minimum value ensures that the contribution to the area is zero if the new point is dominated by a point on the current Pareto front. The hypervolume improvement is the overlapping volume between the each of the hyper-rectangles ($S_k$) in the non-dominated hypervolume and the new hypervolume under the current point. 
Summing over rectangles yields\\[-3ex]
\begin{align}
\label{eq:DqEHVI:HVI}
    \HVI{}\big(\bm{f}(\bm{x})\big) =\sum_{k=1}^{K}\HVI{}_k\big(\bm{f}(\bm{x}), \bm l_k, \bm u_k\big)
    % &=\sum_{k=1}^{K}\lambda_M\big(S_k \cap \Delta(\{\bm f(\bm x)\}, \mathcal P, \bm r)\big)\\
    =\sum_{k=1}^{K}\prod_{m=1}^M \big[z_k^{(m)} - l_k^{(m)}\big]_{+}
\end{align}
%%%%%%%%%%%%%%%%%%%%%%%%
\subsection{Computing $q$-Hypervolume Improvement via the Inclusion-Exclusion Principle}
\label{subsec:DqEHVI:qHVI}

Figure \ref{fig:HI_c} illustrates the \HVI{} in the $q=2$ setting. % To generalize the computation from above to $q>1$, we apply the inclusion-exclusion principle \citep{silva1854, sylvester1883}. 
Given $q$ new points$\{\bm f(\bm x_i)\}_{i=1}^q$, let $A_i:= \Delta(\{\bm f(\bm x_i)\}, \mathcal P, \bm r)$ for $i=1, \dots, q$ be the space dominated by $\bm f(\bm x_i)$ but not dominated by $ \mathcal P$, independently of the other $q-1$ points. Note that $\lambda_M (A_i) = \HVI{}(\bm f(\bm x_i))$. The union of the subsets $A_i$ is the space dominated jointly by the $q$ new points: $\bigcup_{i=1}^q A_i = \bigcup_{i=1}^q \Delta(\{\bm f(\bm x_i)\}, \mathcal P, \bm r)$, and the Lebesgue measure $\lambda_M\big(\bigcup_{i=1}^q A_i\big)$ is the joint \HVI{} from the $q$ new points. Since each subspace $A_i$ is bounded, the restricted Lebesgue measure is finite and we may compute $\lambda_M\big(\bigcup_{i=1}^q A_i\big)$ using the inclusion-exclusion principle \citep{silva1854, sylvester1883}:\\[-2ex]
%we have:
\begin{align}
\HVI{}(\{\bm f(\bm x_i)\}_{i=1}^q) 
= \lambda_M \bigg( \bigcup_{i=1}^q A_i \bigg)
= \sum_{j=1}^q (-1)^{j+1}\sum_{1 \leq i_1 \le \ldots \le i_j \leq q} \lambda_M\big( A_{i_1} \cap \dots \cap A_{i_j} \big)
\end{align}
Since $\{S_k\}_{k=1}^K$ is a disjoint partition,
$\lambda_M( A_{i_1} \cap \dots \cap A_{i_j}) = \sum_{k=1}^{K} \lambda_M(  S_k \cap A_{i_1} \cap \dots \cap A_{i_j} )$, we can compute $\lambda_M ( A_{i_1} \cap \dots \cap A_{i_j} )$ in a piece-wise fashion
across the $K$ hyper-rectangles $\{S_k\}_{k=1}^K$ as the \HV{} of the intersection of $A_{i_1} \cap \dots \cap A_{i_j}$ with each hyper-rectangle $S_k$. The inclusion-exclusion principle has been proposed for computing \HV{} (not \HVI{}) \citep{lopez15}, but it is rarely used because complexity scales exponentially with the number of elements. However, the inclusion-exclusion principle is practical for computing the joint \HVI{} of $q$ points since typically $q << |\mathcal P|$. 

This formulation has three advantages. First, while the new dominated space $A_i$ can be a non-rectangular polytope, the intersection $A_i \cap S_k$ is a \emph{rectangular} polytope, which simplifies computation of overlapping hypervolume. Second, the vertices defining the hyper-rectangle $S_k \cap A_{i_1} \cap \dots \cap A_{i_j} $ are easily derived. The lower bound is simply the $\bm l_k$ lower bound of $S_k$, and the upper bound is the component-wise minimum $\bm z_{k, i_1, ... i_j} := \min \big[\bm u_k,\bm f(\bm x_{i_1}), \ldots, \bm f(\bm x_{i_j})\big]$.
Third, computation can be across all intersections of subsets $A_{i_1} \cap \dots \cap A_{i_j}$ for $1\leq i_j \leq \ldots \leq i_j \leq q$ and across all $K$ hyper-rectangles can be performed in parallel. Explicitly, the \HVI{} is computed as:\\[-3ex]
\begin{align}
\label{eq:qEHVI:explicit}
    \HVI{}(\{\bm f(\bm x_i)\}_{i=1}^q)
    % &= \bigg\vert \bigcup_{i=1}^p A_i \bigg\vert \\
    % &= \sum_{j=1}^q \sum_{1 \leq i_1 \le \ldots \le i_j \leq q}(-1)^{j+1}  \big\vert A_{i_1} \cap \dots \cap A_{i_j} \big\vert\\
    % &=\sum_{k=1}^{K}\sum_{j=1}^q \sum_{1 \leq i_1 \le \ldots \le i_j \leq q} (-1)^{j+1}\big\vert S_k \cap A_{i_1} \cap \dots \cap A_{i_j} \big\vert\\
    % &= \sum_{k=1}^{K}\sum_{j=1}^q \sum_{1 \leq i_1 \le \ldots \le i_j \leq q} (-1)^{j+1} \lambda_M\big(S_k \cap \Delta(\{\bm f(\bm x_{i_1})\}, \mathcal P, \bm r) \cap  \ldots \cap \Delta(\{\bm f(\bm x_{i_j})\}, \mathcal P, \bm r)\big)\\
    % &=\sum_{k=1}^{K}\sum_{j=1}^q \sum_{1 \leq i_1 \le \ldots \le i_j \leq q} (-1)^{j+1} \prod_{m=1}^M\big[z_{k, i_1, ... i_j}^{(m)} - l_k^{(m)}\big]_{+}\\
    =\sum_{k=1}^{K}\sum_{j=1}^q \sum_{X_j \in \mathcal X_j} (-1)^{j+1} \prod_{m=1}^M\big[z_{k, X_j}^{(m)} - l_k^{(m)}\big]_{+}
\end{align}
where $\mathcal X_j := \{X_j \subset \xcand : \vert X_j \vert = j\}$ is the superset of all subsets of $\xcand$ of size $j$, and $z_{k, X_j}^{(m)} := z_{k, i_1, ... i_j}^{(m)}$ for $X_j = \{\bm x_{i_1}, ..., \bm x_{i_j}\}$. See Appendix~\ref{appdx:sec:qEHVIDerivation} for further details of the derivation.

%%%%%%%%%%%%%%%%%%%%%%%%
\subsection{Computing \emph{Expected} $q$-Hypervolume Improvement}
\label{subsec:DqEHVI:ComputeEHVI}

The above approach for computing \HVI{} assumes that we know the true objective values $\bm f(\xcand) = \{\bm f(\bm x_i)\}_{i=1}^q$. In BO, we instead compute \qEHVI{} as the expectation over the posterior model posterior:\\[-4ex] 
\begin{align}
\label{eqn:DqEHVI:ComputeEHVI:qehvi:exp}
\aqEHVI(\xcand) = \mathbb{E}\Bigl[\HVI{}(\bm f(\xcand))\Bigr]  =\int_{-\infty}^\infty \HVI{}(\bm f(\xcand)) d\bm{f}.
\end{align}
Since no known analytical form is known \citep{yang_palar19} for $q>1$ (or in the case of correlated outcomes), we estimate~\eqref{eqn:DqEHVI:ComputeEHVI:qehvi:exp} using MC integration with samples from the joint posterior $\{\bm{f}_t(\bm x_i)\}_{i=1}^q \sim \mathbb P \big(\bm f(\bm x_1),..., \bm f(\bm x_q) | \mathcal D\big), t=1, \ldots N$. Let $\bm z_{k, X_j, t}^{(m)} := \min \big[\bm u_k,\min_{\bm x' \in X_j}\bm f_t(\bm x')\big]$. Then,\\[-2ex]
\begin{equation}
\label{eqn:DqEHVI:ComputeEHVI:qehvi}
\hataqEHVI^N(\xcand) %=\mathbb{E}\Bigl[\HVI{}(\{\bm f(\bm{x}_i)\}_{i=1}^q)\Bigr] 
= \frac{1}{N} \! \sum_{t=1}^{N} \HVI{}(\bm f_t(\xcand)) =\frac{1}{N}\sum_{t=1}^{N} \sum_{k=1}^{K}\sum_{j=1}^q\sum_{X_j \in \mathcal X_j} \!\!\! (-1)^{j+1}\! \prod_{m=1}^M \!\!  \bigl[\bm z_{k, X_j, t}^{(m)} - l_k^{(m)}\bigr]_{\!+}
\end{equation}
\vspace{-1ex}

%We emphasize that our approach lets us parallelize computation over $M$ objectives, the powerset of $\xcand$, all $K$ hyper-rectangles in the decomposition of the nondominated space, and all $N$ MC samples. 

Provided that $\{S_k\}_{k=1}^K$ is an exact partitioning, \eqref{eqn:DqEHVI:ComputeEHVI:qehvi} is an \emph{exact} computation of \qEHVI{} up to the MC estimation error, which scales as $1/\sqrt{N}$ when using $iid$ MC samples regardless of the dimension of the search space \citep{emmerich2006}. In practice, we use randomized quasi MC methods \citep{caflisch1998monte} to reduce the variance and empirically observe low estimation error (see Figure~\ref{fig:mc_acqf} in the Appendix for a comparison of analytic \EHVI{} and (quasi-)MC-based \qEHVI{}).

\label{subsec:DqEHVI:complexity}
\qEHVI{} requires computing the volume of $2^q -1$ hyper-rectangles (the number of subsets of q) for each of $K$ hyper-rectangles and $N$ MC samples. Given posterior samples, the time complexity on a single-threaded machine is: $T_1 = O(MNK(2^q-1))$. In the two-objective case, $K = | \mathcal P| + 1$, but $K$ is super-polynomial in $M$ \citep{yang_emmerich2019}. The number of boxes required for a decomposition of the non-dominated space is unknown for $M\geq4$~\citep{yang_emmerich2019}. \qEHVI{} is agnostic to the partitioning algorithm used, and in \ref{appdx:subsec:approx_decomp}, we demonstrate using \qEHVI{} in higher-dimensional objective spaces using an approximate box decomposition algorithm \citep{couckuyt12}.
% \mb{I don't think it's necessary to go as deep into the hardware details in the paragraph below. It's probably sufficient to just state that this ``parallelizes extremely well'' and ``runs fast in practical problems'', then refer to the appendix for the details.}
%
Despite the daunting workload, the critical work path---the time complexity of the smallest non-parallelizable unit---is constant: $T_\infty = O(1)$.\footnote{As evident from~\eqref{eqn:DqEHVI:ComputeEHVI:qehvi}, the critical path consists of 3 multiplications and 5 summations.} On highly-threaded many-core hardware (e.g. GPUs), our formulation achieves tractable wall times in many practical scenarios: as is shown in Figure \ref{fig:qehi_comp_time} in the Appendix, the computation time is nearly constant with increasing $q$ until an inflection point at which the workload saturates the available cores. 
For additional discussion of both time and memory complexity of \qEHVI{} see Appendix~\ref{appdx:subsec:DqEHVI:complexity}.

%%%%%%%%%%%%%%%%%%%%%%%%
\subsection{Outcome Constraints}
\label{subsec:DqEHVI:OutcomeConstraints}

Our proposed \qEHVI{} acquisition function is easily extended to constraints on auxiliary outcomes. We consider the scenario where we receive observations of $M$ objectives $\bm f(\bm x) \in \mathbb{R}^M$ and $V$ constraints $\bm c^{(v)} \in \mathbb{R}^V$, all of which are assumed to be ``black-box''. We assume w.l.o.g. that $\bm c^{(v)}$ is feasible iff $\bm c^{(v)} \geq 0$. In the constrained optimization setting, we aim to identify the feasible Pareto set:
$\mathcal P_\text{feas} = \{\bm f(\bm x)~~s.t.~~ \bm c(\bm x) \geq \bm 0, ~\nexists ~~\bm x' :  \bm c(\bm x') \geq \bm 0, ~\bm f(\bm x') \succ \bm f(\bm x)\}$.
The natural improvement measure in the constrained setting is \emph{feasible} HVI, which we define for a single candidate point $\bm x$ as 
$% \begin{align}
% \label{eqn:qehvi:HVIc}
    \HVIc{}(\bm f(\bm x), \bm c(\bm x)) := \HVI{}[\bm f(\bm x)] \cdot \mathbbm{1}[\bm c(\bm x) \geq \bm 0]
% \end{align}
$.
Taking expectations, the constrained expected HV can be seen to be the HV weighted by the probability of feasibility. In Appendix~\ref{appdx:subsec:DqEHVI:OutcomeConstraints}, we detail how performing feasibility-weighting on the sample-level allows us to include such auxiliary outcome constraints into our MC formulation in a straightforward way.

%%%%%%%%%%%%%%%%%%%%%%%%%%%%%%%%%%%%%%%%%%%%%%%%%%%%%%%%%%%%%%%%
\section{Optimizing $q$-Expected Hypervolume Improvement}
\label{sec:optimizing}

%%%%%%%%%%%%%%%%%%%%%%%%%%%%%%%
\subsection{Differentiability} 
\label{subsec:optimizing:diffability}While an analytic formula for the gradient of \EHVI{} exists for the $M=2$ objective case in the unconstrained, sequential ($q=1$) setting, no such formula is known in 1) the case of $M>2$ objectives, 2) the constrained setting, and 3) for $q>1$. Leveraging the re-parameterization trick \citep{kingma2013reparam,wilson2017reparamacq} and auto-differentiation, we are able to automatically compute exact gradients of the MC-estimator \qEHVI{} in \emph{all} of the above settings, as well as the gradient of analytic \EHVI{} for $M\geq2$ (see Figure \ref{fig:mc_gradient} in the Appendix for a comparison of the exact gradients of \EHVI{} and the sample average gradients of \qEHVI{} for $M=3$).\footnote{Technically, $\min$ and $\max$ are only sub-differentiable, but  are known to be well-behaved \citep{wilson2017reparamacq}. In our MC setting with GP posteriors, \qEHVI{} is differentiable w.p. 1 if $\bm x$ contains no repeated points.}\footnote{For the constrained case, we replace the indicator with a differentiable sigmoid approximation.} 
%
% In Appendix~\ref{appdx:sec:MCApprox}, we show that the MC-approximated gradient provides a low-variance estimate of the analytical gradient, especially when using RQMC sampling.

%%%%%%%%%%%%%%%%%%%%%%%%%%%%%%%
\subsection{Optimization via Sample Average Approximation} \label{subsec:optimizing:saa}We show in Appendix~\ref{appdx:sec:Convergence} that if mean and covariance function of the GP are sufficiently regular, the gradient of the MC estimator~\eqref{eqn:DqEHVI:ComputeEHVI:qehvi} is an unbiased estimate of the gradient of the exact acquisition function~\eqref{eqn:DqEHVI:ComputeEHVI:qehvi:exp}.
To maximize \qEHVI{}, we could therefore directly apply stochastic optimization methods, as has previously been done for single-outcome acquisition functions~\citep{wilson2017reparamacq,wu16}. Instead, we opt to use the sample average approximation (SAA) approach from \citet{balandat2020botorch}, which allows us to employ deterministic, higher-order optimizers to achieve faster convergence rates. Informally (see Appendix~\ref{appdx:sec:Convergence} for the formal statement), if $\bm \hat{x}_N^* \in \argmax_{\bm x \in \mathcal X}\hataqEHVI^N(\bm x)$, we can show under some regularity conditions that, as $N\rightarrow \infty$, (i) $\hataqEHVI^N(\bm \hat{x}_N^*) \rightarrow \max_{x\in \mathcal X} \aqEHVI(\bm x)~~a.s.$, and (ii) $\text{dist}\bigl(\bm \hat{x}_N^*, \argmax_{\bm x \in \mathcal{X}}\aqEHVI(\bm x)\bigl) \rightarrow 0~~a.s.$. These results hold for any covariance function satisfying the regularity conditions, including such ones that model correlation between outcomes. In particular, our results do not require the outputs to be modeled by independent GPs.

Figure~\ref{fig:optim_comparison} demonstrates the importance of using exact gradients for efficiently and effectively optimizing \EHVI{} and \qEHVI{} by comparing the following optimization methods: L-BFGS-B with exact gradients, L-BFGS-B with gradients approximated via finite differences, and CMA-ES (without gradients). The cumulative time spent optimizing the acquisition function is an order of magnitude less when using exact gradients rather than approximate gradients or zeroth order methods.\\[-4ex]%
\begin{figure}[ht]
    \centering
    \begin{subfigure}{.49\textwidth}
        \centering
        \includegraphics[width=\linewidth]{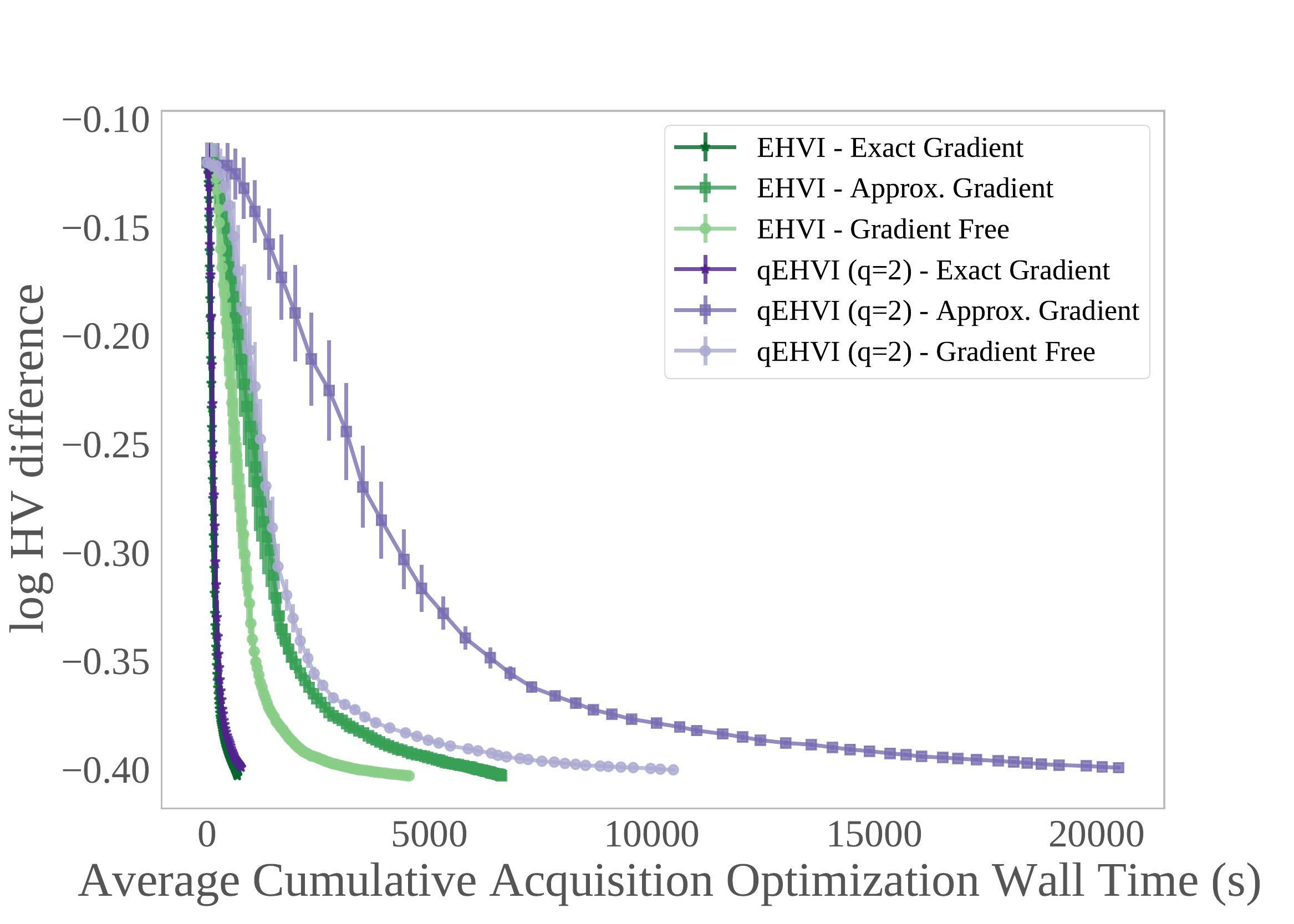}
        \vspace{-3ex}
        \subcaption{\label{fig:optim_comparison}}
        \end{subfigure} %
    \begin{subfigure}{.49\textwidth} 
    \centering
        \includegraphics[width=\linewidth]{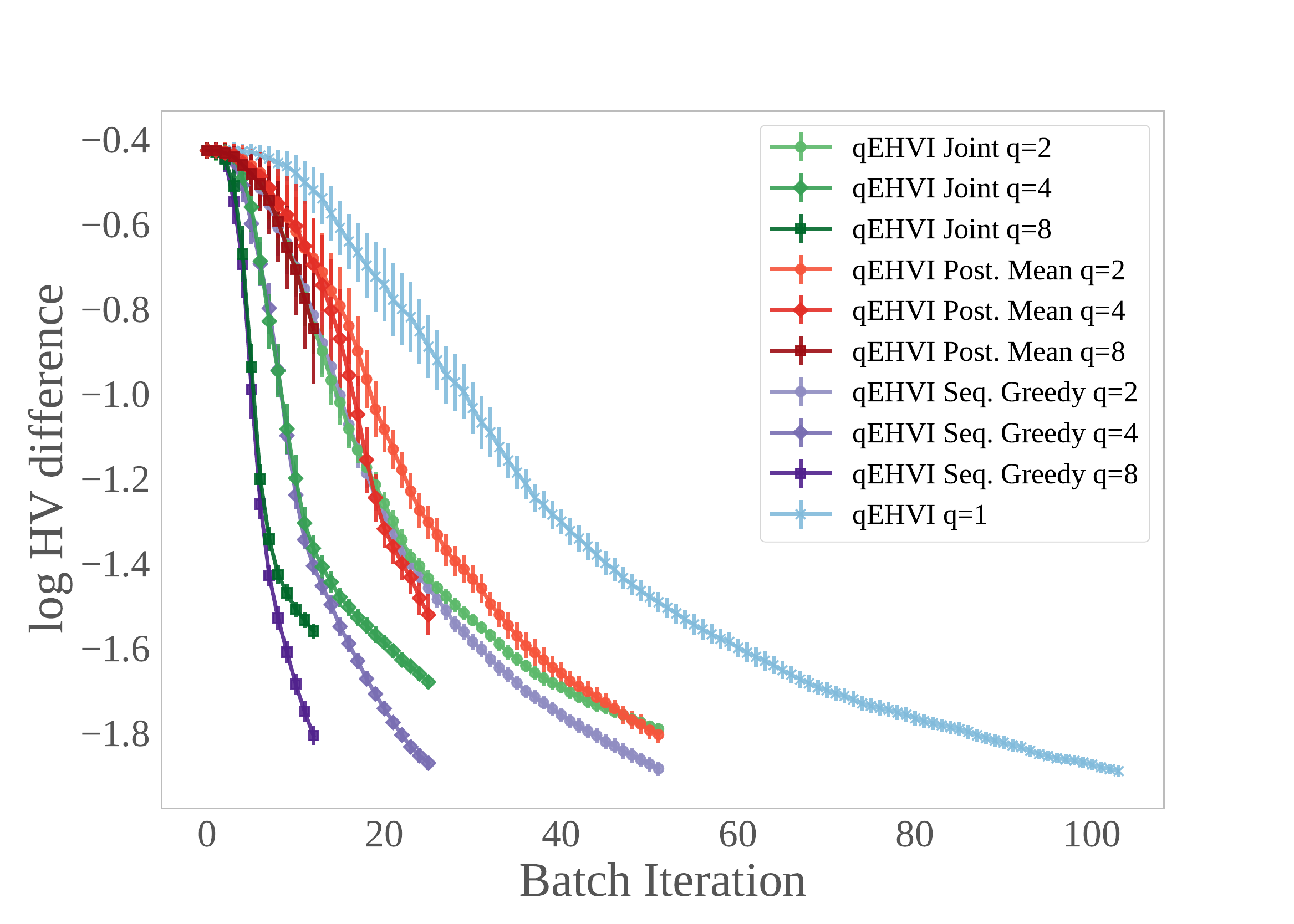}
        \vspace{-3ex}
        \subcaption{\label{pm_impute}}
    \end{subfigure}
    \caption{(a) A comparison of \EHVI{} and \qEHVI{} ($q=2$) optimized with L-BFGS-B using exact gradients, L-BFGS-B using gradients approximated using finite differences, and CMA-ES, a gradient-free method. (b) A comparison of joint optimization, sequential greedy optimization with proper integration at the pending points, and sequential greedy using the posterior mean. Both plots show optimization performance on a DTLZ2 problem ($d=6, M=2$) with a budget of 100 evaluations (plus the initial quasi-random design). We report means and 2 standard errors across 20 trials.\\[-4ex]}%
\end{figure}%
%%%%%%%%%

\subsection{Sequential Greedy and Joint Batch Optimization}  \label{subsec:optimizing:seq_greedy}Jointly optimizing $q$ candidates increases in difficulty with~$q$ because the problem dimension is~$dq$. An alternative is to sequentially and greedily select candidates and condition the acquisition function on the previously selected pending points when selecting the next point \citep{wilson2018maxbo}. Using a submodularity argument similar to that in \citet{wilson2017reparamacq}, the sequential greedy approximation of \qEHVI{} enjoys regret of no more than $\frac{1}{e}\alpha_\text{\qEHVI}^*$, where $\alpha_\text{\qEHVI}^*$ is the optima of $\alpha_\text{\qEHVI}$ \citep{Fisher1978} (see Appendix~\ref{sec:seq_greedy_details}).

% In contrast with many existing methods for multi-objective BO which impute the unobserved outcomes for the pending points $\bm{x}_1, ..., \bm{x}_{i-1}$ with a point estimate (e.g. the posterior mean) \citep{garridomerchn2020parallel}, \qEHVI{} integrates over the unobserved outcomes for the pending points $\bm{x}_1, ..., \bm{x}_{i-1}$ when selecting candidate $\bm x_i$ by using MC samples from the joint posterior $\mathbb P \big(\bm f(\bm x_1),..., \bm f(\bm x_i) | X, Y\big)$, which is required for the error bound above. Typically, using the posterior mean sacrifices the theoretical error bound on the sequential greedy approximation because the bound relies on computing the acquisition function of the selected candidate set at each step, which requires accounting with the uncertainty at the unobserved outcomes in many common acquisition functions (often in the form of an expectation).  
Although sequential greedy approaches have been considered for many acquisition functions \citep{wilson2018maxbo}, no previous work has proposed a proper sequential greedy approach (with integration over the posterior) for parallel \EHVI{}, as this would require computing the Pareto front under each sample $\bm f_t$ from the joint posterior before computing the hypervolume improvement. These operations would be computationally expensive for even modest $N$ and non-differentiable.  \qEHVI{} avoids determining the Pareto set for each sample by using inclusion-exclusion principle to compute the joint \HVI{} over the pending points $\bm x_1, ..., \bm x_{i-1}$ and the new candidate $\bm x_i$ for each MC sample. Figure \ref{pm_impute} empirically demonstrates the improved optimization performance from properly integrating over the unobserved outcomes rather than using the posterior mean or jointly optimizing the $q$ candidates.

%%%%%%%%%%%%%%%%%%%%%%%%%%%%%%%%%%%%%%%%%%%%%%%%%%%%%%%%%%%%%%%%
\section{Benchmarks}
\label{sec:Experiments}

We empirically evaluate \qEHVI{} on synthetic and real world optimization problems. We compare \qEHVI{}\footnote{\label{botorch_code}Acquisition functions are available as part of the open-source library BoTorch \citep{balandat2020botorch}. Code is available at \url{https://github.com/pytorch/botorch}.} against existing state-of-the-art methods including SMS-EGO\footnote{\label{spearmint_code}We leverage existing implementations from the Spearmint library. The code is available at \url{https://github.com/HIPS/Spearmint/tree/PESM}.}, PESMO\footnotemark[\getrefnumber{spearmint_code}], TS-TCH\footnotemark[\getrefnumber{botorch_code}], and analytic EHVI \citep{yang_emmerich2019} with gradients\footnotemark[\getrefnumber{botorch_code}]. Additionally, we compare against a novel extension of ParEGO \citep{parego} that supports parallel evaluation and  constraints (neither of which have been done before to our knowledge); we call this method \qParego{}\footnotemark[\getrefnumber{botorch_code}]. Additionally, we include a quasi-random baseline that selects candidates from a scrambled Sobol sequence. See Appendix~\ref{sec:algo_details} for details on all baseline algorithms.
\\[-3.5ex]
%\subsection{Synthetic Benchmarks}
\paragraph{Synthetic Benchmarks}

We evaluate optimization performance on four benchmark problems in terms of log hypervolume difference, which is defined as the difference between the hypervolume of the true (feasible) Pareto front and the hypervolume of the approximate (feasible) Pareto front based on the observed data; in the case that the true Pareto front is unknown (or not easily approximated), we evaluate the hypervolume indicator. All references points and search spaces are provided in Appendix \ref{sec:problem_details}. For synthetic problems, we consider the Branin-Currin problem ($d=2, M=2$, convex Pareto front) \citep{belakaria2019} and the C2-DTLZ2 ($d=12, M=2, V=1$, concave Pareto front), which is a standard constrained benchmark from the MO literature \citep{deb2019} (see Appendix~\ref{appdx:sec:extra_synthetic_experiments} for additional synthetic benchmarks).\\[-1.5ex]

\begin{figure}[t]
    \centering
    \begin{subfigure}{.49\textwidth}
        \centering
\includegraphics[width=\linewidth]{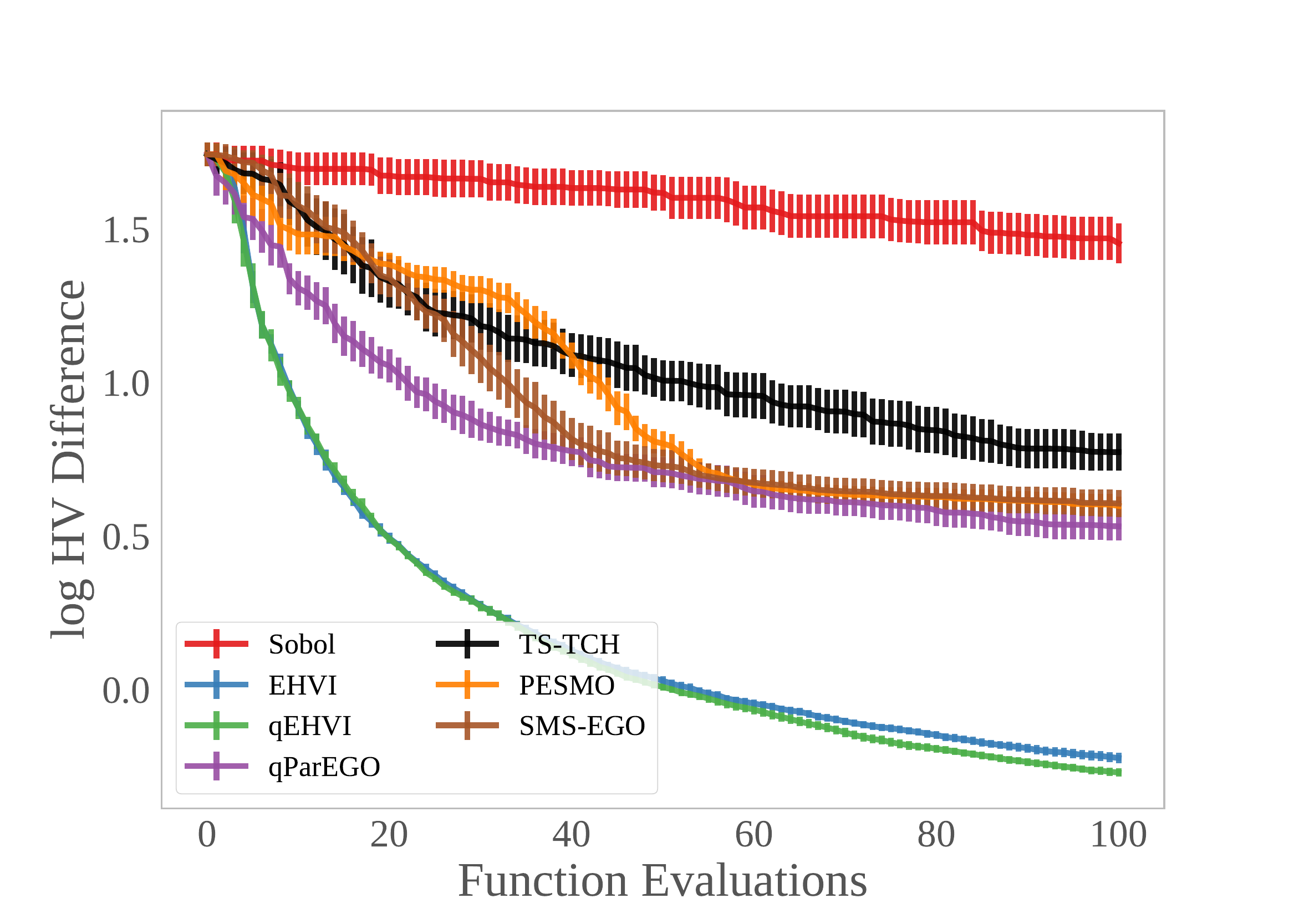}
        \vspace{-4ex}
        \subcaption{\label{fig:abr_hv}}
    \end{subfigure} %
    \vspace{-2ex}
    \begin{subfigure}{.49\textwidth}    
        \centering
        \includegraphics[width=\linewidth]{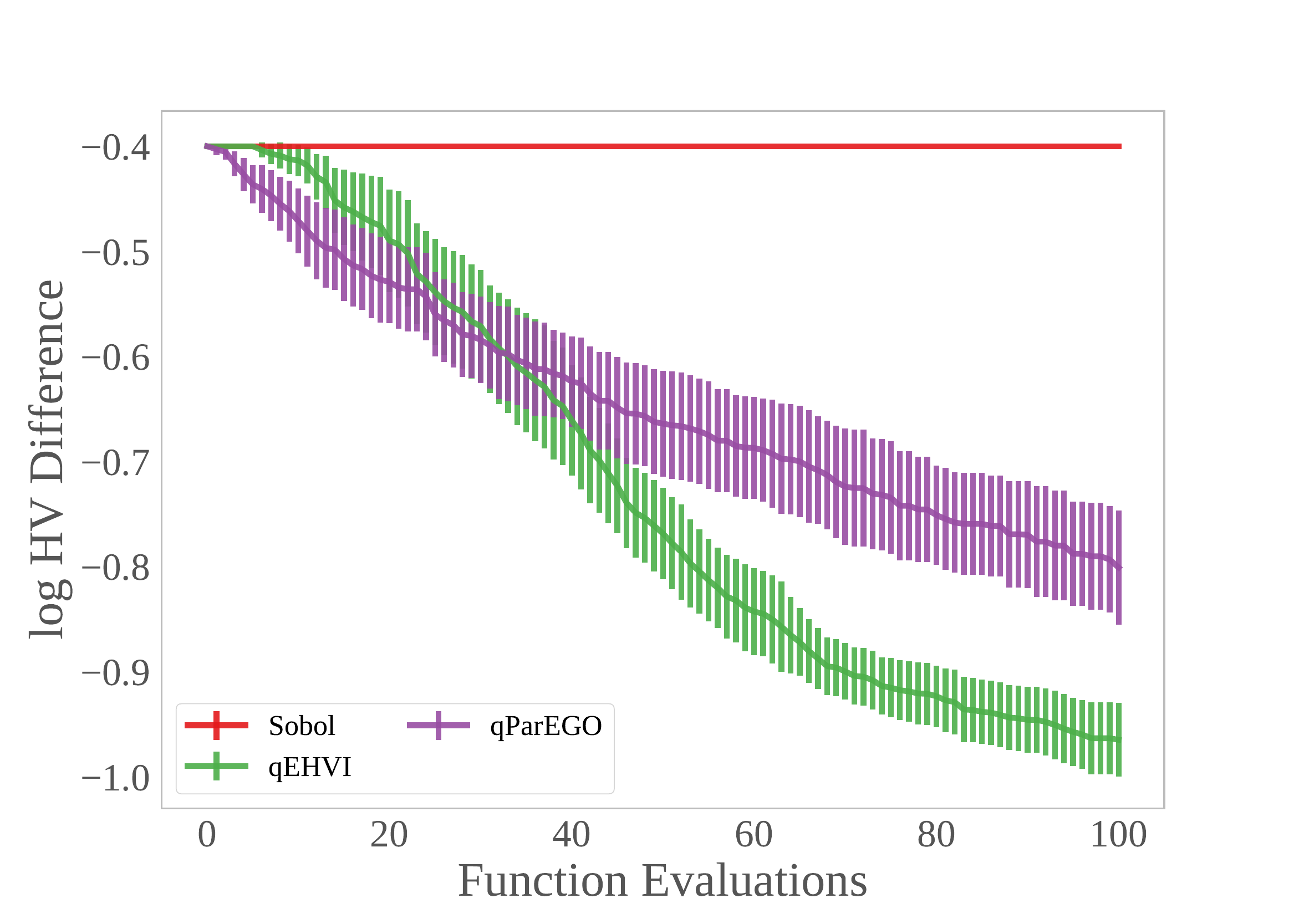}
        \vspace{-4ex}
        \subcaption{\label{fig:constrained_hv}}
    \end{subfigure}
    % \vspace{-2ex}
    \begin{subfigure}{.49\textwidth}
        \centering
        \includegraphics[width=\linewidth]{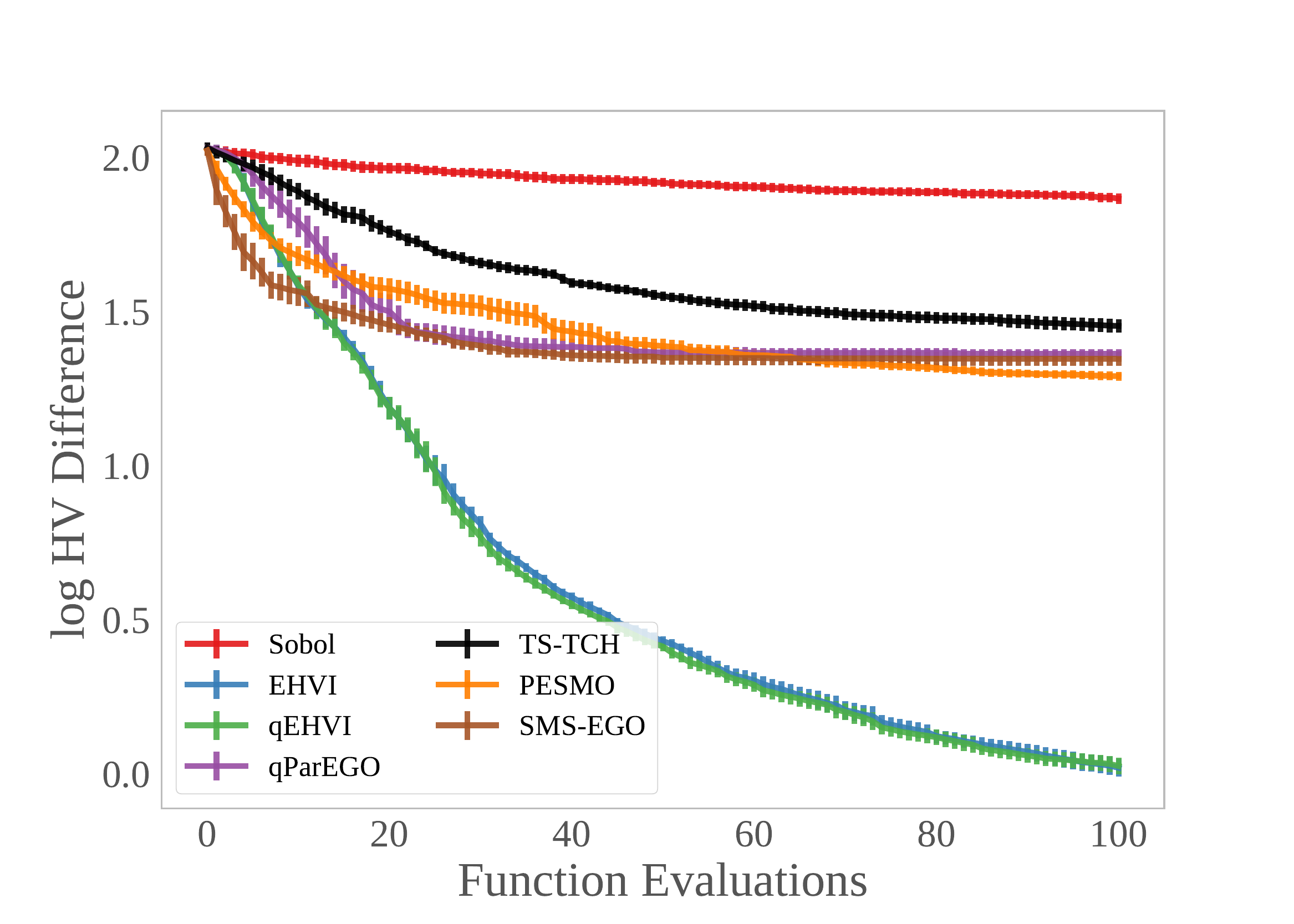}
        \vspace{-4ex}
        \subcaption{\label{fig:daisy_hv}}
    \end{subfigure} %
    \begin{subfigure}{.49\textwidth}    
        \centering
        \includegraphics[width=\linewidth]{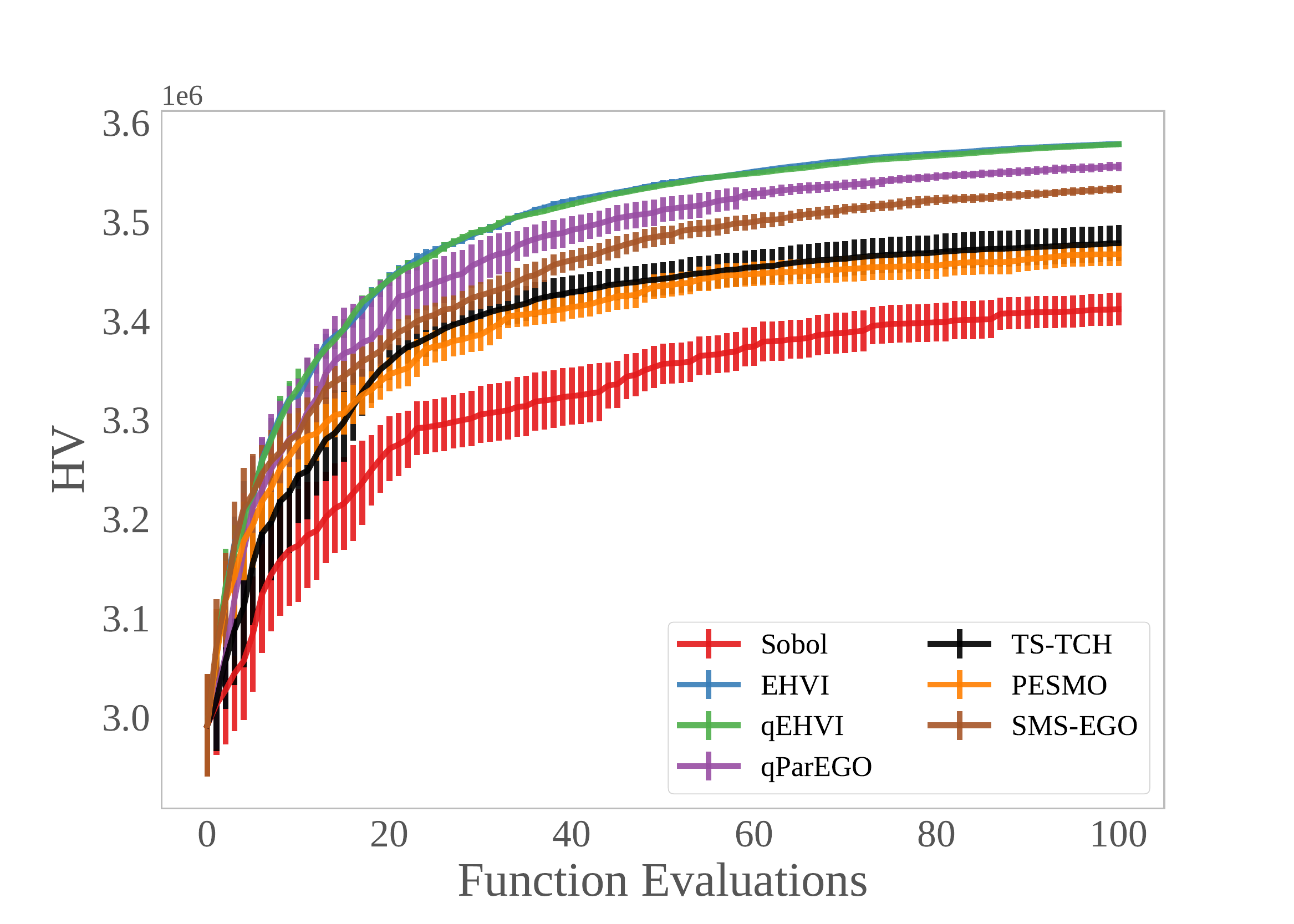}
        \vspace{-4ex}
        \subcaption{\label{fig:daisy_Pareto}}
    \end{subfigure}
    \vspace{-1ex}
    \caption{\label{fig:opt_results}Sequential optimization performance on (a) on the Branin-Currin problem ($q=1$), (b) the C2-DTLZ2 problem, (c) the vehicle crash safety problem ($q=1$), and (d) the ABR control problem ($q=1$). We report the means and 2 standard errors across 20 trials.\\[-4ex]}
\end{figure}

%\subsection{Real-World Experiments}
{\bf{Real-World Benchmarks}}
\\[-0.1ex]
\emph{Structural Optimization in Automobile Safety Design} (\textsc{VehicleSafety}): Vehicle crash safety is an important consideration in the structural design of automobiles. A lightweight car is preferable because of its potentially lower manufacturing cost and better fuel economy, but lighter material can fare worse than sturdier alternatives in a collision, potentially leading to increased vehicle damage and more severe injury to the vehicle occupants \citep{vehicle2005}. We consider the problem designing the thickness of 5 reinforced parts of the frontal frame of a vehicle that considerably affect crash safety. The goal is to minimize: 1) the \textit{mass} of the vehicle; 2) the \textit{collision acceleration} in a full frontal crash---a proxy for bio-mechanical trauma to the vehicle occupants from the acceleration; and 3) the \textit{toe-board intrusion}---a measure of the most extreme mechanical damage to the vehicle in an off-frontal collision \citep{liao2008}. For this problem, we optimize the surrogate from \citet{tanabe2020}.
% \textbf{Policy Optimization for Robot Locomotion}: The goal in this problem is optimize a policy controlling the walking gait of a simulated hexapod robot in order to maximize its average speed and minimize energy it expends.\footnote{A scalarized variant of this problem considered by \citet{Letham2019Re}} Identifying speed-energy trade-off is an important and practical concern when designing control policies because robots often has often have real energy restrictions \citep{yang_wang_2018, ariizumi}. Moreover, sample efficiency critical in robotics as evaluating policies is time-consuming and a poor control policy can damage the robot. The specific problem considered here requires tuning the parameters of a Central Pattern Generator controller \citep{cpg08} for Daisy hexapod robot \citep{hebi}, with 3 motors per leg, on the publicly available PyBullet simulator \citep{coumans2019}. We consider policies with a tripod gait: 3 legs are synchronized and out of phase from the other 3 legs \citep{Letham2019Re}. The policy is controlled by 11 parameters corresponding to the common amplitude, offset and frequency of the joints. The two objectives are \emph{average speed}, measured as distance traveled over the length of the trajectory and \emph{energy expended}, which is not directly reported by the simulator and therefore approximated using joint velocity as in previous work \citep{Letham2019Re}.

\textit{Policy Optimization for Adaptive Bitrate Control} (\textsc{ABR}): Many web services adapt video playback quality adaptively based on the receiver's network bandwith to maintain steady, high quality stream with minimal stalls and buffer periods \citep{Mao2019RealworldVA}. Previous works have proposed controllers with different scalarized objective functions \citep{mao17}, but in many cases, engineers may prefer to learn the set of optimal trade-offs between their metrics of interest, rather than specifying a scalarized objective in advance. In this problem, we decompose the objective function proposed in \citet{mao17} into its constituent metrics and optimize 4 parameters of an ABR control policy on the Park simulator \citep{Mao2019ParkAO} to maximize video quality (bitrate) and minimize stall time. See Appendix \ref{sec:problem_details} for details.
\begin{table*}[t!]
\centering
\caption{\label{table:latency:cpu} Acquisition Optimization wall time in seconds on a CPU (2x Intel Xeon E5-2680 v4 @ 2.40GHz) and a GPU (Tesla V100-SXM2-16GB). We report the mean and 2 standard errors across 20 trials. NA indicates that the algorithm does not support constraints. %We only show \textsc{TS-TCH} for $q=8$ because the average wall time differs by no more 0.02 regardless of $q$.
}
\begin{small}
\begin{sc}
\begin{tabular}{lcccc}
\toprule
\textbf{CPU} & BraninCurrin & C2DTLZ2 & ABR & VehicleSafety\\
\midrule
PESMO (\textit{q}=1) & $249.16 ~(\pm 19.35)$&NA& $214.16 ~(\pm 18.38)$& $492.64 ~(\pm 58.98)$\\
SMS-EGO (\textit{q}=1) & $146.1 ~(\pm 8.57)$& NA&$89.54 ~(\pm 5.79)$& $115.11 ~(\pm 8.21)$\\
TS-TCH (\textit{q}=1) & $2.82 ~(\pm 0.03)$ & NA & $17.22 ~(\pm 0.04)$ & $47.46 ~(\pm 0.05)$\\
\qParego{} (\textit{q}=1) & $1.56 ~(\pm 0.16)$&$4.01 ~(\pm 0.77)$& $7.47 ~(\pm 0.67)$& $1.74 ~(\pm 0.27)$\\
\EHVI{} (\textit{q}=1) & $3.04 ~(\pm 0.16)$&NA& $2.48 ~(\pm 0.19)$& $15.18 ~(\pm 2.24)$\\
\qEHVI{} (\textit{q}=1) & $3.63 ~(\pm 0.23)$&$5.4 ~(\pm 1.18)$& $6.15 ~(\pm 0.71)$& $67.54 ~(\pm 10.45)$\\
\midrule
\textbf{GPU} & BraninCurrin & C2DTLZ2 & ABR & VehicleSafety\\
\midrule
TS-TCH (\textit{q}=1) & $0.07 ~(\pm 0.00)$ & NA & $0.16 ~(\pm 0.00)$ & $0.32 ~(\pm 0.0)$\\
TS-TCH (\textit{q}=2) & $0.07 ~(\pm 0.00)$ & NA & $0.15 ~(\pm 0.00)$ & $0.34 ~(\pm 0.01)$\\
TS-TCH (\textit{q}=4) & $0.09 ~(\pm 0.01)$ & NA & $0.15 ~(\pm 0.00)$ & $0.31 ~(\pm 0.01)$\\
TS-TCH (\textit{q}=8) & $0.08 ~(\pm 0.00)$ & NA & $0.16 ~(\pm 0.00)$ & $0.34 ~(\pm 0.01)$\\
\qParego{} (\textit{q}=1) & $3.2 ~(\pm 0.37)$& $3.85 ~(\pm 0.91)$ & $9.64 ~(\pm 0.96)$& $3.44 ~(\pm 0.51)$\\
\qParego{} (\textit{q}=2) & $7.12 ~(\pm 0.81)$& $12.1 ~(\pm 2.77)$ & $21.19 ~(\pm 1.53)$& $7.32 ~(\pm 0.97)$\\
\qParego{} (\textit{q}=4) & $15.34 ~(\pm 1.69)$& $39.71 ~(\pm 7.40)$ & $35.46 ~(\pm 2.32)$& $17.2 ~(\pm 2.29)$\\
\qParego{} (\textit{q}=8) & $32.11 ~(\pm 4.14)$&$99.58 ~(\pm 15.20)$ & $72.52 ~(\pm 5.04)$& $39.72 ~(\pm 7.13)$\\
\EHVI{} (\textit{q}=1) & $4.53 ~(\pm 0.23)$& NA & $6.82 ~(\pm 0.55)$& $8.95 ~(\pm 0.64)$\\
\qEHVI{} (\textit{q}=1) & $5.98 ~(\pm 0.28)$ & $3.36 ~(\pm 0.94)$ & $7.71 ~(\pm 0.67)$& $10.43 ~(\pm 0.64)$\\
\qEHVI{} (\textit{q}=2) & $11.37 ~(\pm 0.56)$& $21.56 ~(\pm 3.45)$ & $18.32 ~(\pm 1.48)$& $17.67 ~(\pm 1.54)$\\
\qEHVI{} (\textit{q}=4) & $25.29 ~(\pm 1.51)$&$89.18 ~(\pm 10.86)$ & $44.44 ~(\pm 3.53)$& $54.25 ~(\pm 4.17)$\\
\qEHVI{} (\textit{q}=8) & $102.46 ~(\pm 9.22)$&$215.74 ~(\pm 15.85)$ & $100.64 ~(\pm 7.22)$& $255.72 ~(\pm 23.73)$\\
\bottomrule
\end{tabular}
\end{sc}
\end{small}
\vspace{-2ex}
\end{table*}

\subsection{Results}
Figure \ref{fig:opt_results} shows that \qEHVI{} outperforms all baselines in terms of sequential optimization performance on all evaluated problems.
% Furthermore, 
Table \ref{table:latency:cpu} shows that \qEHVI{} achieves wall times that are an order of magnitude smaller than those of PESMO on a CPU in sequential optimization, and maintains competitive wall times even relative to \qParego{} (which has a significantly smaller workload) for large $q$ %batches
on a GPU. TS-TCH has by far the fastest wall time, but this comes at the cost of inferior optimization performance.

Figure~\ref{fig:parallelism} illustrates optimization performance of \emph{parallel} acquisition functions for varying batch sizes. Increasing the level of parallelism leads to faster convergence for all algorithms (Figure~\ref{fig:parallelism}a). In contrast with other algorithms, \qEHVI{}'s sample complexity does not deteriorate substantially when high levels of parallelism are used (Figure~\ref{fig:parallelism}b).
%Sample complexity without deteriorating the optimization performance of \qEHVI{} (Figure~\ref{fig:parallelism}b), where as performance does degrade for \textsc{TS-TCH} and \qParego{}.
%For example, for $q=8$ parallel evaluations, \qEHVI{} has optimality gap when comparing with the $q=1$ case, relative to both TS-TCH and \qParego{} (note that the y-axis is on a log HV difference scale, implying that gaps are smaller for smaller values of the y axis).
\begin{figure}[ht]
\tiny
    \centering
    \begin{subfigure}{.49\textwidth}
        \centering
\includegraphics[width=\linewidth]{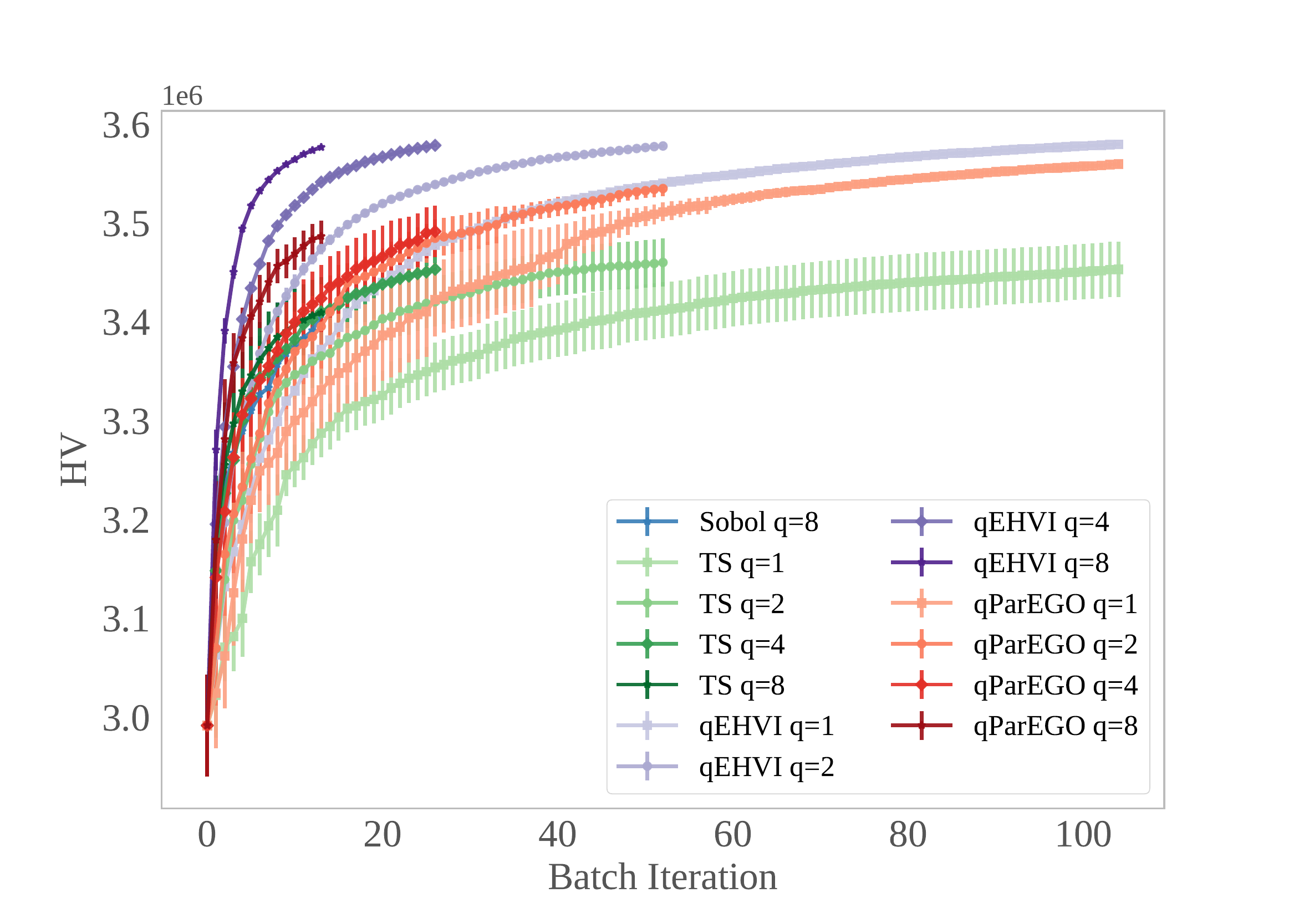}
        \subcaption{\label{fig:parallelism_1}}
    \end{subfigure} %
    \begin{subfigure}{.49\textwidth}    
        \centering
        \includegraphics[width=\linewidth]{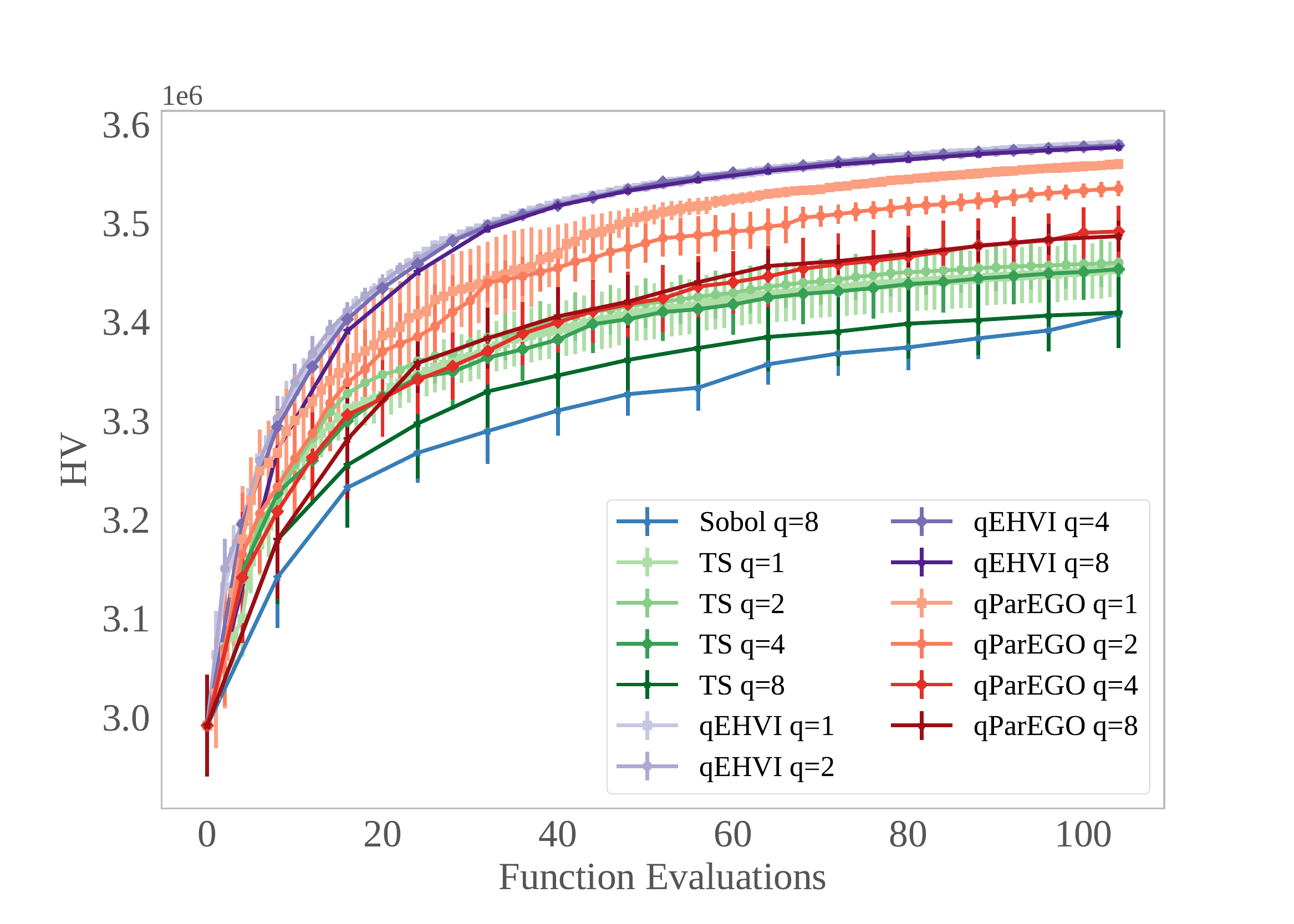}
        \subcaption{\label{fig:parallelism_2}}
    \end{subfigure}
    \vspace{-2ex}
    \caption{\label{fig:parallelism} Parallel optimization performance on the \textsc{ABR} problem with varying batch sizes ($q$) by (a) \emph{batch BO iterations} and (b) \emph{function evaluations}.}
\label{fig:q_anytime_extra2}
\vspace{-6ex}
\end{figure}

\section{Discussion}
We present a practical and efficient acquisition function, \qEHVI{}, for parallel, constrained multi-objective Bayesian optimization. Leveraging differentiable programming, modern parallel hardware, and the Sample Average Approximation, we efficiently optimize \qEHVI{} via quasi second-order methods and provide theoretical convergence guarantees for our approach. Empirically, we demonstrate that our method out-performs state-of-the-art multi-objective Bayesian optimization methods.

One limitation of our approach is that it currently assumes noiseless observations, which, to our knowledge, is the case with all formulations of \EHVI{}. Integrating over the uncertainty around the previous observations~\citep{letham2019noisyei} by using MC samples over the new candidates and the training points, one may be able to account for the noise.%applying the inclusion-exclusion principle to the $N+q$ points, would be computationally infeasible in most cases. Such an integration would be equivalent to noiseless \qEHVI{} computation with a batch size $N + q$, which would be prohibitively expensive since computation scales exponentially with the batch size. 
% "Fairness" of comparisons (e.g. CPU vs GPU. Although, practically speaking though, I don't know of multi-objective BO package with GPU support (other than GPFlowOpt which implements Probability of Hypervolume Improvement), so there really isn't anything off the shelf to compare against.
Another limitation of \qEHVI{} is that its scalability is limited the partitioning algorithm, precluding its use in high-dimensional objective spaces. More scalable partitioning algorithms,
%either approximate algorithms
either approximate algorithms (e.g. the algorithm proposed by \citet{couckuyt12}, which we examine briefly in Appendix \ref{appdx:subsec:approx_decomp})
or more efficient exact algorithms that result in fewer disjoint hyper-rectangles (e.g. \citep{LACOUR2017347, DACHERT2017, yang2019}), will improve the scalability and computation time of of \qEHVI{}. We hope this work encourages researchers to consider more improvements from applying modern computational paradigms and tooling to Bayesian optimization.
 
\newpage
 
\section{Statement of Broader Impact}
Optimizing a single outcome commonly comes at the expense of other secondary outcomes. In some cases, decision makers may be able to form a scalarization of their objectives in advance, but in the researcher's experience, formulating such trade-offs in advance is difficult for most.  Improvements to the optimization performance and practicality of multi-objective Bayesian optimization have the potential to allow decision makers to better understand and make more informed decisions across multiple trade-offs.  We expect these directions to be particularly important as Bayesian optimization is increasingly used for applications such as recommender systems~\cite{letham2019mtbo}, where auxiliary goals such as fairness must be accounted for.  Of course, at the end of the day, exactly what objectives decision makers choose to optimize, and how they balance those trade-offs (and whether that is done in equitable fashion) is up to the individuals themselves. 

\section*{Acknowledgments}
We would like to thank Daniel Jiang for helpful discussions around our theoretical results.

%\qEHVI{} has the potential to improve multi-objective optimization performance in a variety of applications including those considered in our experiments. For example, we demonstrate that \qEHVI{} can produce better sets of Pareto optimal structural designs for automobiles with respect to fuel efficiency, vehicle durability, and occupant safety. Our work, and Bayesian optimization more generally, is agnostic to the particular application it is used in. This means that there is large potential for broad positive impact on society. At the same time, this agnosticism also means that it

%\qEHVI{} yields improved Pareto optimal solution sets, but it is up to the decision maker to make an ethically responsible decision when choosing the final solution.

\small
\bibliographystyle{plainnat}
\bibliography{neurips_2020}

\clearpage

%%%%%%%%%%%%%%%%%%%%%%%%%%%%%%%%%%%%%%%%%%%%%%%%%%%%%%%%%%%%%%%%%%%%%%%%%%%%%%%%
\appendix

\begin{center}
\hrule height 4pt
\vskip 0.25in
\vskip -\parskip
    {\LARGE\bf  Appendix to:\\[2ex] \papertitle}
\vskip 0.29in
\vskip -\parskip
\hrule height 1pt
\vskip 0.2in%
\end{center}

%%%%%%%%%%%%%%%%%%%%%%%%%%%%%%%%%%%%
\section{Derivation of $q$-Expected Hypervolume Improvement}
\label{appdx:sec:qEHVIDerivation}

%%%%%%%%%%%%%%%%
\subsection{Hypervolume Improvement via the Inclusion-Exclusion Principle}
\label{appdx:subsec:qEHVIDerivation:HVI}

The hypervolume improvement of $\bm{f}(\bm{x})$ within the hyper-rectangle $S_k$ is the volume of $S_k \cap \Delta(\{\bm f(\bm x)\}, \mathcal P, \bm r)$ and is given by:
\begin{align*}
    \HVI{}_k\big(\bm{f}(\bm{x}), \bm l_k, \bm u_k\big) = \lambda_M\big(S_k \cap \Delta(\{\bm f(\bm x)\}, \mathcal P, \bm r)\big)
    = \prod_{m=1}^M
\big[z_k^{(m)} - l_k^{(m)}\big]_{+},
\end{align*}
where $u_k^{(m)}, l_k^{(m)}, f^{(m)}(\bm x)$, and $z_k^{(m)}$ denote the $m^{\text{th}}$ component of the corresponding vector and $[\cdot]_{+}$ denotes the $\min(\cdot, 0)$ operation.
% The first $\min$ ensures that we do not go past the right boundary of the rectangle specified by $\bm{r}_k$ and clamping the minimum value ensures that the contribution to the area is zero if the new point is dominated by a point on the current Pareto front. The hypervolume improvement is the overlapping volume between the each of the hyper-rectangles ($S_k$) in the non-dominated hypervolume and the new hypervolume under the current point. 
Summing over all $S_k$ gives the total hypervolume improvement:
\begin{align*}
    \HVI{}\big(\bm{f}(\bm{x})\big) &=\sum_{k=1}^{K}\HVI{}_k\big(\bm{f}(\bm{x}), \bm l_k, \bm u_k\big)\\
    &=\sum_{k=1}^{K}\lambda_M\big(S_k \cap \Delta(\{\bm f(\bm x)\}, \mathcal P, \bm r)\big)\\
    &=\sum_{k=1}^{K}\prod_{m=1}^M
\big[z_k^{(m)} - l_k^{(m)}\big]_{+}.
\end{align*}

We can extend the \HVI{} computation to the $q>1$ case using the inclusion-exclusion principle.

\begin{principle}{\textbf{The inclusion-exclusion principle} \normalfont{\citep{silva1854, sylvester1883, cerasoli}}}
\label{principle:1}
Given a finite measure space $(B, \mathcal A, \mu)$ and a finite sequence of potentially empty or overlapping sets  $\{A_i\}_i=1^n$ where $A_i \in \mathcal A$ and $\mu(B) < \infty$,
then,
\begin{equation*}
\lambda_M \bigg( \bigcup_{i=1}^p A_i \bigg) = 
\sum_{j=1}^p (-1)^{j+1}\sum_{1 \leq i_1 \le \ldots \le i_j \leq p} \lambda_M \big( A_{i_1} \cap ... \cap A_{i_j} \big)
\end{equation*}
\end{principle}

In the context of computing the joint \HVI{} of $q$ new points$\{\bm f(\bm x_i)\}_{i=1}^q$, each subset $A_i$ for $i=1, \dots, q$ is the set of points contained in $\Delta(\{\bm f(\bm x_i)\}, \mathcal P, \bm r)$ --- independently of the other $q-1$ points. $\lambda_M(A_i)$ is the hypervolume improvement from the new point $\bm f(\bm x_i)$: $\lambda_M( A_i) = \HVI{}(\bm f(\bm x_i))$. The union of these subsets is the set of points in the new space dominated by the $q$ new points: $\bigcup_{i=1}^q A_i = \bigcup_{i=1}^q \Delta(\{\bm f(\bm x_i)\}, \mathcal P, \bm r)$. The hypervolume of $\bigcup_{i=1}^q \Delta(\{\bm f(\bm x_i)\}, \mathcal P, \bm r)$ is the hypervolume improvement from the $q$ new points:
\begin{align*}
\HVI{}(\{\bm f(\bm x_i)\}_{i=1}^q) 
&= \lambda_M \bigg( \bigcup_{i=1}^q A_i \bigg) \\
&= \sum_{j=1}^q (-1)^{j+1}\sum_{1 \leq i_1 \le \ldots \le i_j \leq q} \lambda_M\big( A_{i_1} \cap \dots \cap A_{i_j} \big)
\end{align*}
To compute $\lambda_M ( A_{i_1} \cap \dots \cap A_{i_j} )$, we partition the space covered by $A_{i_1} \cap \dots \cap A_{i_j}$ across the $K$ hyper-rectangles $\{S_k\}_{k=1}^K$ and compute the hypervolume of the overlapping space of $A_{i_1} \cap \dots \cap A_{i_j}$ with each $S_k$ independently. Since $\{S_k\}_{k=1}^K$ is a disjoint partition, summing over $K$ gives the hypervolume of $A_{i_1} \cap \dots \cap A_{i_j}$:
$$\lambda_M\big( A_{i_1} \cap \dots \cap A_{i_j} \big) = \sum_{k=1}^{K} \lambda_M\big( S_k \cap A_{i_1} \cap \dots \cap A_{i_j} \big)$$

This has two advantages. First, the new dominated space $A_i$ can be a non-rectangular polytope, but the intersection $A_i \cap S_k$ is a \emph{rectangular} polytope, which simplifies computation of overlapping hypervolume. Second, the vertices defining the hyper-rectangle encapsulated by $S_k \cap A_{i_1} \cap \dots \cap A_{i_j} $ are easily derived. The lower bound is simply the $\bm l_k$ lower bound of $S_k$ and the upper bound is the component-wise minimum $\bm z_{k, i_1, ... i_j} = \min \big[\bm u_k,\bm f(\bm x_{i_1}), \ldots, \bm f(\bm x_{i_j})\big]$. 

Importantly, this is computationally tractable because this specific approach enables parallelizing computation across all intersections of subsets $A_{i_1} \cap \dots \cap A_{i_j}$ for $1\leq i_j \leq \ldots \leq i_j \leq q$ and across all $K$ hyper-rectangles.  Explicitly, the \HVI{} is computed as:

\begin{align*}
\HVI{}(\{\bm f(\bm x_i)\}_{i=1}^q)
&= \lambda_M\bigg( \bigcup_{i=1}^p A_i \bigg) \\
&= \sum_{j=1}^q \sum_{1 \leq i_1 \le \ldots \le i_j \leq q}(-1)^{j+1}  \lambda_M\big( A_{i_1} \cap \dots \cap A_{i_j} \big)\\
&=\sum_{k=1}^{K}\sum_{j=1}^q \sum_{1 \leq i_1 \le \ldots \le i_j \leq q} (-1)^{j+1}\lambda_M\big( S_k \cap A_{i_1} \cap \dots \cap A_{i_j} \big)\\
&= \sum_{k=1}^{K}\sum_{j=1}^q \sum_{1 \leq i_1 \le \ldots \le i_j \leq q} (-1)^{j+1} \lambda_M\big(S_k \cap \Delta(\{\bm f(\bm x_{i_1})\}, \mathcal P, \bm r) \cap  \ldots \cap \Delta(\{\bm f(\bm x_{i_j})\}, \mathcal P, \bm r)\big)\\
&=\sum_{k=1}^{K}\sum_{j=1}^q \sum_{1 \leq i_1 \le \ldots \le i_j \leq q} (-1)^{j+1} \prod_{m=1}^M\big[z_{k, i_1, ... i_j}^{(m)} - l_k^{(m)}\big]_{+}\\
&=\sum_{k=1}^{K}\sum_{j=1}^q \sum_{X_j \in \mathcal X_j} (-1)^{j+1} \prod_{m=1}^M\big[z_{k, X_j}^{(m)} - l_k^{(m)}\big]_{+}
\end{align*}
where $\mathcal \mathcal X_j$ is the superset all subsets of $\xcand$ of size $j$: $\mathcal X_j = \{X_j \subset \xcand : \vert X_j \vert = j\}$ and $z_{k, X_j}^{(m)} = z_{k, i_1, ... i_j}^{(m)}$ for $X_j = \{\bm x_{i_1}, ..., \bm x_{i_j}\}$.

%%%%%%%%%%%%%%%%%%%%%%%%
\subsection{Computing \emph{Expected} Hypervolume Improvement}
\label{appdx:subsec:qEHVIDerivation:ComputeEHVI}

The above approach for computing \HVI{} assumes we know the true objective values $\{\bm f(\bm x_i)\}_{i=1}^q$. Since we do not know the true function values $\{\bm f(\bm x_i)\}_{i=1}^q$, we compute \qEHVI{} as the expectation over the GP posterior. 
\begin{equation}
\label{appdx:eq:qEHVIDerivation:ComputeEHVI:qehvi:exp}
\aqEHVI = \mathbb{E}\Bigl[\HVI{}(\{\bm f(\bm{x}_i)\}_{i=1}^q)\Bigr] =\int_{\mathbb{R}^M} \HVI{}(\{\bm f(\bm x_i)\}_{i=1}^q) d\bm{f}
\end{equation}
In the sequential setting and under the assumption of independent outcomes, \qEHVI{} is simply \EHVI{} and can be expressed in closed form \citep{yang2019}. However when $q>1$, there is no known analytical formulation \cite{yang_palar19}.
Instead, we estimate the expectation in \eqref{appdx:eq:qEHVIDerivation:ComputeEHVI:qehvi:exp} using MC integration with samples from the joint posterior $\mathbb P \big(\bm f(\bm x_1),..., \bm f(\bm x_q) | \mathcal D)$:
\begin{align}
\label{appdx:eq:qEHVIDerivation:ComputeEHVI:qehvi}
\aqEHVI =\mathbb{E}\Bigl[\HVI{}(\{\bm f(\bm{x}_i)\}_{i=1}^q)\Bigr] &\approx \frac{1}{N}\sum_{t=1}^{N} \HVI{}(\{\bm f_t(\bm x_i)\}_{i=1}^q)\\
&=\frac{1}{N}\sum_{t=1}^{N} \sum_{k=1}^{K}\sum_{j=1}^q\sum_{X_j \in \mathcal X_j} (-1)^{j+1} \prod_{m=1}^M\big[\bm z_{k, X_j, t}^{(m)} - l_k^{(m)}\big]_{+}
\end{align}
where $\{\bm{f}_t(\bm x_i)\}_{i=1}^q \sim \mathbb P \big(\bm f(\bm x_1),..., \bm f(\bm x_q) | X, Y\big)$ is the $t^{\text{th}}$ sample from the joint posterior over $\xcand$ and $\bm z_{k, X_j, t}^{(m)} = \min \big[\bm u_k,\min_{\bm x' \in X_j}\bm f_t(\bm x')\big]$.

%We emphasize that our approach lets us parallelize computation over $M$ objectives, the powerset of $\xcand$, all $K$ hyper-rectangles in the decomposition of the nondominated space, and all $N$ MC samples. 

% In addition, provided that the hyper-rectangle decomposition algorithm is exact, we emphasize this \qEHVI{} computation is \emph{exact} up to the MC estimation error, which scales as $1/\sqrt{N}$ regardless of the dimension of the search space \citep{emmerich2006}. In practice, we use QMC integration \citep{caflisch1998monte} to reduce the variance of the MC estimator and its gradient (see Figures \ref{fig:mc_gradient} and \ref{fig:mc_acqf}).

%%%%%%%%%%%%%%%%%%%%%%%%
\subsection{Supporting Outcome Constraints}
\label{appdx:subsec:DqEHVI:OutcomeConstraints}

Recall that we defined the constrained hypervolume improvement as
\begin{align}
\label{eqn:qehvi:HVIc}
    \HVIc{}(\bm f(\bm x), \bm c(\bm x)) =\HVI{}[\bm f(\bm x)] \cdot \mathbbm{1}[\bm c(\bm x) \geq \bm 0].
\end{align}
For $q=1$ and assuming independence of the objectives and the constraints, the expected $\HVIc{}$ is the product of the expected \HVI{} and the probability of feasibility (the expectation of $\mathbbm{1}[\bm c(\bm x) \geq \bm 0]$) \citep{Feliot_2016}. However, requiring objectives and constraints to be independent is unnecessary when estimating the expectation with MC integration using samples from the joint posterior.

In the parallel setting, if all constraints are satisfied for all $q$ candidates $\xcand{} = \{\bm x_i\}_{i=1}^q$, $\HVIc{}$ is simply \HVI{}. If a subset $\mathcal V \subset \xcand{}, \mathcal V \ne \varnothing$ of the candidates violate at least one of the constraints, then the feasible \HVI{} is the \HVI{} of the set of feasible candidates: $\HVIc(\xcand{}) = \HVI(\xcand \setminus \mathcal V)$. That is, the \emph{hypervolume contribution} (i.e. the marginal \HVI{}) of an infeasible point is zero. In our formulation, \HVI{} can be computed by multiplying~\eqref{eqn:DqEHVI:ComputeEHVI:qehvi} with an additional factor $\prod_{\bm x' \in X_j} \prod_{v=1}^V\mathbbm{1}[c^{(v)}(\bm x') \geq 0]$:
\begin{equation}
\label{appdx:eqn:DqEHVI:cqHVI}
    \HVIc{}(\{\bm f(\bm x_i), \bm c(\bm x_i)\}_{i=1}^q) =
    \sum_{k=1}^{K}\sum_{j=1}^q\sum_{X_j \in \mathcal X_j}(-1)^{j+1} \Bigg[\bigg(\prod_{m=1}^M\big[\bm z_{k, X_j}^{(m)} - l_k^{(m)}\big]_{+}\bigg)\prod_{\bm x' \in X_j} \prod_{v=1}^V\mathbbm{1}[c^{(v)}(\bm x') \geq 0]\Bigg].
\end{equation}
The additional factor $\prod_{\bm x'\in X_j} \prod_{v=1}^V\mathbbm{1}[c^{(v)}(\bm x_a) \geq 0]$ indicates whether all constraints are satisfied for all candidates in a given subset $X_j$. 
Thus $\HVIc{}$ can be computed in the same fashion as \HVI{}, but with the additional step of setting the \HV{} of all subsets containing $\bm x'$ to zero if $\bm x'$ violates any constraint. We can now again perform MC integration as in \eqref{eqn:DqEHVI:ComputeEHVI:qehvi} to compute the expected constrained hypervolume improvement.

In this formulation, the marginal hypervolume improvement from a candidate is weighted by the probability that the candidate is feasible. The marginal hypervolume improvements are highly dependent on the outcomes of the other candidates. Importantly, the MC-based approach enables us to properly estimate the marginal hypervolume improvements across candidates by sampling from the joint posterior.

Note that while the \emph{expected} constrained hypervolume $\mathbb{E} \bigl[ \HVIc{}(\{\bm f(\bm x_i), \bm c(\bm x_i)\}_{i=1}^q) \bigr]$ is differentiable, we may \emph{not} differentiate inside the expectation (hence we cannot expect simply differentiating \eqref{appdx:eqn:DqEHVI:cqHVI} on the sample-level to provide proper gradients). We therefore replace the indicator with a sigmoid function with temperature parameter $\epsilon$, which provides a differentiable relaxation
\begin{align}
    \mathbbm{1}[c^{(v)}(\bm x') \geq 0] \approx s(c^{(v)}(\bm x'); \epsilon) := \frac{1}{1 + \exp(-c^{(v)}(\bm x') /\epsilon)}
\end{align}
that becomes exact in the limit $\epsilon \searrow 0$.

As in the unconstrained parallel scenario, there is no known analytical expression for the expected feasible hypervolume improvement. Therefore, we again use MC integration to approximate the expectation:
\begin{subequations}
\label{appdx:eqn:DqEHVI:cqEHVI}
\begin{align}
\aqEHVIc(\bm x) &= \mathbb{E}\Bigl[\HVIc{}(\{\bm f(\bm{x}_i), \bm c(\bm x_i)\}_{i=1}^q)\Bigr]\\
% &= \int_{-\infty}^\infty\int_{-\infty}^\infty \HVIc{}(\{\bm f(\bm x_i), \bm c(\bm x_i)\}_{i=1}^q) d\bm{f} d\bm{c}\\
&\approx \frac{1}{N}\sum_{t=1}^{N} \HVIc{}(\{\bm f_t(\bm x_i), c_t(\bm x_i)\}_{i=1}^q)\\
&\approx\frac{1}{N}\sum_{t=1}^{N} \sum_{k=1}^{K}\sum_{j=1}^q\sum_{X_j \in \mathcal X_j}(-1)^{j+1} \Bigg[ \bigg(\prod_{m=1}^M\big[\bm z_{k, X_j,t}^{(m)} - l_k^{(m)}\big]_{+}\bigg)\prod_{\bm x' \in X_j} \prod_{v=1}^V
% \mathbbm{1}[c_t^{(v)}(\bm x') \geq 0]\Bigg]
s(c^{(v)}(\bm x'); \epsilon)
\Bigg]
\end{align}
\end{subequations}

\subsubsection{Inclusion Exclusion principle for $\HVIc{}$}
Equation~\eqref{appdx:eqn:DqEHVI:cqHVI} holds when the indicator function because $\HVIc{}$ is equivalent to \HVI{} with the subset of feasible points. However, the sigmoid approximation can result in non-zero error. The error function  $\varepsilon: 2^{\xcand} \rightarrow \mathbb R$ can be expressed as
 
 $$\varepsilon(X) = \prod_{\bm x' \in X}\prod_{v=1}^V \mathbbm{1}[c(\bm x') > 0]  -  \prod_{\bm x' \in X}\prod_{v=1}^V s(c(\bm x'), \epsilon)$$ 
The error function gives a value to each to each element of $2^{\xcand}$.
Weight functions have been studied in conjunction with the inclusion-exclusion principle \citep{weighted_inclusion_exclusion}, but under the assumption of that the weight of a set is the sum of the weights of its elements: $w(A) = \sum_{a \in A} w(a)$. In our case, the weight function of a set $A$ is the product the weights of its elements. There, it is not obvious whether the inclusion-exclusion principle will hold in this case.
\begin{theorem} 
\label{thm:constrained_hvi}
Given a feasible Pareto front $\mathcal P_\text{feas}$, a partitioning $\{(\bm l_k, \bm u_k\}_{k=1}^K$ of the objective space $\mathbb R^M$ that is not dominated by the $\mathcal P_\text{feas}$, then for a set of points $\xcand$ with objective values $\bm f(\xcand)$ and constraint values $\bm c(\xcand)$, 
$$\HVIc(\bm f(\xcand), \bm c(\xcand), \mathcal P, \bm r) = \HVI(\bm f'(\xcand), \mathcal P', \bm r')$$
where $\bm f'(\xcand)$ is the set of objective-constraint vectors for each candidate point $\bm f'(\bm x) \in \mathbb R^{M+V}$,  $\mathcal P'$ is the set of vectors $\bm [ f^{(1)}(\bm x),..., f^{(M)}(\bm x), \bm 0_V] \in \mathbb R^{M+V}$, and $\bm r' = [r^{(1)},..., r^{(M)}, \bm 0_V] \in \mathbb R^{M+V}$.
\end{theorem}

\begin{proof}
Recall equation \ref{appdx:eqn:DqEHVI:cqHVI},
\begin{align*}
\HVIc{}(\{\bm f(\bm{x}_i), \bm c(\bm x_i)\}_{i=1}^q)
&=\sum_{k=1}^{K}\sum_{j=1}^q\sum_{X_j \in \mathcal X_j}(-1)^{j+1} \Bigg[ \bigg(\prod_{m=1}^M\big[ z_{k, X_j}^{(m)} - l_k^{(m)}\big]_{+}\bigg)\prod_{\bm x' \in X_j} \prod_{v=1}^V
\mathbbm{1}[c^{(v)}(\bm x') \geq 0]\Bigg].
\end{align*}
Note that the constraint product
\begin{equation}
\label{eqn:constrain_prod}
\begin{split}
\prod_{\bm x' \in X_j} \prod_{v=1}^V
\mathbbm{1}[c^{(v)}(\bm x') \geq 0] &=  \prod_{v=1}^V\prod_{\bm x' \in X_j}\mathbbm{1}[c^{(v)}(\bm x') \geq 0]\\ 
&=  \prod_{v=1}^V\min_{\bm x' \in X_j}\mathbbm{1}[c^{(v)}(\bm x') \geq 0]\\ 
&= \prod_{v=1}^V\min\bigg[1,\min_{\bm x' \in X_j} 
\mathbbm{1}[c^{(v)}(\bm x') \geq 0]\bigg]\\
&= \prod_{v=1}^V\Bigg[\min\bigg[1,\min_{\bm x' \in X_j}
\mathbbm{1}[c^{(v)}(\bm x') \geq 0]\bigg] - 0\Bigg].
\end{split}
\end{equation}
For $v = 1, \ldots, V$, $k=1,...K$, let $l_k^{(M+v)} = 0$ and $u_k^{(M+v)} = 1$. Then, substituting into the following expression from Equation \ref{eqn:constrain_prod} gives
\begin{align*}
\min\bigg[1,\min_{\bm x' \in X_j}
\mathbbm{1}[c^{(v)}(\bm x') \geq 0]\bigg] 
&= \min\bigg[u_k^{(M+v)},\min_{\bm x' \in X_j}
\mathbbm{1}[c^{(v)}(\bm x') \geq 0]\bigg]
\end{align*}
Recall from Section 4, that $z$ is defined as: $\bm z_k := \min \big[\bm u_k,\bm f(\bm x)\big]$. The high-level idea is that if we consider the indicator of the slack constraints $\mathbbm{1}[c^{(v)}(\bm x') \geq 0]$ as objectives, then the above expression is consistent with the definition of $z$ at the beginning of section 4. For $v=1, \ldots, V$,
\begin{align*}
z_{k, X_j}^{(M+v)} = \min\bigg[1,\min_{\bm x' \in X_j}
\mathbbm{1}[c^{(v)}(\bm x') \geq 0]\bigg] 
\end{align*}
Thus,
\begin{align*}
\prod_{\bm x' \in X_j} \prod_{v=1}^V
\mathbbm{1}[c^{(v)}(\bm x') \geq 0] 
&= \prod_{v=1}^V\Bigg[\min\bigg[1,\min_{\bm x' \in X_j}
\mathbbm{1}[c^{(v)}(\bm x') \geq 0]\bigg] - 0\Bigg]\\
&=\prod_{v=1}^V \big[  z_{k, X_j}^{(M+v)} - l_k^{(M+v)}\big]_{\text{+}}
\end{align*}

Returning to the $\HVIc{}$ equation, we have
\begin{equation}
\label{eqn:equivalent_hvi}
\begin{split}
\HVIc{}(\{\bm f(\bm{x}_i), \bm c(\bm x_i)\}_{i=1}^q)
&=\sum_{k=1}^{K}\sum_{j=1}^q\sum_{X_j \in \mathcal X_j}(-1)^{j+1} \Bigg[ \bigg(\prod_{m=1}^M\big[ z_{k, X_j}^{(m)} - l_k^{(m)}\big]_{+}\bigg)\prod_{\bm x' \in X_j} \prod_{v=1}^V
\mathbbm{1}[c^{(v)}(\bm x') \geq 0]\Bigg]\\
&=\sum_{k=1}^{K}\sum_{j=1}^q\sum_{X_j \in \mathcal X_j}(-1)^{j+1} \Bigg[ \bigg(\prod_{m=1}^M\big[z_{k, X_j}^{(m)} - l_k^{(m)}\big]_{+}\bigg)\prod_{v=M+1}^{M+V} \big[  z_{k, X_j}^{(v)} - l_k^{(M+v)}\big]_{\text{+}}\Bigg]\\
&=\sum_{k=1}^{K}\sum_{j=1}^q\sum_{X_j \in \mathcal X_j}(-1)^{j+1} \Bigg[\prod_{m=1}^{M+V}\big[z_{k, X_j}^{(m)} - l_k^{(m)}\big]_{+}\Bigg]
\end{split}
\end{equation}
\end{proof}

Now consider the case when a sigmoid approximation 
$\mathbbm{1}[c^{(v)}(\bm x') \geq 0] \approx s(c^{(v)}(\bm x'); \epsilon)$ is used.
The only change to Equation \ref{eqn:equivalent_hvi} is that
$$z_{k, X_j}^{(m)} \approx \hat{z}_{k, X_j}^{(m)} = \min\bigg[u_k^{(M+v)},\min_{\bm x' \in X_j}
S[c^{(v)}(\bm x'), \epsilon]\bigg].$$
If $S[c^{(v)}(\bm x'), \epsilon] = \mathbbm{1}[c^{(v)}(\bm x') \geq 0]$ for all $v, \bm x'$, then HVI is computed exactly without approximation error. If $S[c^{(v)}(\bm x'), \epsilon] \mathbbm{1}[c^{(v)}(\bm x') \geq 0]$ for any $v, \bm x'$, then there is approximation error: the hypervolume improvement from all subsets containing $\bm x'$ is proportional to $\prod_{v=1}^V \min_{\bm x' \in X}s(c(\bm x'), \epsilon)$. Since the constraint outcomes are directly considered as components in the hypervolume computation, the inclusion-exclusion principle incorporates the approximate indicator properly.

%%%%%%%%%%%%%%%%%%%%%%%%
\subsection{Complexity}
\label{appdx:subsec:DqEHVI:complexity}

Recall from Section~\ref{subsec:DqEHVI:complexity} that, given posterior samples, the time complexity on a single-threaded machine is $T_1 = O(MNK(2^q-1))$.  
The space complexity required for maximum parallelism is also is $T_1$ (ignoring the space required by the models), which does limit scalability to larger $M$ and $q$, but difficulty scaling to large $M$ is a known limitaiton of \EHVI{} \citep{yang2019}. To reduce memory load, rectangles could be materialized and processed in chunks at the cost of additional runtime. In addition, our implementation of \qEHVI{} uses the box decomposition algorithm from \citet{couckuyt12}, but we emphasize \qEHVI{} is agnostic to the choice of partitioning algorithm and using a more efficient partitioning algorithm (e.g. \citep{yang2019, DACHERT2017, LACOUR2017347}) may significantly improve memory footprint on GPU and enable larger using $q$ in many scenarios.

% Using the Threaded Many-core Memory framework for analyzing time complexity \citep{Ma2014}, the time complexity of computing \qEHVI{} on a machine with $P$ cores and  $\mathcal T$ threads per core (used by the algorithm) is $O\bigl(\max \bigl( \frac{T_1}{P},T_\infty, \frac{AB}{\mathcal T P} \bigr)\bigr)$, where $A$ is the number of global memory transactions and $B$ is the time for a slow global memory lookup \sd{If we were to naively load every element from globally memory individually, then $A=T_1$, but I believe we slice tensors and move them from global memory in chunks, so this should be quite a bit less. @Max, any thoughts here?}. Modern GPUs have thousands of cores and can run thousands of threads per core \citep{nvidia}. The space complexity required for maximum parallelism is also is $T_1$ (ignoring the space required by the models), but rectangles could be materialized and processed in chunks to reduce memory load at the cost of additional runtime.

%%%%%%%%%%
\section{Error Bound on Sequential Greedy Approximation}
\label{sec:seq_greedy_details}
If the acquisition function $\mathcal L(\xcand)$ is a normalized, monotone, submodular set function (where submodular means that the increase in $\mathcal L(\xcand)$ is non-increasing as elements are added to $\xcand$ and normalized means that $\mathcal L(\emptyset) = 0$), then the sequential greedy approximation of $\mathcal L$ enjoys regret of no more than $\frac{1}{e}\mathcal L^*$, where $\mathcal L^*$ is the optima of $\mathcal L$ \citep{Fisher1978}. We have $\alpha_{\text{\qEHVI{}}}(\xcand) = \mathcal L(\xcand) = \mathbb E_{\bm f} \big(\HVI\big[\bm f(\xcand)\big]\big)$. Since \HVI{} is a submodular set function  \citep{friedrich14} and the expectation of a stochastic submodular function is also submodular \citep{asadpour}, $\alpha_{\text{\qEHVI{}}}(\xcand)$ is also submodular and therefore its sequential greedy approximation enjoys regret of no more than $\frac{1}{e}\mathcal \alpha_{\text{\qEHVI{}}}^*$. Using the result from \citet{wilson2018maxbo}, the MC-based approximation $\hat{\alpha}_{\text{\qEHVI{}}}(\xcand) = \sum_{t=1}^{N}\HVI\big[\bm f_t(\xcand)\big]$
also enjoys the same regret bound since \HVI{} is a normalized submodular set function.\footnote{As noted in \citet{wilson2018maxbo}, submodularity technically requires the search space $\mathcal X$ to be finite, whereas in BO, it will typically be infinite. \citet{wilson2018maxbo} note that in similar scenarios, submodularity has been extended to infinite sets $\mathcal X$ (e.g. \citet{srinivas}).}

%%%%%%%%%%%%%%%%%%%%%%%%%%%%%%%%%%%%
\section{Convergence Results}
\label{appdx:sec:Convergence}

For the purpose of stating our convergence results, we recall some concepts and notation from~\citet{balandat2020botorch}. First, consider a sample $\{\bm f_t(\bm x_1)\}_{i=1}^q$ from the multi-output posterior of the GP surrogate model. Let $\bm x \in \mathbb{R}^{qd}$ be the stacked set of candidates $\xcand$ and let $\bm f_t(\bm x) := [f_t(\bm x_1)^T, \dotsc, f_t(\bm x_q)^T]^T$ be the stacked set of corresponding objective vectors. It is well known that, using the reparameterization trick, we can write 
\begin{align}
    \bm f_t(\bm x) = \mu(\bm x) + L(\bm x) \epsilon_t,
\end{align}
where $\mu: \mathbb{R}^{qd} \rightarrow \mathbb{R}^{qM}$ is the mean function of the multi-output GP, $L(\bm x) \in \mathbb{R}^{qM \times qM}$ is a root decomposition (typically the Cholesky decomposition) of the multi-output GP's posterior covariance $\Sigma(\bm x) \in \mathbb{R}^{qM \times qM}$, and $\epsilon_t \in \mathbb{R}^{qM}$ with $\epsilon_t \sim \mathcal{N}(0, I_{qM})$.

% $\nabla_{\!\bx} \alpha(\bx; \Phi,\calD)$ can often be obtained from~\eqref{eq:Acquisition:MyopicMC} via the reparameterization trick~\citep{kingma2013reparam,rezende2014stochbackprop}. The basic idea is that $\xi \sim f_\calD(\bx)$ can be expressed as a suitable (differentiable) deterministic transformation $\xi = h_{\calD}(\bx, \epsilon)$ of an auxiliary random variable $\epsilon$ independent of~$\bx$.
% %
% For instance, if $f_\calD(\bx) \sim \mathcal{N}(\mu_\bx, \Sigma_\bx)$, then $h_{\calD}(\bx, \epsilon) = \mu_\bx + L_{\bx} \epsilon$, with $\epsilon \sim \mathcal{N}(0, I)$ and $L_{\bx}L_{\bx}^T = \Sigma_\bx$.

For $\bm x \in \mathcal X$, consider the MC-approximation $\hataqEHVI^N(\bm x)$ from \eqref{eqn:DqEHVI:ComputeEHVI:qehvi}. Denote by $\nabla_{\bm x}\hataqEHVI^N(\bm x)$ the gradient of $\hataqEHVI^N(\bm x)$, obtained by averaging the gradients on the sample-level:
\begin{align}
\label{appdx:eq:Convergence:SampleGradient}
    \nabla_{\bm x}\hataqEHVI^N(\bm x) := \frac{1}{N} \sum_{t=1}^{N} \nabla_{\bm x}\HVI(\{f_t(\bm x_i)\}_{i=1}^q)
\end{align} 

Let $\aqEHVI^* := \max_{\bm x \in \mathcal X} \aqEHVI(\bm x)$ denote the maximum of the true acquisition function \qEHVI{}, and let $\mathcal{X}^* := \argmax_{\bm x \in \mathcal X} \aqEHVI(\bm x)$ denote the set of associated maximizers. 
    
\begin{theorem}
\label{thm:Convergence:SAA}
    Suppose that $\mathcal X$ is compact and that $f$ has a Multi-Output Gaussian Process prior with continuously differentiable mean and covariance functions. If the base samples $\{\epsilon_t\}_{t=1}^N$ are drawn i.i.d. from $\mathcal N(0,I_{qM})$, and if $\hat{\bm x}^*_N \in \argmax_{\bm x \in \mathcal X} \hataqEHVI^N(\bm x)$, then 
    \begin{enumerate}[label={(\arabic*)}]
        \item $\aqEHVI(\hat{\bm x}^*_N) \rightarrow \aqEHVI^*$ a.s.
        \item $\textnormal{dist}(\hat{\bm x}_{\!N}^*, \mathcal{X}^*) \rightarrow 0$ a.s.
    \end{enumerate}
\end{theorem}

In addition to the almost sure convergence in Theorem~\ref{thm:Convergence:SAA}, deriving a result on the convergence rate of the optimizer, similar to the one obtained in~\citep{balandat2020botorch}, should be possible. We leave this to future work. 
Moreover, the results in Theorem~\ref{thm:Convergence:SAA} can also be extended to the situation in which the base samples are generated using a particular class of randomized QMC methods (see similar results in \citep{balandat2020botorch}).

\begin{proof}
We consider the setting from~\citet[Section D.5]{balandat2020botorch}. Let $\epsilon ~\sim \mathcal{N}(0, I_{qM})$,
%denote the random variable used in the reparameterization trick, 
so that we can write the posterior over outcome $m$ at $\bm x$ as the random variable $f^{(m)}(\bm x, \epsilon) = S_{\{i_j, m\}}(\mu(\bm x) + L(\bm x)\epsilon)$, 
% Since the $f(\bm x, \epsilon)$ are samples from a multi-output model (a MVN of dimension $qM$), we can write $f(\bm x, \epsilon)_{i_j}^{(m)} = S_{\{i_j, m\}}(\bm \mu(\bm x) + L(\bm x)\epsilon)$,
where $\mu(\bm x)$ and $L(\bm x)$ are the (vector-valued) posterior mean and the Cholesky factor of posterior covariance, respectively, and $S_{\{i_j, m\}}$ is an appropriate selection matrix (in particular, $\|S_{\{i_j, m\}}\|_{\infty} \leq 1$ for all $i_j$ and $m$). 
Let
\begin{align*}
    A(\bm x, \epsilon) = \sum_{k=1}^{K}\sum_{j=1}^q \sum_{X_j \in \mathcal X_j} (-1)^{j+1} \prod_{m=1}^M\big[z_{k, X_j}^{(m)}(\epsilon) - l_k^{(m)}\big]_{+}
\end{align*}
where
\begin{align*}
    z_{k, X_j}^{(m)}(\epsilon) = \min \big[u_k^{(m)}, f^{(m)}(\bm x_{i_1}, \epsilon), \ldots, f^{(m)}(\bm x_{i_j}, \epsilon)\big] 
\end{align*}
and $X_j = \{\bm x_{i_1},\ldots, \bm x_{i_j}\}$.
Following \citep[Theorem 3]{balandat2020botorch}, we need to show that there exists an integrable function $\ell:\mathbb{R}^{q \times M}\mapsto \mathbb{R}$ such that for almost every $\epsilon$ and all $\bm x, \bm y \subseteq \mathcal X, \bm x, \bm y \in \mathbb{R}^{q \times d}$, 
\begin{align}
    |A(\bm x, \epsilon) - A(\bm y, \epsilon)| \leq \ell(\epsilon) \|\bm x - \bm y\|.
\end{align}
Let us define $$\tilde{a}_{kmjX_j}(\bm x, \epsilon) := \Bigl[\min \big[u_k^{(m)}, f^{(m)}(\bm x_{i_1}, \epsilon), \ldots, f^{(m)}(\bm x_{i_j}, \epsilon)\big] - l_k^{(m)}\Bigr]_{+}.$$
Linearity implies that it suffices to show that this condition holds for
\begin{align}
    \tilde{A}(\bm x, \epsilon) &:=  \prod_{m=1}^M \tilde{a}_{kmjX_j}(\bm x, \epsilon)
    =\prod_{m=1}^M\Bigl[\min \big[u_k^{(m)}, f^{(m)}(\bm x_{i_1}, \epsilon), \ldots, f^{(m)}(\bm x_{i_j}, \epsilon)\big] - l_k^{(m)}\Bigr]_{+}
\end{align}
for all~$k$, $j$, and $X_j$. 
Observe that 
\begin{align*}
    \tilde{a}_{kmjX_j}(\bm x, \epsilon)
    &\leq \Bigl| \min \big[u_k^{(m)}, f^{(m)}(\bm x_{i_1}, \epsilon), \ldots, f^{(m)}(\bm x_{i_j}, \epsilon)\big] - l_k^{(m)}\Bigr| \\
    &\leq | l_k^{(m)} | + \Bigl| \min \big[ u_k^{(m)}, f^{(m)}(\bm x_{i_1}, \epsilon), \ldots, f^{(m)}(\bm x_{i_j}, \epsilon)\big]  \Bigr|.
\end{align*}
Note that if $u_k^{(m)} =\infty$, then $\min[u_k^{(m)}, f(\bm x, \epsilon)_{i_1}^{(m)}, ... f^{(m)}(\bm x_{i_j}, \epsilon)] = \min[f^{(m)}(\bm x_{i_1}, \epsilon), ... f^{(m)}(\bm x_{i_j}, \epsilon)]$. If $u_k^{(m)} < \infty$, then $\min[u_k^{(m)}, f^{(m)}(\bm x_{i_1}, \epsilon), ... f^{(m)}(\bm x_{i_j}, \epsilon)] <\bigl| \min[f^{(m)}(\bm x_{i_1}, \epsilon), ... f^{(m)}(\bm x_{i_j}, \epsilon)]\bigr| +\bigl| u_k^{(m)}\bigr|$. Let $w_k^{(m)} = u_k^{(m)}$ if $u_k^{(m)} < \infty$ and 0 otherwise. Then
\begin{align*}
    \tilde{a}_{kmjX_j}(\bm x, \epsilon)
    &\leq | l_k^{(m)} | + | w_k^{(m)} | + \bigl| \min \big[ f^{(m)}(\bm x_{i_1}, \epsilon), \ldots, f^{(m)}(\bm x_{i_j}, \epsilon)\big]  \bigr| \\
    &\leq | l_k^{(m)} |+ | w_k^{(m)} | + \sum_{i_1, \dotsc, i_j} \bigl| f^{(m)}(\bm x_{i_j}, \epsilon) \bigr|.
\end{align*}
We therefore have that 
\begin{align*}
    |\tilde{a}_{kmjX_j}(\bm x, \epsilon)| \leq | l_k^{(m)} | + |w_k^{(m)} | + |X_j| \bigl( \| \mu^{(m)}(\bm x)\| + \|L^{(m)}(\bm x)\| \|\epsilon\| \bigr)
\end{align*}
for all $k, m, j, X_j$, where $|X_j|$ denotes the cardinality of the set $X_j$. 
Under our assumptions (compactness of $\mathcal X$, continuous differentiability of mean and covariance function), both $\mu(\bm x)$ and $L(\bm x)$, as well as their respective gradients w.r.t. $\bm x$, are uniformly bounded. In particular there exist $C_1, C_2 < \infty$ such that
\begin{align*}
    |\tilde{a}_{kmjX_j}(\bm x, \epsilon)| &\leq C_1 + C_2 \|\epsilon\| 
    %\\
    % \|\partial_{\bm x} \tilde{a}_{km}(\bm x, \epsilon)\| &\leq C_3 \|\epsilon\|
\end{align*}
for all $k, m, j, X_j$.

%
% Consider first the case $M =2$. Noting that for any $g, h, \bm x, \bm y$ we have
% \begin{subequations}
% \label{eq:telescope}
% \begin{align}
% |g(\bm x)h(\bm x) - g(\bm y)h(\bm y)| &= |g(\bm x)(h(\bm x) - h(\bm y)) + h(\bm y)(g(\bm x) - g(\bm y))| \\
% &\leq |g(\bm x)| |h(\bm x) - h(\bm y)| + |h(\bm y)| |g(\bm x) - g(\bm y)|
% \end{align}
% \end{subequations}
%
Dropping indices $k, j, X_j$ for simplicity, observe that
\begin{subequations}
\label{appdx:eq:Convergence:SplitProduct}
\begin{align}
    \bigl|\tilde{A}(\bm x, \epsilon) - \tilde{A}(\bm y, \epsilon)\bigr| 
    &= \bigl|\tilde{a}_{1}(\bm x, \epsilon) \tilde{a}_{2}(\bm x, \epsilon) - \tilde{a}_{1}(\bm y, \epsilon) \tilde{a}_{2}(\bm y, \epsilon)\bigr| \\
    &= \bigl|\tilde{a}_{1}(\bm x, \epsilon) \bigl(\tilde{a}_{2}(\bm x, \epsilon) - \tilde{a}_{2}(\bm y, \epsilon)\bigr) + \tilde{a}_{2}(\bm y, \epsilon) \bigl(\tilde{a}_{1}(\bm x, \epsilon) - \tilde{a}_{1}(\bm y, \epsilon)\bigr) \bigr| \\
    &\leq |\tilde{a}_{1}(\bm x, \epsilon)| \bigl|\tilde{a}_{2}(\bm x, \epsilon) - \tilde{a}_{2}(\bm y, \epsilon)\bigr| + |\tilde{a}_{2}(\bm y, \epsilon)| \bigl|\tilde{a}_{1}(\bm x, \epsilon) - \tilde{a}_{1}(\bm y, \epsilon)\bigr|.
\end{align}
\end{subequations}
Furthermore, 
\begin{align*}
    |\tilde{a}_{kmjX_j}(\bm x, \epsilon) - \tilde{a}_{kmjX_j}(\bm y, \epsilon)| &\leq \sum_{i_1, \dotsc, i_j} \bigl| S_{\{i_j, m\}}(\mu(\bm x) + L(\bm x)\epsilon) - S_{\{i_j, m\}}(\mu(\bm y) + L(\bm y)\epsilon) \bigr| \\
    &\leq |X_j| \Bigl( \|\mu(\bm x) - \mu(\bm y) \| + \|L(\bm x) - L(\bm y)\| \|\epsilon\| \Bigr).
\end{align*}
Since $\mu$ and $L$ have uniformly bounded gradients, they are Lipschitz. Therefore, there exist $C_3, C_4 < \infty$ such that 
\begin{align*}
    |\tilde{a}_{kmjX_j}(\bm x, \epsilon) - \tilde{a}_{kmjX_j}(\bm y, \epsilon)| &\leq (C_3 + C_4 \|\epsilon\|) \|\bm x - \bm y \|
\end{align*}
for all $\bm x, \bm y, k, m, j, X_j$. Plugging this into~\eqref{appdx:eq:Convergence:SplitProduct} above, we find that
\begin{align*}
    \bigl|\tilde{A}(\bm x, \epsilon) - \tilde{A}(\bm y, \epsilon)\bigr| 
    &\leq 2 \Bigl(C_1C_3 + (C_1C_4 + C_2C_3) \|\epsilon\| + C_2C_4 \|\epsilon\|^2 \Bigr) \|\bm x - \bm y \|
\end{align*}
for all $\bm x, \bm y$ and $\epsilon$. For $M>2$ we generalize the idea from~\eqref{appdx:eq:Convergence:SplitProduct}, making sure to telescope the respective expressions. It is not hard to see that with this, there exist $C < \infty$ such that 
\begin{align*}
    \bigl|\tilde{A}(\bm x, \epsilon) - \tilde{A}(\bm y, \epsilon)\bigr| 
    &\leq C \sum_{m=1}^M \|\epsilon\|^m \|\bm x - \bm y \|
\end{align*}
Letting $\ell(\epsilon) := C \sum_{m=1}^M \|\epsilon\|^m$, we observe that $\ell(\epsilon)$ is integrable (since all absolute moments exist for the Normal distribution).

The result now follows from in~\citet[Theorem 3]{balandat2020botorch}.
\end{proof}

Besides the above convergence result, we can also show that the sample average gradient of the MC approximation of \qEHVI{} is an unbiased estimator of the true gradient of \qEHVI{}:

\begin{proposition}
\label{prop:eq:Convergence:UnbiasedGradientEstimate}
    Suppose that the GP mean and covariance function are continuously differentiable. Suppose further that the candidate set $\bm x$ has no duplicates, and that the sample-level gradients $\nabla_{\bm x}\HVI(\{f_t(\bm x_i)\}_{i=1}^q)$ are obtained using the reparameterization trick as in \citep{balandat2020botorch}. Then
    \begin{align}
    \mathbb{E}\bigl[\nabla_{\bm x}\hataqEHVI^N(\bm x)\bigr] = \nabla_{\bm x}\aqEHVI(\bm x),
    \end{align}
    that is, the averaged sample-level gradient is an unbiased estimate of the gradient of the true acquisition function.
    % (expectation in w.r.t. the base samples.
\end{proposition}

\begin{proof}
    This proof follows the arguments \citet[Theorem 1]{wang2016parallel}, which leverages \citet[Theorem 1]{glasserman88}. We verify the conditions of \citet[Theorem 1]{glasserman88} below. Using the arguments from \citep{balandat2020botorch}, we know that, under the assumption of differentiable mean and covariance functions, the samples $\bm{f}_t(\bm x)$ %:= \{\bm{f}_t(\bm x_i)\}_{i=1}^q$
    are continuously differentiable w.r.t. $\bm x$ (since there are no duplicates, and thus the covariance $\Sigma(\bm x)$ is non-singular). Hence, \citet[A1]{glasserman88} is satisfied.
    Furthermore, it is easy to see from~\eqref{eq:DqEHVI:HVI} that $\HVI{}(\{\bm f(\bm x_i)\}_{i=1}^q)$ is \textit{a.s.} continuous and is differentiable w.r.t. $\bm{f}_t(\bm x)$ on $\mathbb{R}^M$, except on the edges of the hyper-rectangle decomposition $\{S_k\}_{k=1}^K$ of the non-dominated space, which satisfies \citep[A3]{glasserman88}. The set of points defined by the union of these edges clearly has measure zero under any non-degenerate (non-singular covariance) GP posterior on $\mathbb{R}^M$, so \citet[A4]{glasserman88} holds. Therefore \citet[Lemma 2]{glasserman88} holds, so $\HVI{}(\{\bm f(\bm x_i)\}_{i=1}^q)$ is \textit{a.s.} piece-wise differentiable w.r.t. $\bm x$.
    
    Lastly, we need to show that the result in \citet[Lemma 3]{glasserman88} holds: 
$$\mathbb E\bigg[\sup_{x_{ci} \notin \tilde{D}}\vert A'(\bm x, \epsilon) \vert \bigg] < \infty.$$
As in \citet[Theorem 1]{wang2016parallel}, we fix $\bm x$ except for $x_{ci}$ where $x_{ci}$ is the $c^\text{th}$ component of the $i^\text{th}$ point,
We need to show that $\mathbb E\bigl[\sup_{x_{ci} \notin \tilde{D}}\vert A'(\bm x, \epsilon) \vert \bigr] < \infty$. By linearity, it suffices to show that $\mathbb E\bigl[\sup_{x_{ci} \notin \tilde{D}}\vert \tilde{A}'(\bm x, \epsilon) \vert \bigr] < \infty$.
We have
\begin{align*}
\mathbb E\bigg[\sup_{x_{ci} \notin \tilde{D}}\vert \tilde{A}'(\bm x, \epsilon) \vert \bigg]
&=\mathbb E\bigg[\sup_{x_{ci} \notin \tilde{D}}\bigg\vert \frac{\partial \tilde{A}(\bm x, \epsilon)}{\partial x_{ci}}\bigg\vert\bigg].
\end{align*}

Consider the $M=2$ case. We have
$\tilde{A}(\bm x, \epsilon) = a_1(\bm x, \epsilon)a_2(\bm x, \epsilon)$, where

$$a_m(\bm x, \epsilon) = \Bigl[\min \big[u_k^{(m)}, f^{(m)}(\bm x_{i_1}, \epsilon), \ldots, f^{(m)}(\bm x_{i_j}, \epsilon)\big] - l_k^{(m)}\Bigr]_{+}.$$ 

The partial derivative of $\tilde{A}(\bm x, \epsilon)$ with respect to $x_{ci}$ is
% Note that for $x_{ci} \notin \tilde{D}$  $\frac{\partial a_m(\bm x, \epsilon)}{\partial x_{ci}} = 0$ if
% \begin{align*}
%     \frac{\partial \tilde{A}(\bm x, \epsilon)}{\partial x_{ci}}  &=\frac{\partial }{\partial x_{ci}}\Bigl[\min \big[ S_{\{i_1, 1\}}\big(\bm\mu(\bm x) + L(\bm x)\epsilon\big), \ldots, S_{\{i_j, 1\}}\big(\bm\mu(\bm x) + L(\bm x)\epsilon\big)\big] - l_k^{(1)}\Bigr]_{+}\\
%     &\cdot \Bigl[\min \big[ S_{\{i_1, 2\}}\big(\bm\mu(\bm x) + L(\bm x)\epsilon\big), \ldots, S_{\{i_j, 2\}}\big(\bm\mu(\bm x) + L(\bm x)\epsilon\big)\big] - l_k^{(2)}\Bigr]_{+}\\
%     &+ \Bigl[\min \big[ S_{\{i_1, 1\}}\big(\bm\mu(\bm x) + L(\bm x)\epsilon\big), \ldots, S_{\{i_j, 1\}}\big(\bm\mu(\bm x) + L(\bm x)\epsilon\big)\big] - l_k^{(1)}\Bigr]_{+}\\
%     &\cdot \frac{\partial }{\partial x_{ci}}\Bigl[\min \big[ S_{\{i_1, 2\}}\big(\bm\mu(\bm x) + L(\bm x)\epsilon\big), \ldots, S_{\{i_j, 2\}}\big(\bm\mu(\bm x) + L(\bm x)\epsilon\big)\big] - l_k^{(2)}\Bigr]_{+}
% \end{align*}
\begin{align*}
    \frac{\partial \tilde{A}(\bm x, \epsilon)}{\partial x_{ci}}  &=\frac{\partial a_1(\bm x, \epsilon)}{\partial x_{ci}}a_2(\bm x, \epsilon) + a_1(\bm x, \epsilon)\frac{\partial a_2(\bm x, \epsilon)}{\partial x_{ci}},
\end{align*}
and therefore
\begin{align*}
    \Bigl| \frac{\partial \tilde{A}(\bm x, \epsilon)}{\partial x_{ci}}\Bigr|  
    &\leq \Bigl|\frac{\partial a_1(\bm x, \epsilon)}{\partial x_{ci}}\Bigr| \cdot \Bigl|a_2(\bm x, \epsilon)\Bigr|+  \Bigl|a_1(\bm x, \epsilon)\Bigr|\cdot\Bigl|\frac{\partial a_2(\bm x, \epsilon)}{\partial x_{ci}}\Bigr|
\end{align*}

Since we are only concerned with $x_{ci} \notin \tilde{D}$, 

$$a_m(\bm x, \epsilon) = \Bigl[\min \big[ f^{(m)}(\bm x_{i_1}, \epsilon), \ldots, f^{(m)}(\bm x_{i_j}, \epsilon)\big] - l_k^{(1)}\Bigr]_{+}.$$

As in the proof of Theorem~\ref{thm:Convergence:SAA}, we write 
% Let $\epsilon ~\sim \mathcal{N}(0, I_{qM})$. We can write the posterior over outcome $m$ at $\bm x$ as the random variable $f^{(m)}(\bm x, \epsilon) = S_{\{i_j, m\}}(\bm \mu(\bm x) + L(\bm x) \epsilon)$, 
% % Since the $f(\bm x, \epsilon)$ are samples from a multi-output model (a MVN of dimension $qM$), we can write $f(\bm x, \epsilon)_{i_j}^{(m)} = S_{\{i_j, m\}}(\bm \mu(\bm x) + L(\bm x)\epsilon)$,
% where $\mu(\bm x)$ and $L(\bm x)$ are the posterior mean and Cholesky factor of posterior covariance, respectively, and $S_{\{i_j, m\}}$ is an appropriate selection matrix (in particular, $\|S_{\{i_j, m\}}\|_{\infty} \leq 1$ for all $i_j$ and $m$). 
the posterior over outcome $m$ at $\bm x$ as the random variable $f^{(m)}(\bm x, \epsilon) = S_{\{i_j, m\}}(\bm \mu(\bm x) + L(\bm x) \epsilon)$, where $\epsilon ~\sim \mathcal{N}(0, I_{qM})$ and $S_{\{i_j, m\}}$ is an appropriate selection matrix. 
With this,
\begin{align*}
a_m(\bm x, \epsilon) &=
    \Bigl[\min \big[ S_{\{i_1, 1\}}\big(\mu(\bm x) + L(\bm x)\epsilon\big), \ldots, S_{\{i_j, 1\}}\big(\mu(\bm x) + L(\bm x)\epsilon\big)\big] - l_k^{(1)}\Bigr]_{+}.
\end{align*}

Since the interval $\mathcal X$ is compact and the mean, covariance, and Cholesky factor of the covariance $\mu(\bm x), C(\bm x), L(\bm x)$ are continuously differentiable, for all $m$ we have
$$\sup_{x_{ci}} \bigg\vert \frac{\partial  \mu^{(m)}(\bm x_a)}{\partial x_{ci}}\bigg\vert = \mu_a^{*, (m)} < \infty, \qquad \sup_{x_{ci}} \bigg\vert \frac{\partial L^{(m)}(\bm x)}{\partial x_{ci}}\bigg\vert = L_{ca}^{*, (m)}< \infty.$$

Let $\mu^{(m)}_{**} = \max_{a} \mu_a^{*, (m)}$, $L^{(m)}_{**} = \max_{a, b} L^{*, (m)}_{ab}(\bm x)$, where $L^{(m)}_{ab}$  is the element at row $a$, column $b$ in $L^{(m)}$, the Cholesky factor for outcome $m$. Let $\epsilon^{(m)} \in \mathbb{R}^q$ denote the vector of i.i.d. $\mathcal N(0,1)$ samples corresponding to outcome $m$. Then we have
\begin{align*}
    \biggl|\frac{\partial }{\partial x_{ci}}\Bigl[&
    [\min \big[ S_{\{i_1, 1\}}\big(\mu(\bm x) + L(\bm x)\epsilon\big), \ldots, S_{\{i_j, 1\}}\big(\mu(\bm x) + L(\bm x)\epsilon\big)\big] - l_k^{(1)}\Bigr]_{+} \biggr|\\
    &\leq\Bigl|\Bigl[ \mu^{(m)}_{**} + L^{(m)}_{**}||\epsilon^{(m)}||_1 - l_k^{(m)}\Bigr]_{+}\Bigr|\\
    &\leq \Bigl|\mu^{(m)}_{**} + L^{(m)}_{**}||\epsilon^{(m)}||_1\Bigr| + \Bigl|l_k^{(m)}\Bigr|.
\end{align*}
Under our assumptions (compactness of $\mathcal{X}$, continuous differentiability of mean and covariance function) both $\bm \mu(\bm x)$ and $L(\bm x)$, as well as their respective gradients, are uniformly bounded. In particular there exist $C_1^{(m)}, C_2^{(m)} < \infty$ such that
\begin{align*}
    \bigl| S_{\{a, m\}}\big(\mu(\bm x) + L(\bm x)\epsilon\big) - l_k^{(m)} \bigr|
    &\leq C_1^{(m)} + C_2^{(m)}||\epsilon^{(m)}||_1
\end{align*} for all $a=i_1, ..., i_j$.

Hence,
\begin{align*}
    \biggl|\frac{\partial \tilde{A}(\bm x, \epsilon)}{\partial x_{ci}}\biggr|  &\leq \Biggl[\Bigl|\mu^{(1)}_{**} + C^{(1)}_{**}||\epsilon^{(1)}||_1\Bigr| + \Bigl|l_k^{(1)}\Bigr|\Biggr]\Biggl[C_1^{(2)} + C_2^{(2)}||\epsilon^{(2)}||_1\Biggr]\\ &+\Biggl[C_1^{(1)} + C_2^{(1)}||\epsilon^{(1)}||_1\Biggr] \Biggl[\Bigl|\mu^{(2)}_{**} + C^{(2)}_{**}||\epsilon^{(2)}||_1\Bigr| + \Bigl|l_k^{(2)}\Bigr|\Biggr]
\end{align*}
% Taking the expectation over $\bm \epsilon$ and noting that $\epsilon^{(m)}$ are independent for all $m$, we have 
% \begin{align*}
%     \mathbb E\bigg(\frac{\partial \tilde{A}(\bm x, \epsilon)}{\partial x_{ci}}\bigg) &\leq \Biggl[\Bigl|\mu^{(1)}_{**} + C^{(1)}_{**} \mathbb E[||\epsilon^{(1)}||_1]\Bigr| - \Bigl|l_k^{(1)}\Bigr|\Biggr]\Biggl[C_1^{(2)} + C_2^{(2)}\mathbb E[||\epsilon^{(2)}||_1]\Biggr]\\ &+\Biggl[C_1^{(1)} + C_2^{(1)}\mathbb E[||\epsilon^{(1)}||_1]\Biggr] \Biggl[\Bigl|\mu^{(2)}_{**} + C^{(2)}_{**}\mathbb E[||\epsilon^{(2)}||_1]\Bigr| - \Bigl|l_k^{(2)}\Bigr|\Biggr]\\
%     &\leq \Biggl[\Bigl|\mu^{(1)}_{**} + \frac{\pi}{2}qC^{(1)}_{**} \Bigr| - \Bigl|l_k^{(1)}\Bigr|\Biggr]\Biggl[C_1^{(2)} + \frac{\pi}{2}qC_2^{(2)}\Biggr]\\ &+\Biggl[C_1^{(1)} + \frac{\pi}{2}qC_2^{(1)}\Biggr] \Biggl[\Bigl|\mu^{(2)}_{**} + \frac{\pi}{2}qC^{(2)}_{**}\Bigr| - \Bigl|l_k^{(2)}\Bigr|\Biggr]
%     < \infty
% \end{align*}
%
Since $\epsilon$ is absolutely integrable,
$$\mathbb E\biggl( \biggl|\frac{\partial \tilde{A}(\bm x, \epsilon)}{\partial x_{ci}} \biggr| \biggr) < \infty.$$
Hence, $\mathbb E\bigl[\sup_{x_{ci} \notin \tilde{D}}\vert A'(\bm x, \epsilon) \vert \bigr] < \infty$. This can be extended to $M>2$ in the same manner using the product rule to obtain
\begin{align*}
    \mathbb E\bigg(\frac{\partial \tilde{A}(\bm x, \epsilon)}{\partial x_{ci}}\bigg) &\leq \sum_{m=1}^M \Bigg(\Biggl[\Bigl|\mu^{(m)}_{**} + C^{(m)}_{**} \mathbb E[||\epsilon^{(m)}||_1]\Bigr| + \Bigl|l_k^{(1)}\Bigr|\Biggr]\prod_{n=1,n\neq m}^M\Biggl[C_1^{(n)} + C_2^{(n)}\mathbb E[||\epsilon^{(n)}||_1]\Biggr]\Bigg)\\
    &\leq \sum_{m=1}^M \Bigg(\Biggl[\Bigl|\mu^{(m)}_{**} + \frac{\pi}{2}qC^{(m)}_{**}\Bigr| + \Bigl|l_k^{(1)}\Bigr|\Biggr]\prod_{n=1,n\neq m}^M\Biggl[C_1^{(n)} + \frac{\pi}{2}qC_2^{(n)}]\Biggr]\Bigg).
\end{align*}
Hence, $\mathbb E\bigl[\sup_{x_{ci} \notin \tilde{D}}\vert A'(\bm x, \epsilon) \vert \bigr] < \infty$ for $M \geq 2$ and \citet[Theorem 1]{glasserman88} holds.
\end{proof}

\clearpage
%%%%%%%%%%%%%%%%%%%%%%%%%%%%%%%%%%%%
\section{Monte-Carlo Approximation}
\label{appdx:sec:MCApprox}
Figure \ref{fig:mc_gradient} shows the gradient of analytic \EHVI{} and the MC estimator \qEHVI{} on slice of a 3-objective problem. Even using only $N=32$ QMC samples, the average sample gradient has very low variance. Moreover, fixing the base samples also greatly reduces the variance without introducing bias.
\begin{figure}[ht]
    \centering
    \begin{subfigure}{\textwidth}
        \centering
        \includegraphics[width=\linewidth]{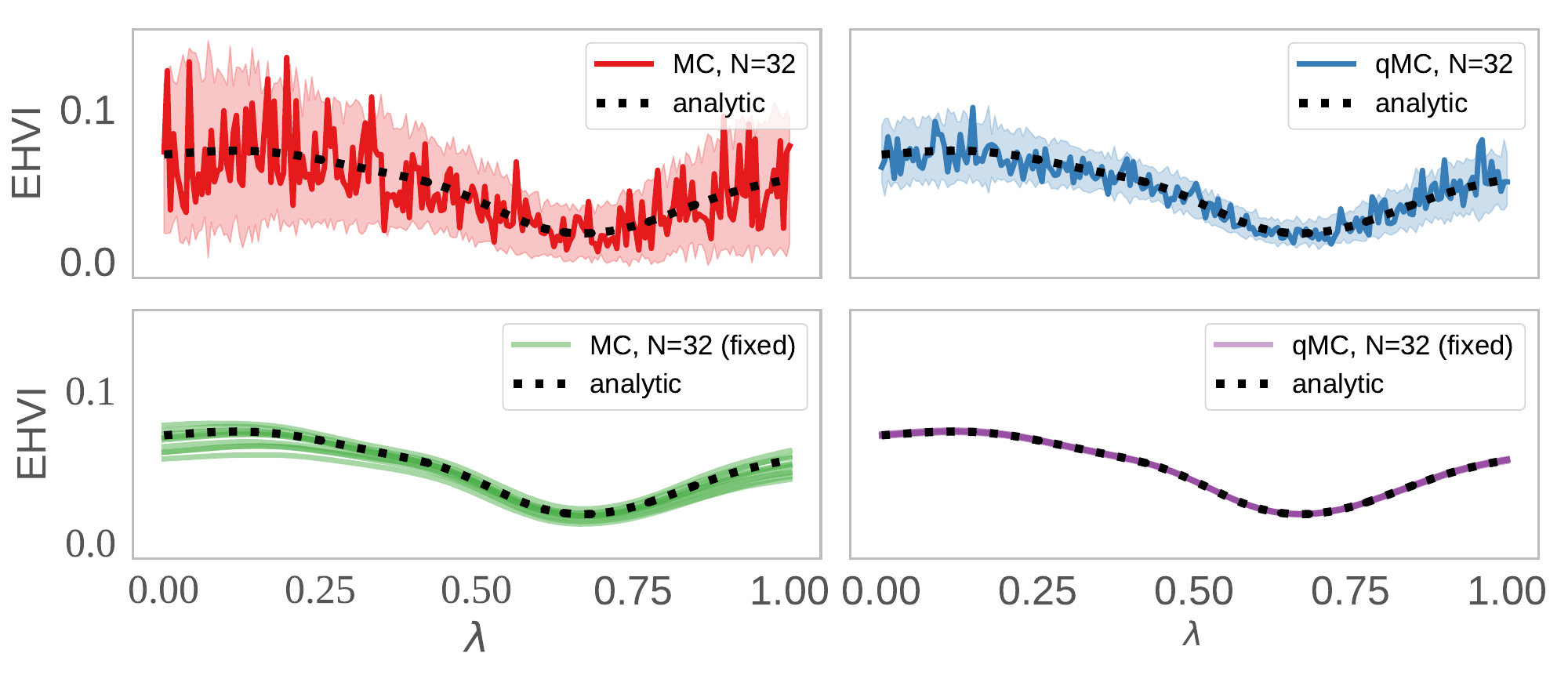}
    \subcaption{\label{fig:mc_acqf} A comparison of the analytic \EHVI{} acquisition function and the MC-based \qEHVI{} for $q=1$.
    }
    \end{subfigure} %
    \begin{subfigure}{\textwidth}    
        \centering
        \includegraphics[width=\linewidth]{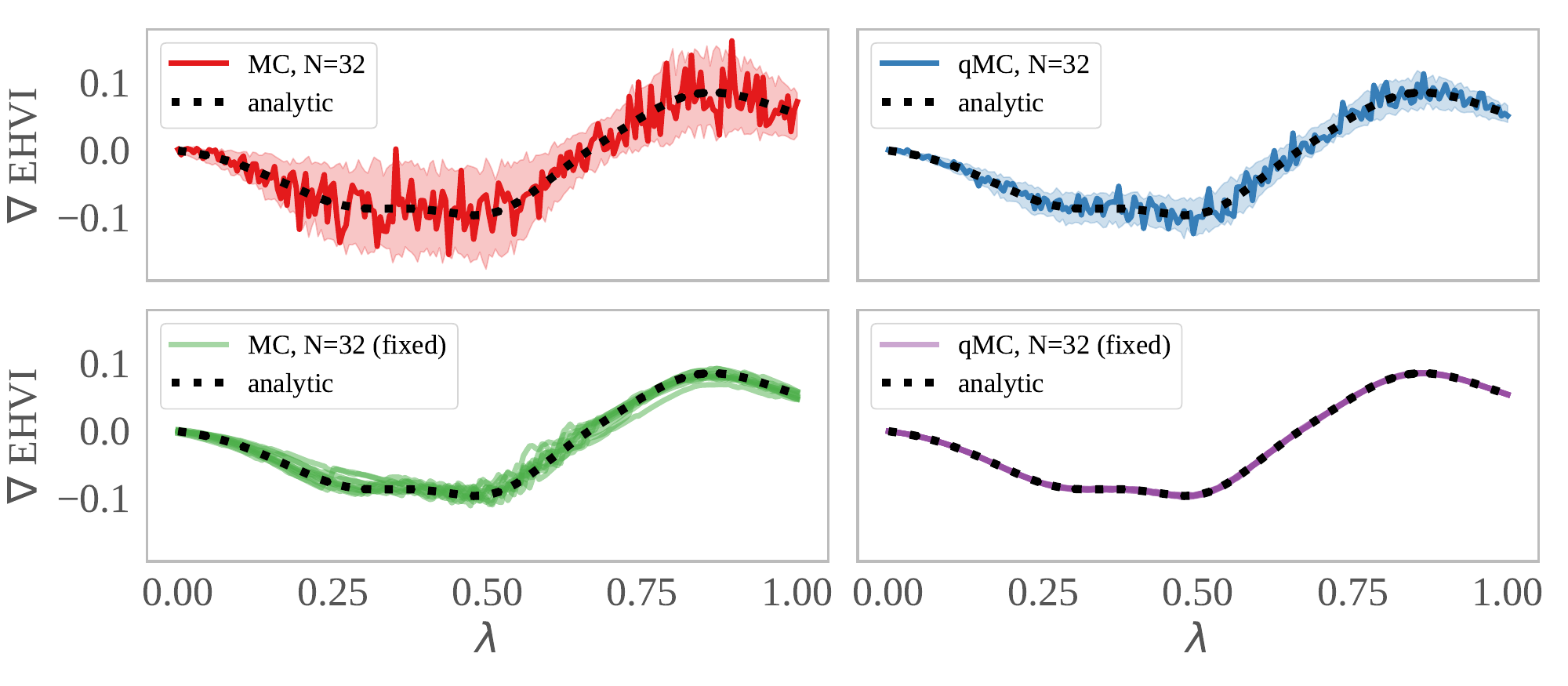} 
        \subcaption{\label{fig:mc_gradient}A comparison of the exact gradient of analytic \EHVI{} and the exact sample average gradient of the MC-based \qEHVI{} for $q=1$.}
    \end{subfigure}
    \caption{A comparison of (a) the analytic \EHVI{} and MC-based \qEHVI{} for $q=1$ and (b) a comparison of the exact gradient $\nabla \alpha_{\EHVI}$ of analytic \EHVI{} and average sample gradient of the MC-estimator $\nabla \hat{\alpha}_{\text{\qEHVI{}}}$ over a slice of the input space on a DTLZ2 problem ($q=1$, $M=3$, $d=6$) \citep{dtlz}.  $x^{(0)}$ is varied across $0 \leq \lambda \leq 1$, while $x^{(i)}$ for $1, ... D$ are held constant. In each of (a) and (b), the top row show \qEHVI{} where the (quasi-)standard normal base samples are resampled for each value of $x^{(0)}$. The solid line is one sample average (across (q)MC samples) and the shaded area is the mean plus 2 standard errors across 50 repetitions. The bottom row uses the same base samples for evaluating each test point and the sample average for each of 50 repetitions is plotted.}
\end{figure}
\FloatBarrier
%%%%%%%%%%%%%%%%%%%%%%%%%%%%%
\section{Experiment Details}
\label{sec:experiment_details}
\subsection{Algorithms}
\label{sec:algo_details}
For TS-TCH, we draw a sample from the joint posterior over a discrete set of $1000d$ points sampled from a scrambled Sobol sequence. For PESMO, we follow \cite{garridomerchn2020parallel} and use a Pareto set of size 10 for each sampled GP, which is optimized over a discrete set of $1000d$ points sampled from a scrambled Sobol sequence. The current Pareto front is approximated by optimizing the posterior means over a grid as is done in \citet{garrido2019predictive, garridomerchn2020parallel}. For SMS-EGO, we use the observed Pareto front. 
All acquisition functions are optimized with L-BFGS-B (with a maximum of 200 iterations); SMS-EGO \citep{smsego} and PESMO \citep{garrido2019predictive} use gradients approximated by finite differences and all other methods use exact gradients. For all methods, each outcome is modeled with an independent Gaussian process with a Matern $5/2$ ARD kernel. The methods implemented in Spearmint use a fully Bayesian treatment of the hyperparameters with 10 samples from posterior over the hyperparamters, and the methods implemented in BoTorch use maximum a posteriori estimates of the GP hyperparameters. All methods are initialized with $2(d+1)$ points from a scrambled Sobol sequence. \qParego{} and \qEHVI{} use $N=128$  QMC samples. 

\subsubsection{Reference point specification}
\label{appdx:sec:RefPoint}
There is a large body of literature on the effects of reference point specification \citep{Auger09, Ishibuchi11, ishibuchi18}. The hypervolume indicator is sensitive to specified the reference point: a reference point that is far away from the Pareto front will favor extreme points, where as reference point that is close to the Pareto front gives more weight to less extreme points \citep{ishibuchi18}. Sensitivity to the reference point is affects both the evaluation of different MO methods and the utility function for methods that rely \HV{}. In practice, a decision maker may be able to specify a reference point that satisfies their preference with domain knowledge. If a reference point is provided by the decision maker, previous work has suggested heuristics for choosing reference points for use in an algorithm's utility function \citep{Ishibuchi11, smsego}. We follow previous work \citep{yang2019, yang_emmerich2019} and assume that the reference point is known.

We also considered (but did not use in our experiments) a dynamic reference point strategy where at each BO iteration, the reference point is selected to be a point slightly worse than the nadir (component-wise minimum) point of the current observed Pareto front for computing the acquisition function: $\bm r = \bm y_\text{nadir} -  0.1\cdot \vert \bm y_\text{nadir}\vert$ where $\bm y_\text{nadir} = \big(\min_{y^{(1)} \in \mathcal D^{(1)}} y^{(1)}, \ldots, \min_{y^{(m)} \in \mathcal D^{(m)}} y^{(m)}\big)$. This reference point is used in SMS-EMOA in  \citet{Ishibuchi11}), and we find similar average performance (but higher variance) on problems to using a known reference point with continuous Pareto fronts. If the Pareto front is discontinuous, then it is possible not all sections of the Pareto front will be reached. 
  
\subsubsection{\qParego{}}
\label{appdx:sec:qparego}
Previous work has only considered unconstrained sequential optimization with ParEGO \citep{parego, bradford18} and ParEGO is often optimized with gradient-free methods \citep{smsego}. To the best of our knowledge, \qParego{} is the first to support parallel and constrained optimization. Moreover, we compute exact gradients via auto-differentiation for acquisition optimization. ParEGO is typically implemented by applying augmented Chebyshev scalarization and modeling the scalarized outcome \citep{parego}. However, recent work has shown that composite objectives offer improved optimization performance \citep{astudillo2019composite}. \qParego{} uses a MC-based Expected Improvement \citep{jones98} acquisition function, where the objectives are modeled independently and the augmented Chebyshev scalarization \citep{parego} is applied to the posterior samples as a composite objective. This approach enables the use of sequential greedy optimization of $q$ candidates with proper integration over the posterior at the pending points. Importantly, the sequential greedy approach allows for using different random scalarization weights for selecting each of the $q$ candidates. \qParego{} is extended to the constrained setting by weighting the EI by the probability of feasibility \citep{gardner2014constrained}. We estimate the probability of feasiblity using the posterior samples and approximate the indicator function with a sigmoid to maintain differentiablity as in constrained \qEHVI{}. \qParego{} is trivially extended to the noisy setting using Noisy Expected Improvement \citep{letham2019noisyei, balandat2020botorch}, but we use Expected Improvement in our experiments as all of the problems are noiseless.
\subsection{Benchmark Problems}
The details for the benchmark problems below assume minimization of all objectives. Table \ref{table:ref_points} provides the reference points used for all benchmark problems.
\label{sec:problem_details}
\begin{table*}[t!]
\centering
\caption{\label{table:ref_points} Reference points for all benchmark problems. Assuming minimization. In our benchmarks, equivalently maximize the negative objectives and multiply the reference points by -1.}
\begin{small}
\begin{sc}
\begin{tabular}{lc}
\toprule
Problem & Reference Point\\
\midrule
BraninCurrin & (18.0, 6.0)\\
DTLZ2 & $(1.1, ..., 1.1) \in \mathbb R^M$\\
ABR & (-150.0, 3500.0, 5.1) \\
Vehicle Crash Safety & (1864.72022, 11.81993945, 0.2903999384)\\
ConstrainedBraninCurrin & (90.0, 10.0)\\
C2-DTLZ2 & $(1.1, ..., 1.1) \in \mathbb R^M$\\
\bottomrule
\end{tabular}
\end{sc}
\end{small}
\end{table*}

\textbf{Branin-Currin}
\begin{align*}
f^{(1)}(x_1', x_2') &=  (x_2 - \frac{5.1}{4 \pi^ 2} x_1^2 + \frac{5}{\pi} x_1 - r)^2 + 10 (1-\frac{1}{8\pi}) \cos(x_1) + 10\\
f^{(2)}(x_1, x_2) &= \bigg[1 - \exp\bigg(-\frac{1} {(2x_2)}\bigg)\bigg] \frac{2300 x_1^3 + 1900x_1^2 + 2092 x_1 + 60}{100 x_1^3 + 500x_1^2 + 4x_1 + 20}
\end{align*}
where $x_1, x_2 \in [0,1]$, $x_1' = 15x_1 - 5$, and $x_2' = 15x_2$.

The constrained Branin-Currin problem uses the following disk constraint from \citep{gelbart2014unknowncon}:
 $$c(x_1', x_2') = 50 - (x_1' - 2.5)^2 - (x_2' - 7.5)^2
)  \geq 0$$

\textbf{DTLZ2}
The objectives are given by \citep{dtlz}:
\begin{align*}
    f_1(\bm x) &= (1+ g(\bm x_M))\cos\big(\frac{\pi}{2}x_1\big)\cdots\cos\big(\frac{\pi}{2}x_{M-2}\big) \cos\big(\frac{\pi}{2}x_{M-1}\big)\\
    f_2(\bm x) &= (1+ g(\bm x_M))\cos\big(\frac{\pi}{2}x_1\big)\cdots\cos\big(\frac{\pi}{2}x_{M-2}\big) \sin\big(\frac{\pi}{2}x_{M-1}\big)\\
    f_3(\bm x) &= (1+ g(\bm x_M))\cos\big(\frac{\pi}{2}x_1\big)\cdots\sin\big(\frac{\pi}{2}x_{M-2}\big)\\
    \vdots\\
     f_M(\bm x) &= (1+ g(\bm x_M))\sin\big(\frac{\pi}{2}x_1\big)
\end{align*}
where $g(\bm x) = \sum_{x_i \in \bm x_M} (x_i - 0.5)^2, \bm x \in [0,1]^d,$ and $\bm x_M$ represents the last $d - M +1$ elements of $\bm x$.

The C2-DTLZ2 problem adds the following constraint \citep{deb2019}:
$$c(\bm x) = - \min \bigg[\min_{i=1}^M\bigg((f_i(\bm x) -1 )^2 + \sum_{j=1, j=i}^M (f_j^2 - r^2) \bigg), \bigg(\sum_{i=1}^M \big((f_i(\bm x) - \frac{1}{\sqrt{M}})^2 - r^2\big)\bigg)\bigg]\geq 0$$

\textbf{Vehicle Crash Safety}
The objectives are given by \citep{tanabe2020}:
\begin{align*}
    f_1(\bm x) &=1640.2823 + 2.3573285x_1 + 2.3220035x_2 + 4.5688768x_3 + 7.7213633x_4 + 4.4559504x_5\\
    f_2(\bm x) &= 6.5856+ 1.15x_1 - 1.0427x_2 + 0.9738x_3+ 0.8364x_4 - 0.3695x_1x_4 + 0.0861x_1x_5\\
    &+ 0.3628x_2x_4 +  0.1106x_1^2 - 0.3437x_3^2 + 0.1764x_4^2\\
    f_3(\bm x) &= -0.0551 + 0.0181x_1 + 0.1024x_2 + 0.0421x_3 - 0.0073x_1x_2 + 0.024x_2x_3-0.0118x_2x_4\\ 
    &- 0.0204x_3x_4 - 0.008x_3x_5 - 0.0241x_2^2 + 0.0109x_4^2
\end{align*}
where $\bm x \in [1,3]^5$.

\textbf{Policy Optimization for Adaptive Bitrate Control} The controller is given by:
$a_t = x_0\hat{z}_{\text{bd},t} + x_2z_{\text{bf}, t} + x_3$, where $\hat{z}_{\text{bd},t} = \frac{\sum_{t_i < t} z_{\text{bd}, t_i}\exp(-x_1 t_i)}{\sum_{t_i < t} \exp(-x_1 t_i)}$ is estimated bandwidth at time $t$ using an exponential moving average, $z_{\text{bf}, t}$ is the buffer occupancy at time $t$, and $x_0, ... x_3$ are the parameters we seek to optimize. We evaluate each policy on a set of 400 videos, where the number of time steps (chunks) in each video stream trajectory depends on the size of the video.
%%%%%%%%%%%%%%%%%%%%%%%%%%%%%%%
\clearpage
\section{Additional Empirical Results}
\subsection{Additional Sequential Optimization Results}
\label{appdx:sec:extra_synthetic_experiments}
We include results for an additional synthetic benchmark: the \textbf{DTLZ2} problem from the MO literature \citep{dtlz} ($d=6, M=2$). Figure \ref{fig:dtlz2} shows that \qEHVI{} outperforms all other baseline algorithms on the DTLZ2 in terms of sequential optimization performance with competitive wall times as shown in \ref{table:latency:extra_problems}.

\begin{figure}[ht]
    \centering
    \includegraphics[width=0.49\linewidth]{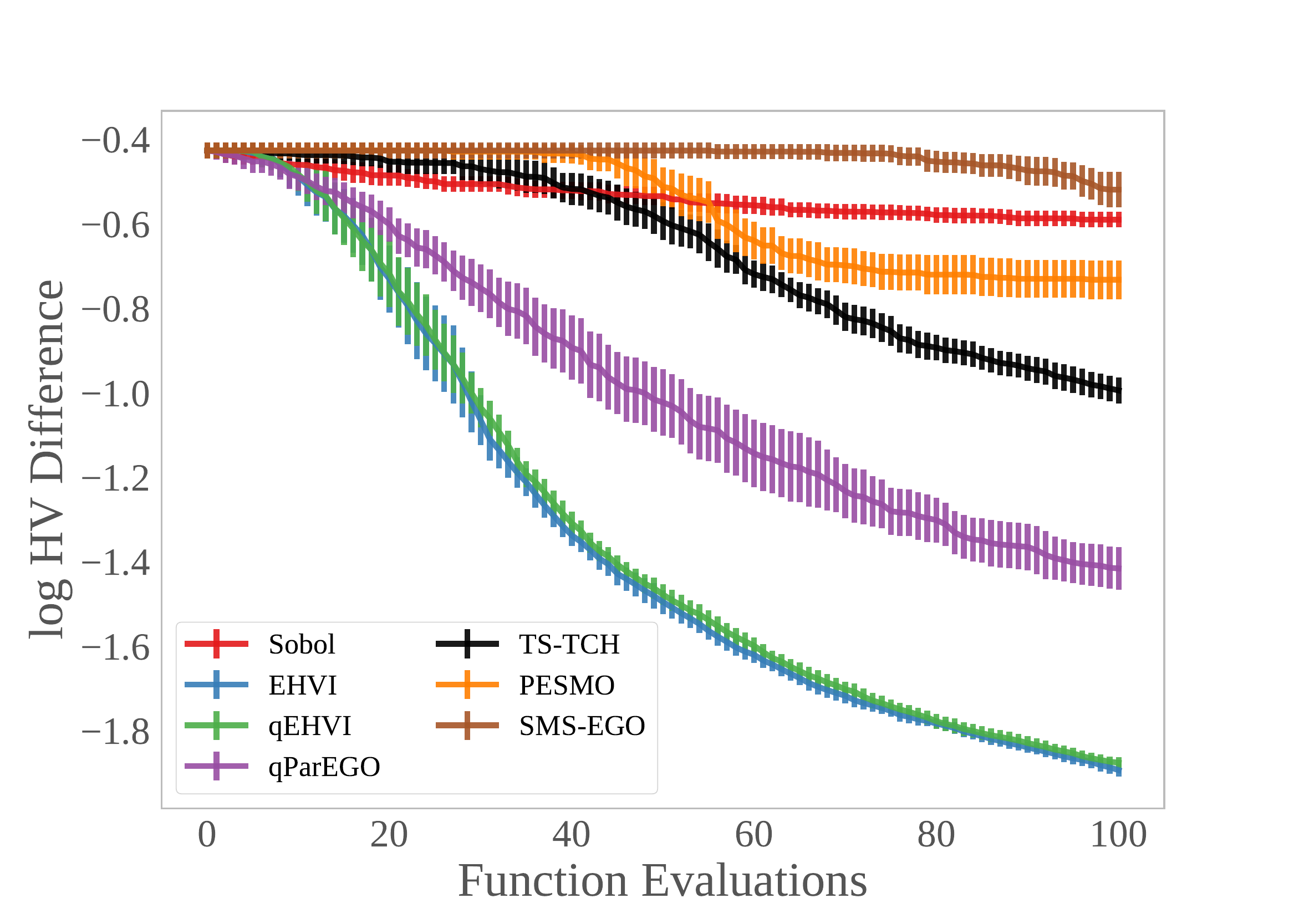}
        
    \caption{\label{fig:dtlz2}Optimization performance on the DTLZ2 synthetic function ($d=6, M=2$).}
\end{figure}
\FloatBarrier
\begin{table*}[t!]
\centering
\caption{\label{table:latency:extra_problems} Acquisition Optimization wall time in seconds on a CPU (2x Intel Xeon E5-2680 v4 @ 2.40GHz) and on a GPU (Tesla V100-SXM2-16GB). The mean and two standard errors are reported. NA indicates that the algorithm does not support constraints.}
\begin{small}
\begin{sc}
\begin{tabular}{lcc}
\toprule
\textbf{CPU}& ConstrainedBraninCurrin & DTLZ2 \\
\midrule
PESMO (\textit{q}=1) & NA & $278.53 ~(\pm 25.66)$\\
SMS-EGO (\textit{q}=1) & NA & $104.26 ~(\pm 7.66)$\\
TS-TCH (\textit{q}=1) & NA & $52.55 ~(\pm 0.06)$\\
\qParego{} (\textit{q}=1) & $2.4 ~(\pm 0.37)$ & $4.68 ~(\pm 0.46)$\\
\EHVI{} (\textit{q}=1) & NA& $3.58 ~(\pm 0.28)$\\
\qEHVI{} (\textit{q}=1) &$5.69 ~(\pm 0.43)$  & $5.95 ~(\pm 0.45)$\\
\midrule
\textbf{GPU}& ConstrainedBraninCurrin & DTLZ2\\
\midrule
TS-TCH (\textit{q}=1)& NA & $0.25 ~(\pm 0.00)$\\
TS-TCH (\textit{q}=2) & NA & $0.27 ~(\pm 0.00)$\\
TS-TCH (\textit{q}=4) &NA & $0.28 ~(\pm 0.00)$\\
TS-TCH (\textit{q}=8) & NA & $0.32 ~(\pm 0.01)$\\
\qParego{} (\textit{q}=1)& $3.52 ~(\pm 0.34)$ & $9.04 ~(\pm 0.93)$\\
\qParego{} (\textit{q}=2) &$6.0 ~(\pm 0.56)$ & $14.23 ~(\pm 1.55)$\\
\qParego{} (\textit{q}=4) &$12.07 ~(\pm 0.98)$ & $40.5 ~(\pm 3.21)$\\
\qParego{} (\textit{q}=8) &$33.1 ~(\pm 3.32)$ & $84.15 ~(\pm 6.9)$\\
\EHVI{} (\textit{q}=1) & NA & $84.15 ~(\pm 6.9)$\\
\qEHVI{} (\textit{q}=1) & $5.61 ~(\pm 0.17)$ & $10.21 ~(\pm 0.58)$\\
\qEHVI{} (\textit{q}=2) &$19.06 ~(\pm 5.88)$ & $17.75 ~(\pm 0.97)$\\
\qEHVI{} (\textit{q}=4) &$29.26 ~(\pm 2.01)$ & $40.41 ~(\pm 2.78)$\\
\qEHVI{} (\textit{q}=8) &$91.56 ~(\pm 5.51)$& $106.51 ~(\pm 7.69)$\\
\bottomrule
\end{tabular}
\end{sc}
\end{small}
\end{table*}

\FloatBarrier
\subsection{Performance with Increasing Parallelism}
Figure \ref{fig:q_anytime} shows that that the performance of \qEHVI{} performance does not degrade substantially, whereas performance does degrade for \qParego{} and \textsc{TS-TCH} on some benchmark problems. We include results for all problems in Section~\ref{sec:Experiments} and Appendix~\ref{appdx:sec:extra_synthetic_experiments} as well as a \textbf{Constrained Branin-Currin} problem (which is described in Appendix~\ref{sec:problem_details}). 
\begin{figure}[ht]
    \centering
    \begin{subfigure}{.49\textwidth}
        \centering
    \includegraphics[width=\linewidth]{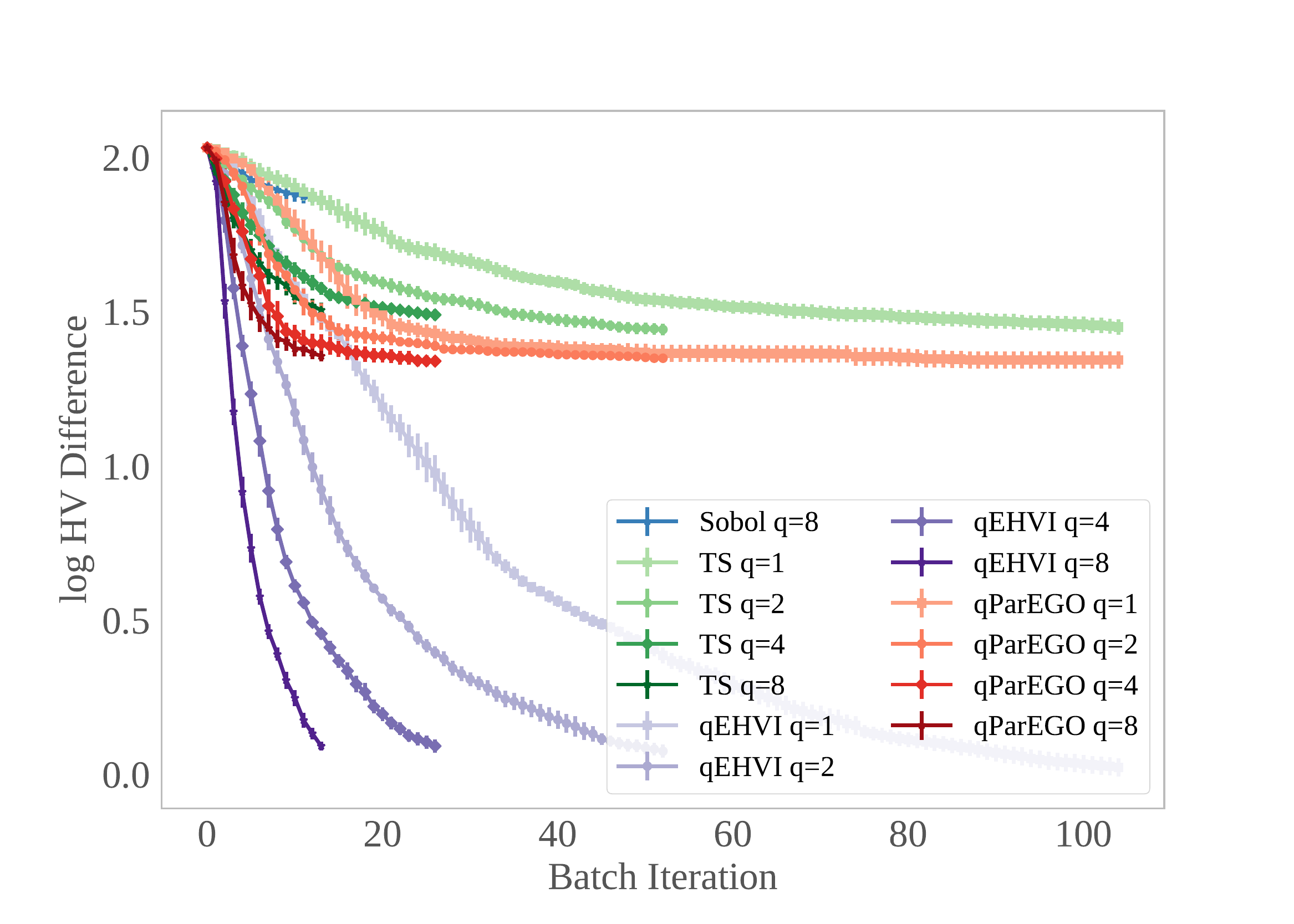}
        \subcaption{\textsc{VehicleSafety}\label{fig:vehicle_batch_iteration}}
    \end{subfigure} %
    \hspace{0.5ex}
    \begin{subfigure}{.49\textwidth}    
        \centering
        \includegraphics[width=\linewidth]{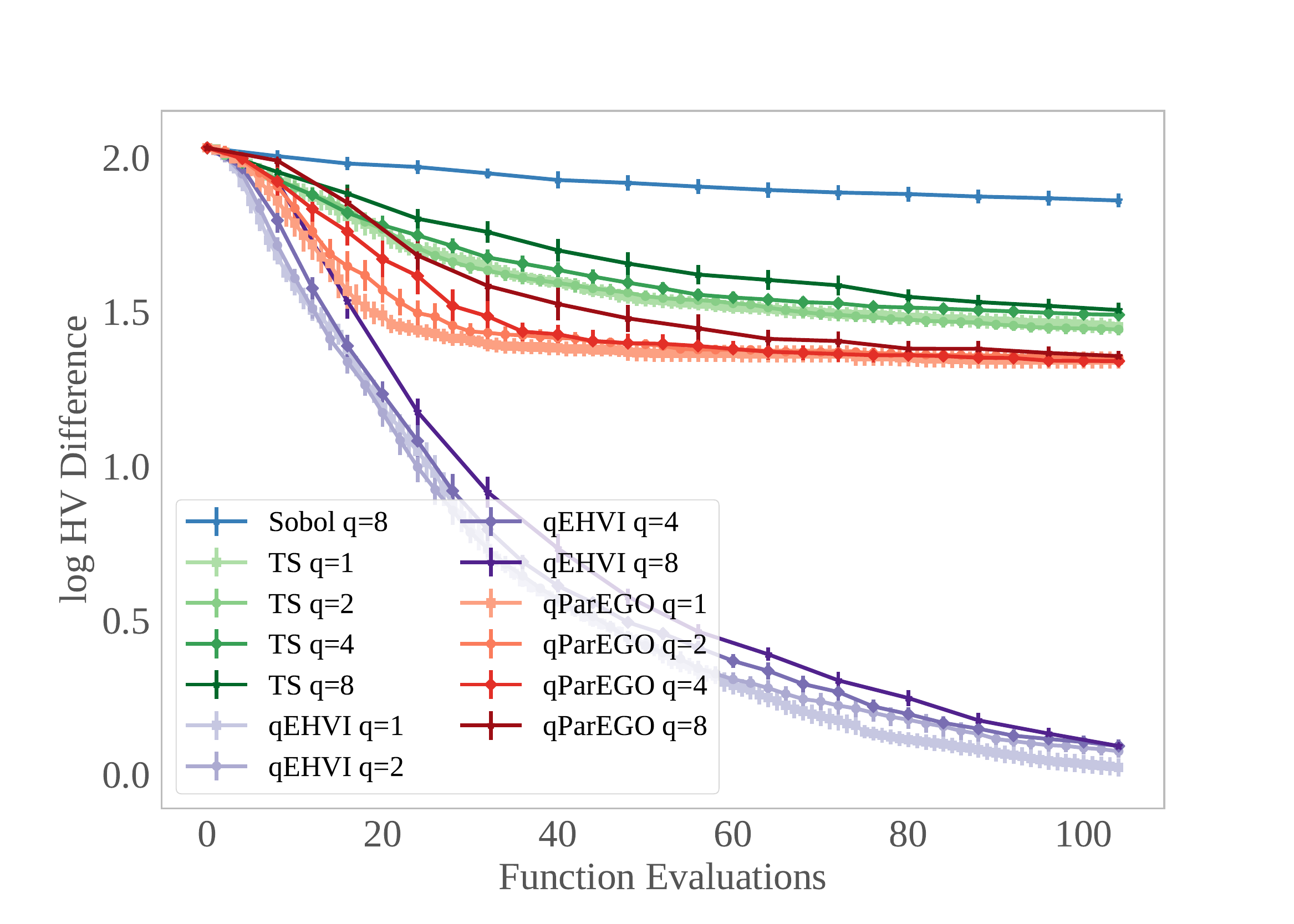}
        \subcaption{\textsc{VehicleSafety}\label{fig:vehicle_func_evals}}
    \end{subfigure}
    \begin{subfigure}{.49\textwidth}
        \centering
    \includegraphics[width=\linewidth]{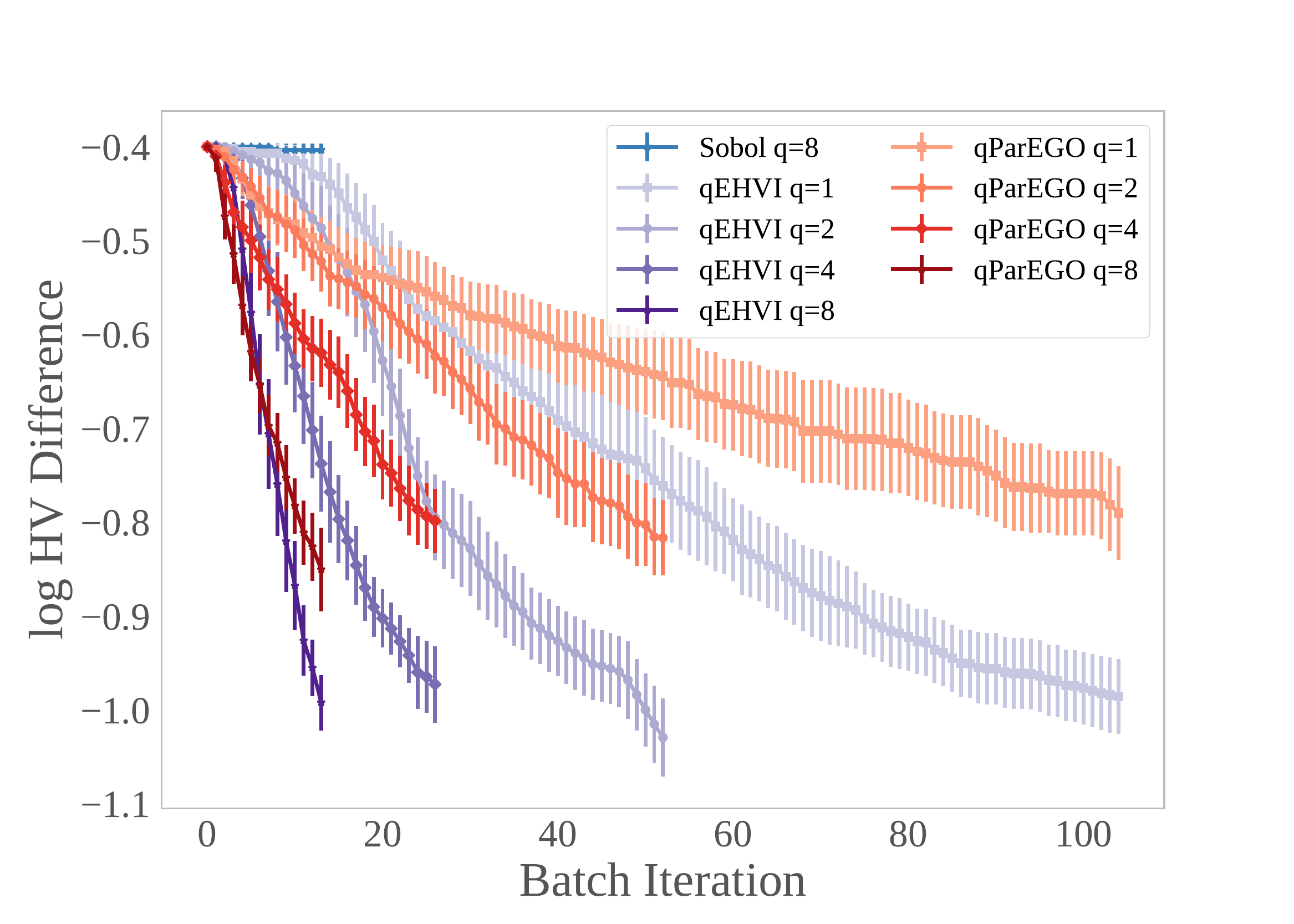}
        \subcaption{\textsc{C2DTLZ2}\label{fig:c2dtlz2_batch_iteration}}
    \end{subfigure} %
    \hspace{0.5ex}
    \begin{subfigure}{.49\textwidth}    
        \centering
        \includegraphics[width=\linewidth]{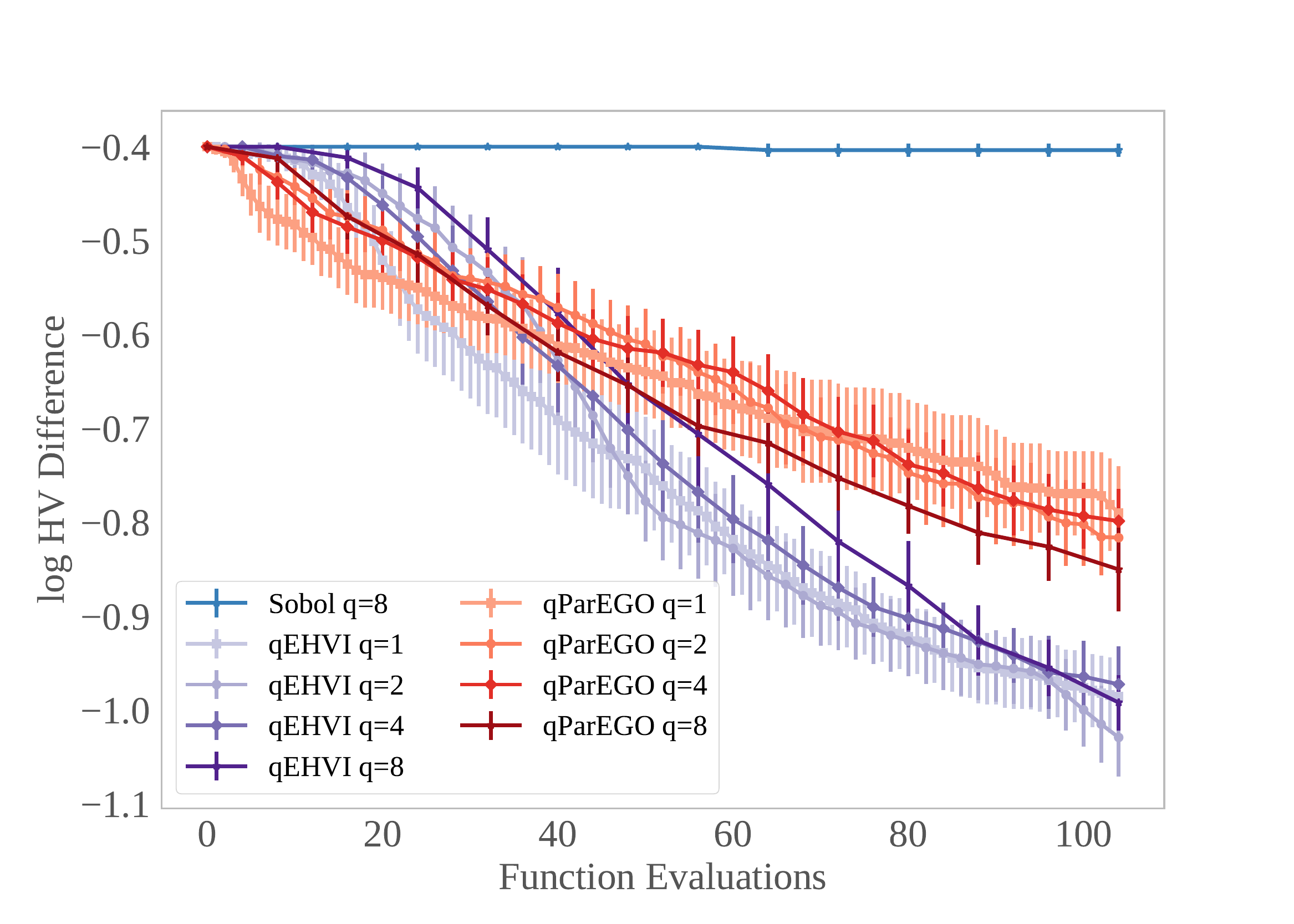}
        \subcaption{\textsc{C2DTLZ2}\label{fig:c2dtlz2_func_evals}}
    \end{subfigure}
        \begin{subfigure}{.49\textwidth}
        \centering
    \includegraphics[width=\linewidth]{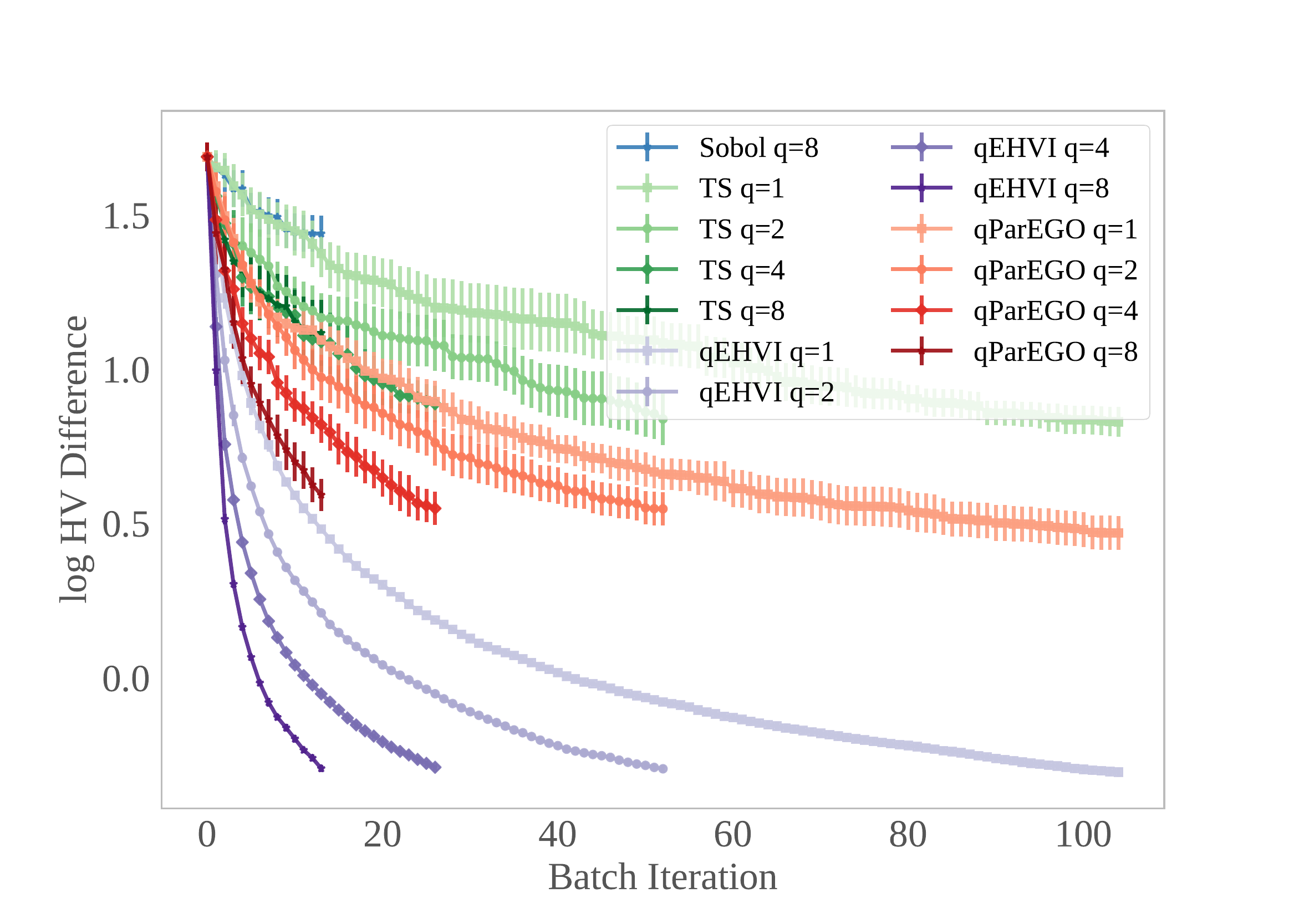}
        \subcaption{\textsc{BraninCurrin}\label{fig:bc_batch_iteration}}
    \end{subfigure} %
    \hspace{0.5ex}
    \begin{subfigure}{.49\textwidth}    
        \centering
        \includegraphics[width=\linewidth]{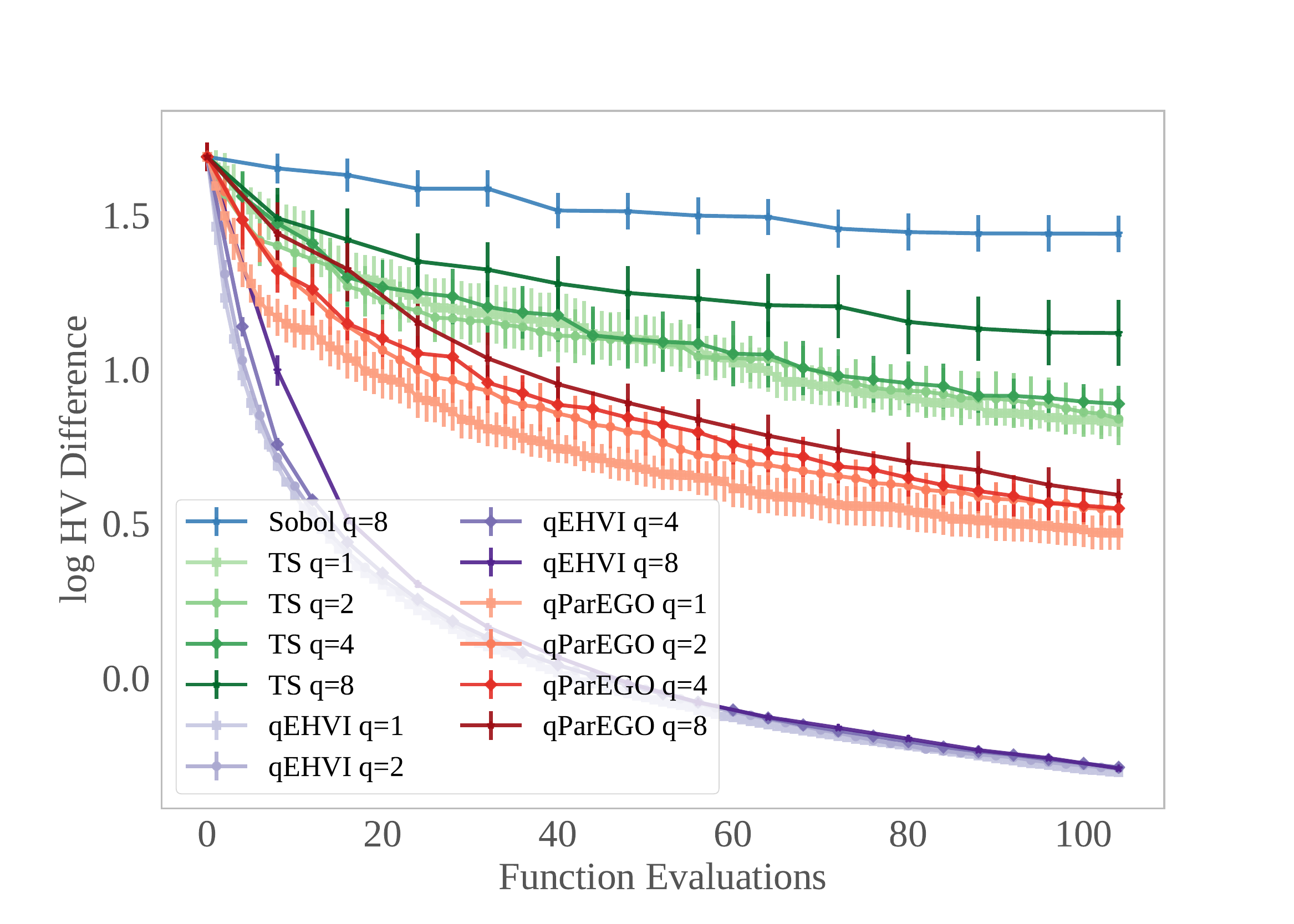}
        \subcaption{\textsc{BraninCurrin}\label{fig:bc_func_evals}}
    \end{subfigure}
    \caption{Optimization performance of parallel acquisition functions over \emph{batch BO iterations} (left) and \emph{function evaluations} (right) for benchmark problems in Section~\ref{sec:Experiments}.}
\label{fig:q_anytime}
\end{figure}

\begin{figure}[ht]
    \centering
    \begin{subfigure}{.49\textwidth}
        \centering
    \includegraphics[width=\linewidth]{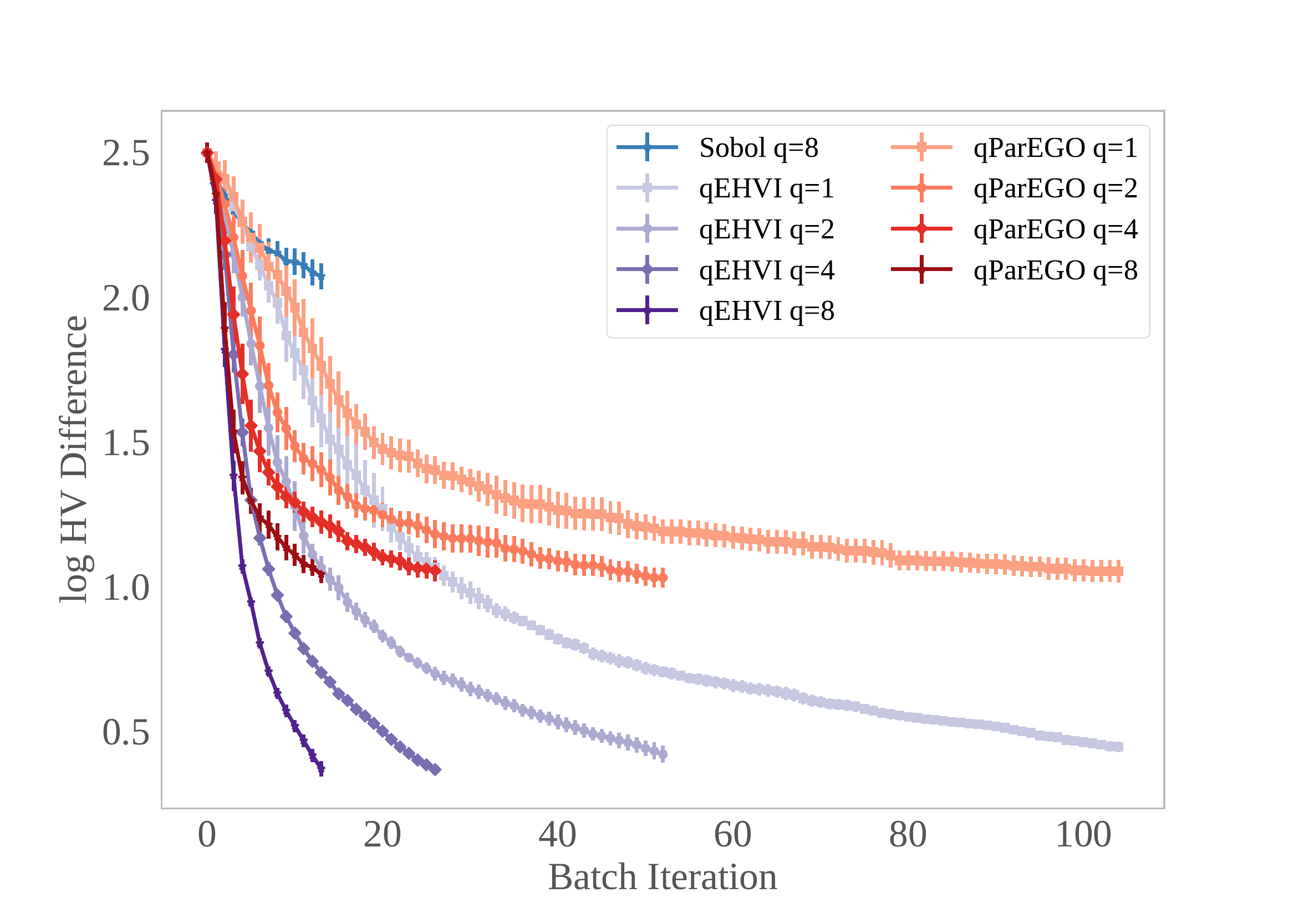}
        \subcaption{\textsc{ConstrainedBraninCurrin}\label{fig:constrained_bc_batch_iteration}}
    \end{subfigure} %
    \hspace{0.5ex}
    \begin{subfigure}{.49\textwidth}    
        \centering
        \includegraphics[width=\linewidth]{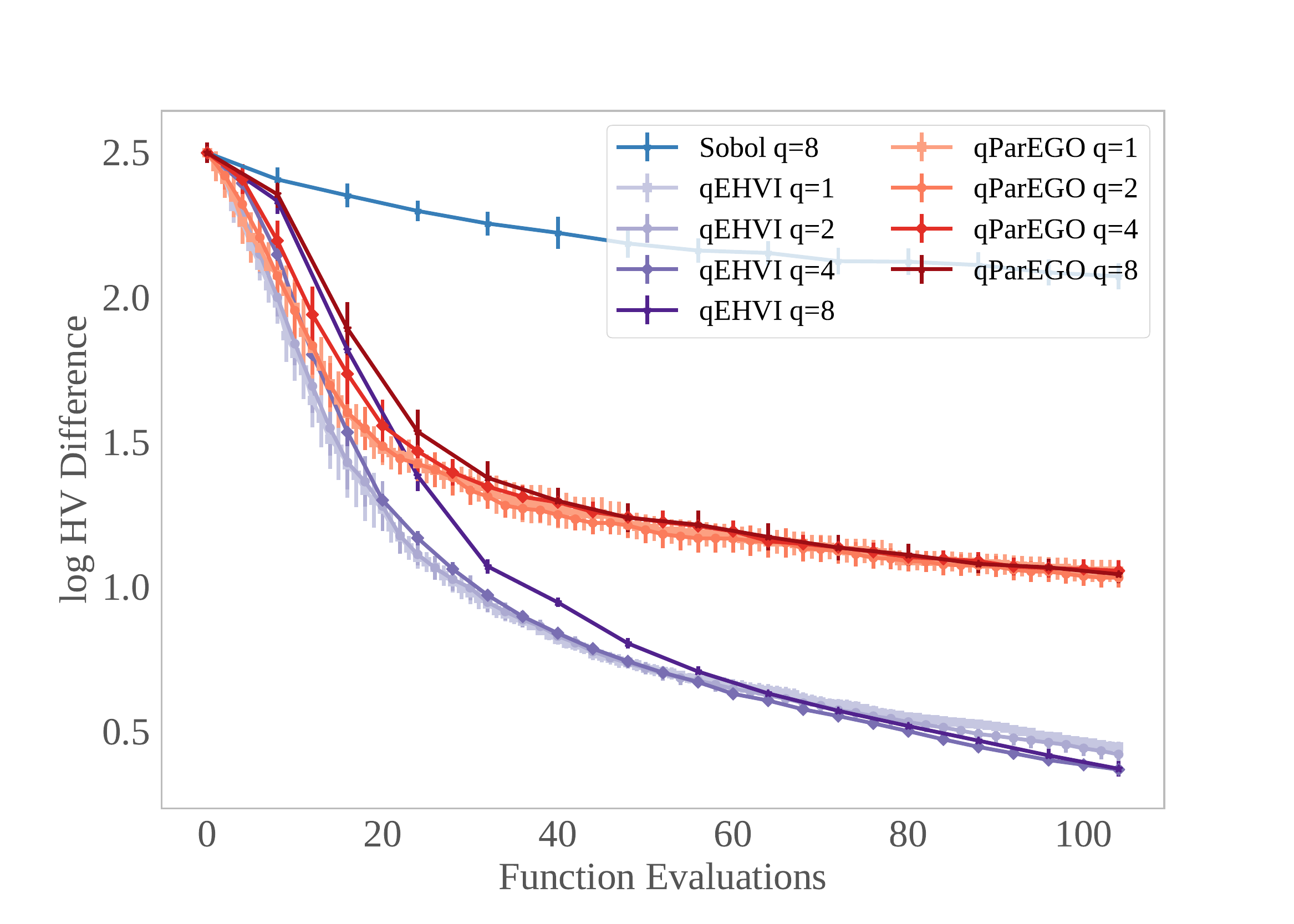}
        \subcaption{ \textsc{ConstrainedBraninCurrin}\label{fig:constrained_bc_func_evals}}
    \end{subfigure}
    \begin{subfigure}{.49\textwidth}
        \centering
    \includegraphics[width=\linewidth]{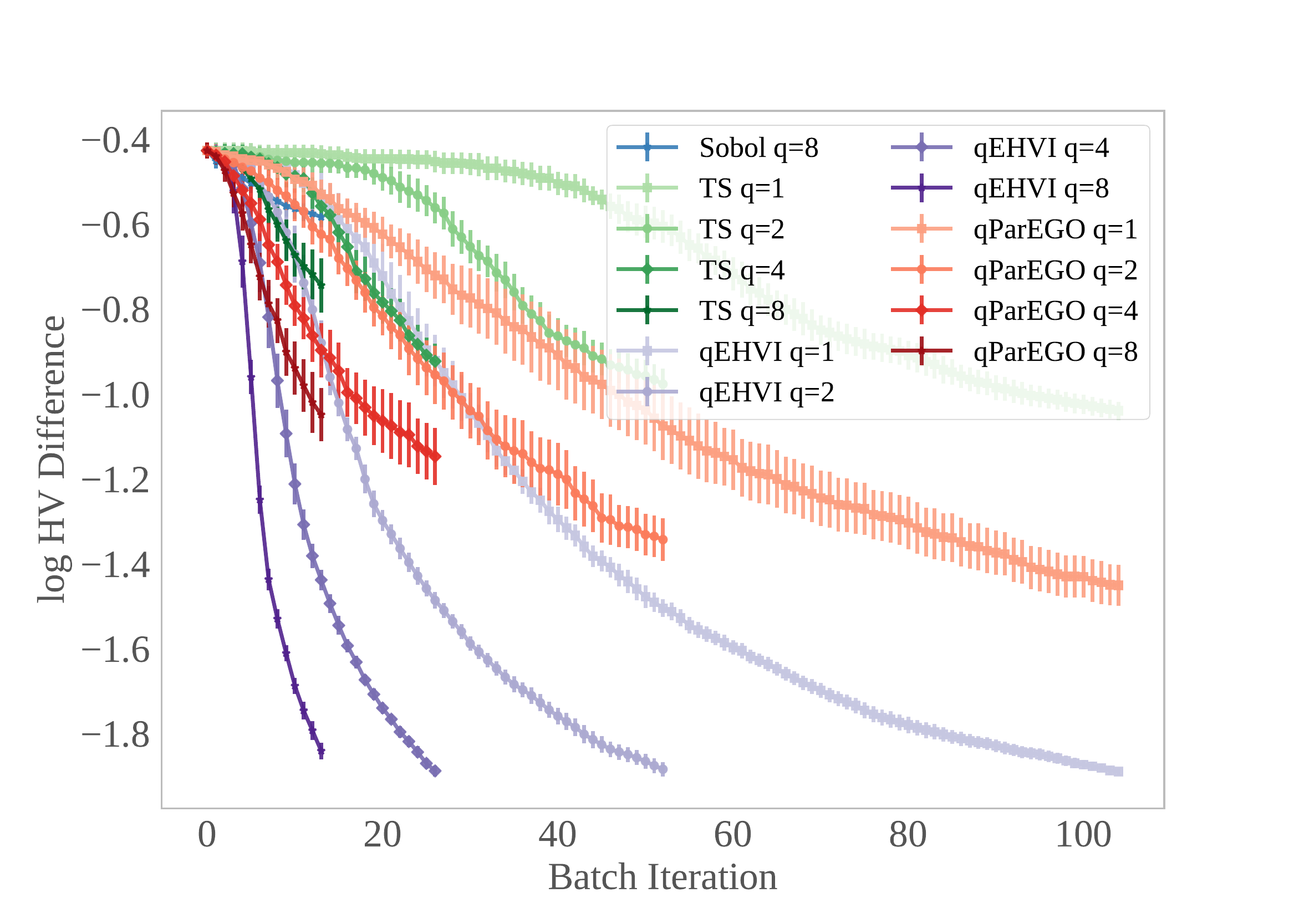}
        \subcaption{\textsc{DTLZ2} ($M=2, d=6$)\label{fig:dtlz2_batch_iteration}}
    \end{subfigure} %
    \hspace{0.5ex}
    \begin{subfigure}{.49\textwidth}    
        \centering
        \includegraphics[width=\linewidth]{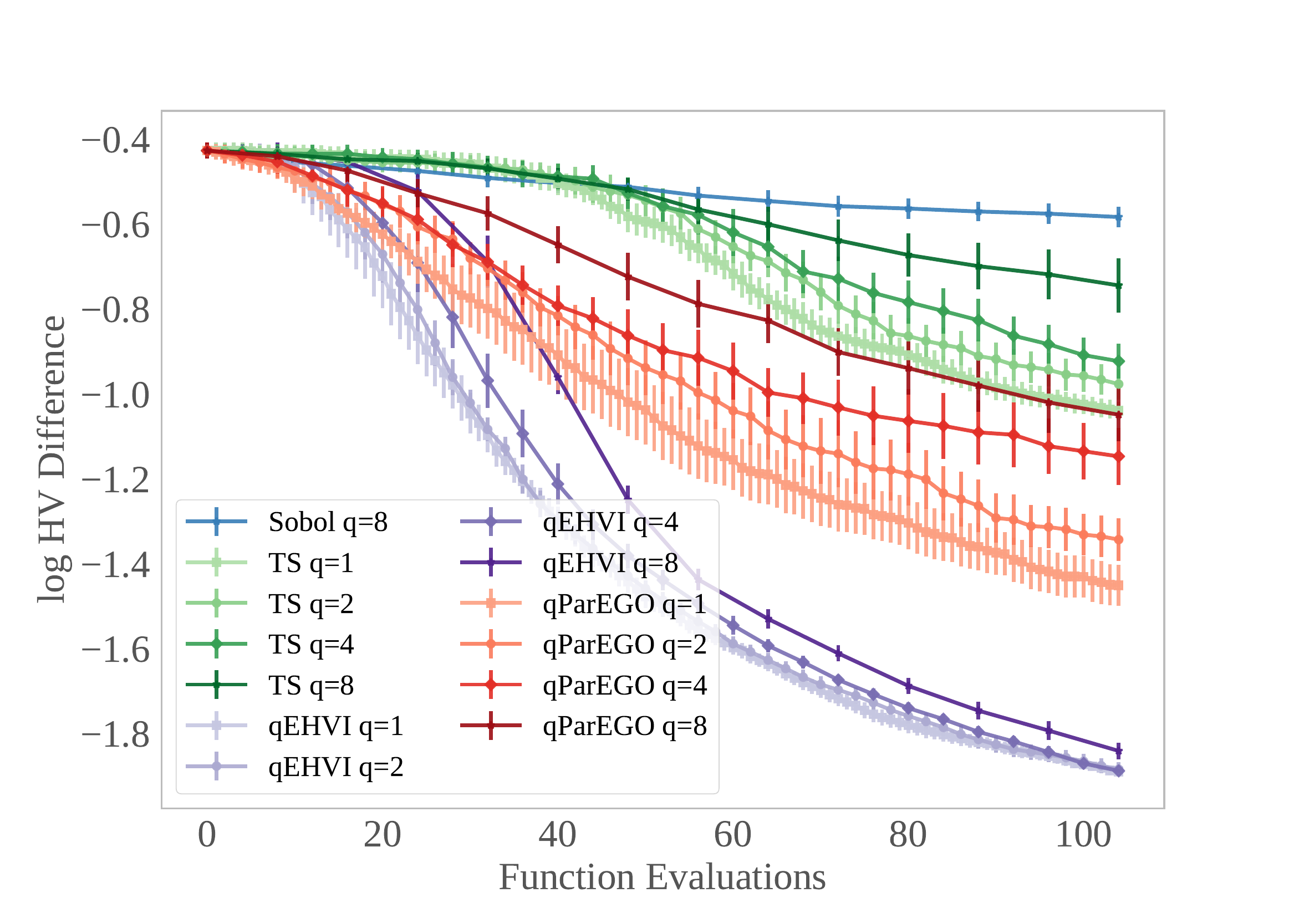}
        \subcaption{ \textsc{DTLZ2} ($M=2, d=6$)\label{fig:dtlz2_func_evals}}
    \end{subfigure}
    \caption{Optimization performance of parallel acquisition functions over \emph{batch BO iterations} (left) and \emph{function evaluations} (right) for additional benchmark problems.}
\label{fig:q_anytime_extra}
\end{figure}
\FloatBarrier
\subsection{Noisy Observations}
Although neither \qEHVI{} nor any variant of expected hypervolume improvement (to our knowledge) directly account for noisy observations, noisy observations are a practical challenge. We empirically evaluate the performance of all algorithms on a Branin-Currin function where observations have additive, zero-mean, $iid$ Gaussian noise; the unknown standard deviation of the noise is set to be $1\%$ of the range of each objective. Fig~\ref{fig:branin_currin_noisy} shows that \qEHVI{} performs favorably in the presence of noise, besting all algorithms including Noisy \qParego{} ($q$NParego) (described in Appendix \ref{appdx:sec:qparego}), PESMO and TS-TCH, all of which account for noise.
\begin{figure}[ht]
    \centering
    \includegraphics[width=0.5\linewidth]{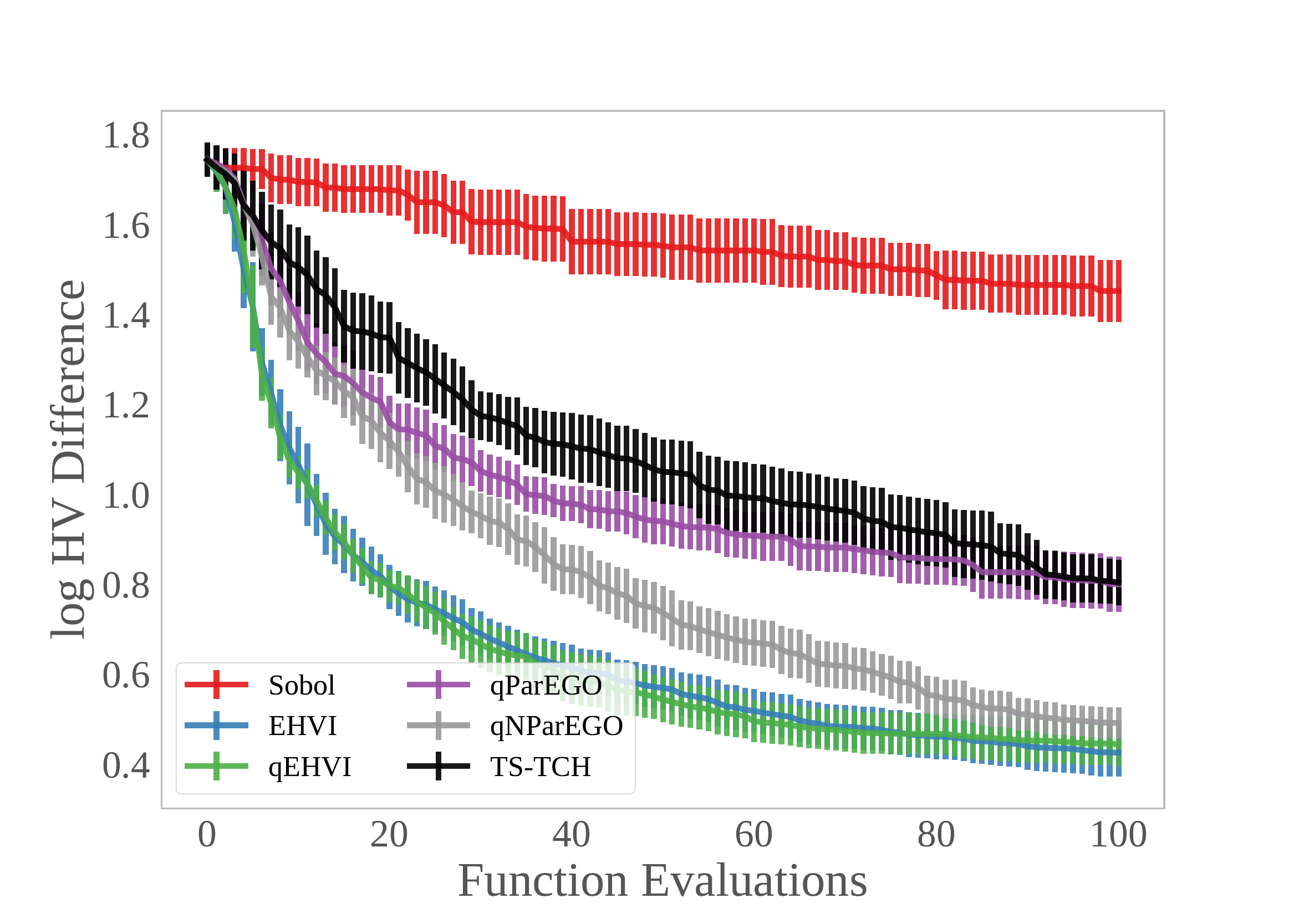}
    \caption{\label{fig:branin_currin_noisy} Sequential optimization performance on a noisy Branin-Currin problem.}
\end{figure}

\subsection{Approximate Box Decompositions}
\label{appdx:subsec:approx_decomp}
\EHVI{} becomes prohibitively computationally expensive in many scenarios with $\geq 4$ objectives because of the wall time of partitioning the non-dominated space into disjoint rectangles \citep{couckuyt12}. Therefore, in addition to providing an exact binary partitioning algorithm, \citet{couckuyt12} propose an approximation that terminates the partitioning algorithm when the new additional set of hyper-rectangles in the partitioning has a total hypervolume of less than a predetermined fraction $\zeta$ of the hypervolume dominated by the Pareto front. While \qEHVI{} is guaranteed to be exact when an exact partitioning of the non-dominated space is used, \qEHVI{} is agnostic to the partitioning algorithm used and is compatible with more scalable approximate methods.

We evaluate the performance of \qEHVI{} with approximation of various fidelities $\zeta$ on DTLZ2 problems with 3 and 4 objectives (with $d=6$). $\zeta=0$ corresponds to an exact partitioning and the approximation is monotonically worse as $\zeta$ increases.  Larger values of $\zeta$ degrade optimization performance (Figure~\ref{fig:approx_optim}), but can result in substantial speedups (Table~\ref{table:approx_walltime}).  Even with coarser levels of approximation, \qEHVI() performs better than \qParego{} with respect to log hypervolume difference, while achieving wall time improvements of 2-7x compared to exact \qEHVI{}.

\begin{figure}[ht]
    \centering
    \begin{subfigure}{.49\textwidth}
        \centering
\includegraphics[width=\linewidth]{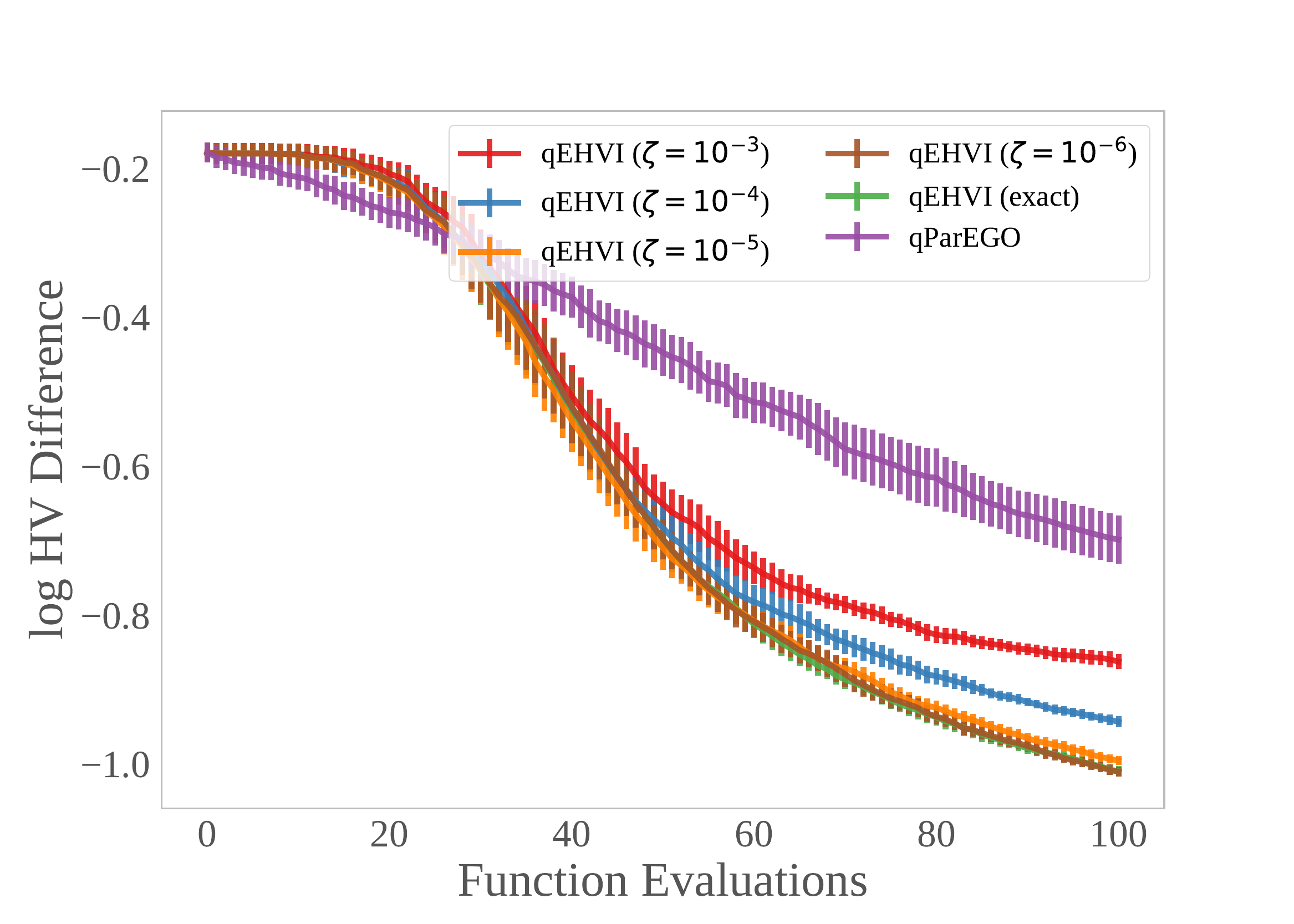}
        \subcaption{\label{fig:approx_m3}}
    \end{subfigure} %
    \begin{subfigure}{.49\textwidth}    
        \centering
        \includegraphics[width=\linewidth]{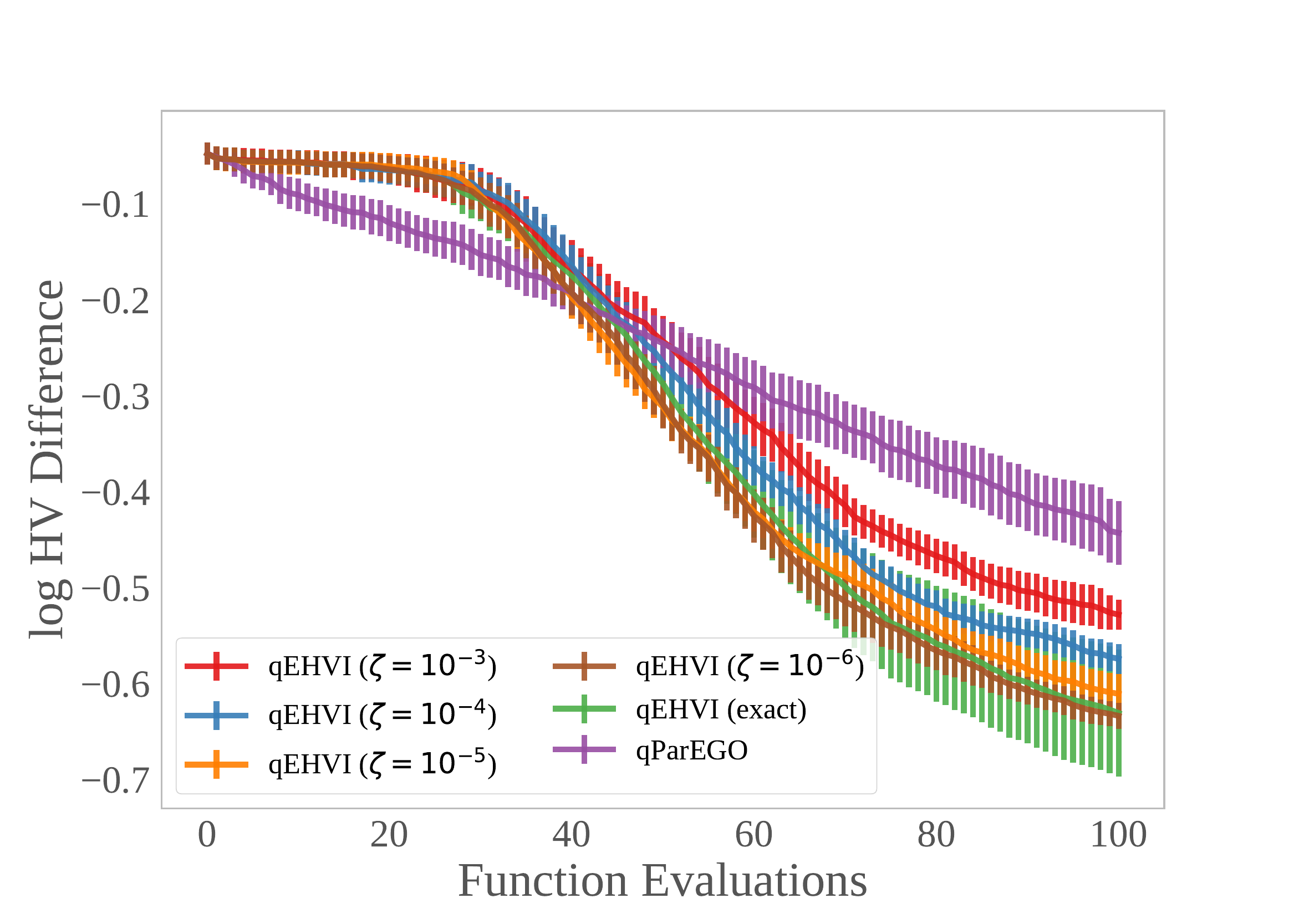}
        \subcaption{\label{fig:approx_m4}}
    \end{subfigure}
    \caption{\label{fig:approx_optim}Optimization performance on DTLZ2 problems ($d=6$) with approximate partitioning using various approximation levels $\zeta$ for (a) $M=3$ objectives and (b) $M=4$ objectives.}
\end{figure}

\begin{table*}[t!]
\centering
\begin{small}
\begin{sc}
\begin{tabular}{lcc}
\toprule
\textbf{CPU}& DTLZ2 ($M=3$) & DTLZ2 ($M=4$) \\
\midrule
\qParego{} & $5.86 ~(\pm 0.51)$ & $5.6 ~(\pm 0.53)$\\
\qEHVI{} ($\zeta=10^{-3}$) & $6.89 ~(\pm 0.41)$ & $9.53 ~(\pm 0.49)$\\
\qEHVI{} ($\zeta=10^{-4}$) & $9.83 ~(\pm 0.9)$ & $17.47 ~(\pm 1.2)$\\
\qEHVI{} ($\zeta=10^{-5}$) & $18.99 ~(\pm 2.72)$ & $60.27 ~(\pm 3.57)$\\
\qEHVI{} ($\zeta=10^{-6}$) & $37.9 ~(\pm 7.47)$ & $136.15 ~(\pm 12.88)$ \\
\qEHVI{} (exact) & $45.52 ~(\pm 9.83)$ & $459.33 ~(\pm 77.95)$\\
\bottomrule
\end{tabular}
\end{sc}
\end{small}
\caption{\label{table:approx_walltime} Acquisition function optimization wall time with approximate hypervolume computation, in seconds on a CPU (2x Intel Xeon E5-2680 v4 @ 2.40GHz).  The mean and two standard errors are reported.}
\end{table*}

\FloatBarrier
\subsection{Acquisition Computation Time}
Figure \ref{fig:qehi_comp_time} show the acquisition computation time for different $M$ and $q$. The inflection points corresponds to available processor cores becoming saturated. For large $M$ an $q$ on the GPU, memory becomes an issue, but we discuss ways of mitigating the issue in Appendix \ref{appdx:subsec:DqEHVI:complexity}.
\begin{figure}[ht]
        \centering
        \includegraphics[width=0.5\linewidth]{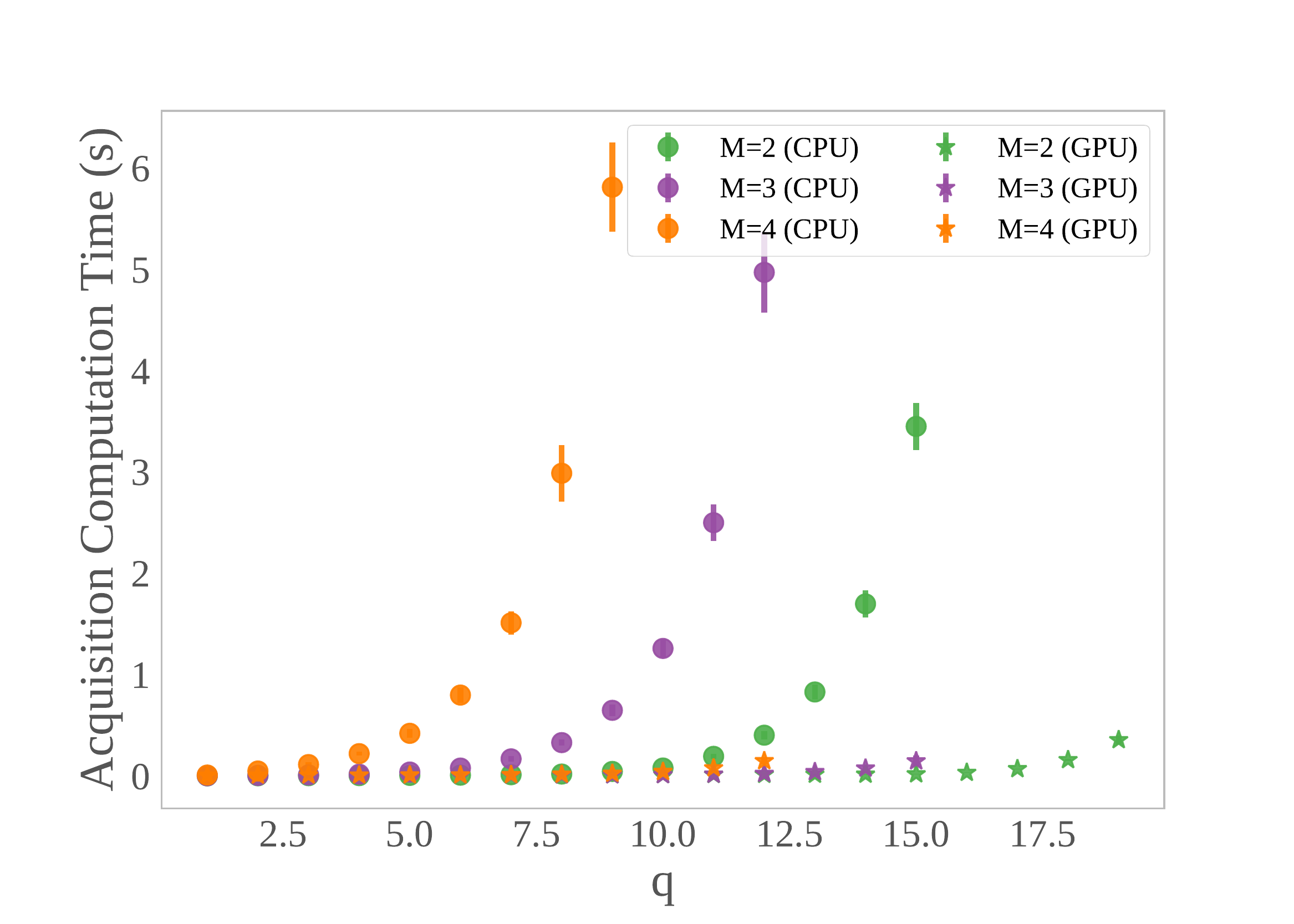}
        \caption{\label{fig:qehi_comp_time}Acquisition computation time for different batch sizes $q$ and numbers of objectives $M$ (this excludes the time required to compute the acquisition function \emph{given} box decomposition of the non-dominated space). This uses $N=512$ MC samples, $d=6$, $\vert \mathcal P\vert = 10$, %(which is not totally fair as you will likely have a larger Pareto set in higher dimensional objective spaces),
        and 20 training points.
        %The wall time of obtaining a box decomposition is super polynomial in the number of objectives. \textbf{The key point from this figure is that by exploiting parallelism on modern hardware, we are able to maintain reasonable wall time as the total computational load increases exponentially with respect to $q$.}
        CPU time was measured on 2x Intel Xeon E5-2680 v4 @ 2.40GHz and GPU time was measured on a Tesla V100-SXM2-16GB GPU using 64-bit floating point precision. The mean and 2 standard errors over 1000 trials are reported.}
\end{figure}
\end{document}